\documentclass{article}
\usepackage{ijcai22}

\usepackage{times}
\usepackage{soul}
\usepackage{url}
\usepackage[hidelinks]{hyperref}
\usepackage[utf8]{inputenc}
\usepackage[small]{caption}
\usepackage{graphicx}
\usepackage{amsmath}
\usepackage{amsthm}
\usepackage{booktabs}
\usepackage{algorithm}
\usepackage{defs}

\title{Lexicographic Multi-Objective Reinforcement Learning}

\author{
Joar Skalse\footnote{Contact Author}\and
Lewis Hammond\and
Charlie Griffin\And
Alessandro Abate\\
\affiliations
Department of Computer Science, University of Oxford
\emails
\{joar.skalse, lewis.hammond, charlie.griffin, aabate\}@cs.ox.ac.uk
}

\begin{document}

\maketitle

\begin{abstract}
    In this work we introduce reinforcement learning techniques for solving lexicographic multi-objective problems.  These are problems that involve multiple reward signals, and where the goal is to learn a policy that maximises the first reward signal, and subject to this constraint also maximises the second reward signal, and so on. We present a family of both action-value and policy gradient algorithms that can be used to solve such problems, and prove that they converge to policies that are lexicographically optimal. We evaluate the scalability and performance of these algorithms empirically, demonstrating their practical applicability. As a more specific application, we show how our algorithms can be used to impose safety constraints on the behaviour of an agent, and compare their performance in this context with that of other constrained reinforcement learning algorithms. 
\end{abstract}

\section{Introduction}

Reinforcement learning (RL) algorithms learn to solve tasks in unknown environments by a process of trial and error, where the task typically is encoded as a scalar reward function. However, there are tasks for which it is difficult (or even infeasible) to create such a function. Consider, for example, Isaac Asimov's three Laws of Robotics -- the task of following these laws involves multiple (possibly conflicting) objectives, some of which are \emph{lexicographically} (i.e.\ categorically) more important than others. There is, in general, no straightforward way to write a scalar reward function that encodes such a task without ever incentivising the agent to prioritise less important objectives. In such cases, it is difficult (and often unsuitable) to apply standard RL algorithms.

In this work, we introduce several RL techniques for solving \emph{lexicographic multi-objective problems}. More precisely, we present both a family of action-value algorithms and a family of policy gradient algorithms that can accept multiple reward functions $R_1, \dots, R_m$, and that learn a policy $\pi$ such that $\pi$ maximises expected discounted $R_1$-reward, and among all policies that do so, $\pi$ also maximises expected discounted $R_2$-reward, and so on. These techniques can easily be combined with a wide range of existing RL algorithms. We also prove the convergence of our algorithms, and benchmark them against state-of-the-art methods for constrained reinforcement learning in a number of environments.

\subsection{Related Work}

Lexicographic optimisation in Multi-Objective RL (MORL) has previously been studied by \cite{gabor1998}, whose algorithm is a special case of one of ours (cf.\ Footnote~\ref{footnote:re_gabor}). Our contribution extends this work to general, state-of-the-art RL algorithms. Unlike \cite{gabor1998}, we also prove that our algorithms converge to the desired policies, and provide benchmarks against other state-of-the-art algorithms in more complex environments. Other MORL algorithms combine and trade off rewards in different ways; for an overview, see \cite{Roijers2013,Liu2015}. Lexicographic optimisation more generally is a long-studied problem -- see, e.g.\ \cite{Mitten1974,Rentmeestersa1996,Wray2015MultiObjectivePW}.

A natural application of lexicographic RL (LRL) is to learn a policy that maximises a performance metric, subject to satisfying a safety constraint. 
This setup has been tackled with dynamic programming in \cite{LA18}, and has also been studied within RL. 
For example, \cite{Tessler2019} introduce an algorithm that maximises a reward subject to the constraint that the expectation of an additional penalty signal should stay below a certain threshold; \cite{chow2015riskconstrained} introduce techniques to maximise a reward subject to constraints on the value-at-risk (VaR), or the conditional value-at-risk (CVaR), of a penalty signal; and \cite{miryoosefi2019reinforcement} discuss an algorithm that accepts an arbitrary number of reward signals, and learns a policy whose expected discounted reward vector lies inside a given convex set. 

Our contributions add to this literature and, unlike the methods above, allow us to encode safety constraints in a principled way without prior knowledge of the level of safety that can be attained in the environment. Note that lexicographic optimisation of two rewards is qualitatively different from maximising one reward subject to a constraint on the second, and thus the limit policies of LRL and the algorithms above will not, in general, be the same. 
Other methods also emphasise staying safe \textit{while} learning; see e.g.\ \cite{achiam2017constrained,Thomas2013,PAR19}.
In contrast, our algorithms do not guarantee safety while learning, but rather learn a safe limit policy.

\section{Background}





\paragraph{Reinforcement Learning.} The RL setting is usually formalised as a \textit{Markov Decision Process} (MDP), which is a tuple $\langle S,A,T,I,R,\gamma \rangle$ where $S$ is a set of states, $A$ is a set of actions, $T : S \times A \rightsquigarrow S$ is a \textit{transition function}, $I$ is an initial state distribution over $S$, $R : S \times A \times S \rightsquigarrow \mathbb{R}$ a \textit{reward function}, where $R(s,a,s')$ is the reward obtained if the agent moves from state $s$ to $s'$ by taking action $a$, and $\gamma \in [0,1]$ is a \textit{discount factor}.
Here, $f : X \rightsquigarrow Y$ denotes a probabilistic mapping $f$ from $X$ to $Y$. A state is \textit{terminal} if $T(s,a)=s$ and $R(s,a,s)=0$ for all $a$. 

A (stationary) \textit{policy} is a mapping $\pi : S \rightsquigarrow A$ that specifies a distribution over the agent's actions in each state. The \textit{value function} $v_\pi (s)$ of $\pi$ is defined as the \textit{expected $\gamma$-discounted cumulative reward} when following $\pi$ from $s$, i.e.\ $v_\pi (s) := \mathbb{E}_\pi\left[\sum_{t=0}^\infty \gamma^t R(s_t,a_t,s_{t+1}) \mid s_0 = s \right]$. When $\gamma = 1$, we instead consider the limit-average of this expectation. The objective in RL can then be expressed as maximising $J(\pi) := \sum_s I(s) v_\pi (s)$. Given a policy $\pi$ we may also define the \textit{q-function} $q_\pi(s,a) := \mathbb{E}_{s' \sim T(s,a)}\left[R(s,a,s') + v_\pi(s')\right]$ and the \textit{advantage function} $a_\pi(s,a) := q_\pi(s,a) - v_\pi(s)$.

\paragraph{Value-Based Methods.}
A value-based agent has two main components:  a \textit{Q-function} $Q: S \times A \rightarrow \mathbb{R}$ that predicts the expected future discounted reward conditional on taking a particular action in a particular state; and a \textit{bandit algorithm} that is used to select actions in each state. The $Q$-function can be represented as a lookup table (in which case the agent is \textit{tabular}), or as a function approximator.


There are many ways to update the $Q$-function. 
One popular rule is \textit{$Q$-Learning} \cite{watkins1989}:
\begin{align*}
    Q(s_t, a_t) \gets &\big(1-\alpha_t(s_t,a_t)\big)\cdot Q(s_t, a_t)\\
    &+ \alpha_t(s_t,a_t)\cdot \big(r_t + \gamma \max_a Q(s_{t+1}, a)\big),
\end{align*}
where $t$ is the time step and $\alpha_t(s_t, a_t)$ is a \textit{learning rate}. 
One can replace the term $\max_a Q(s_{t+1}, a)$ in the rule above with $Q(s_{t+1}, a_{t+1})$ or $\mathbb{E}_{a \sim \pi(s)}[Q(s_{t+1}, a)]$ (where $\pi$ is the policy that describes the agent's current behaviour) to obtain the \textit{SARSA} \cite{Rummery1994} or the \textit{Expected SARSA} \cite{Seijen2009} updates respectively.

\paragraph{Policy-Based Methods.}
In these methods, the policy $\pi(\cdot ; \theta)$ is differentiable with respect to some $\theta \in \Theta \subset \mathbb{R}^x$, and $\theta$ is updated according to an objective $K(\theta)$. If using $K^{\ac} (\theta) \coloneqq J(\theta)$ then we may estimate this using:
$$K^{\ac} (\theta) \coloneqq \expect_t \big[ \log \pi(a_t \mid s_t ; \theta) \cdot A_\theta(s_t,a_t) \big],$$
where $A_\theta$ is an estimate of $a_\theta$.
One often computes $A_\theta$ by approximating $v_\theta$ with a function $V$ parameterised by some $w \in W \subset \mathbb{R}^y$, and using the fact that the expected \textit{temporal difference error} $\delta_t := r_t + \gamma v_\theta(s_{t+1}) - v_\theta(s_{t})$ (or $r_t + v_\theta(s_{t+1}) - v_\theta(s_{t}) - J(\theta)$ when $\gamma = 1$) equals $a_\theta(s_t,a_t)$ \cite{Bhatnagar2009}. Such algorithms are known as \emph{Actor-Critic} (AC) algorithms \cite{Konda2000}.\footnote{Due to the choice of baseline \cite{Sutton1999}, we describe here the classic \emph{Advantage Actor-Critic} (A2C) algorithm.} 

More recently, other policy gradient algorithms have used \emph{surrogate} objective functions,
which increase stability in training by penalising large steps in policy space, and can be viewed as approximating the \emph{natural} policy gradient \cite{Amari1998,Kakade2001}.
One common such penalty is the Kullback–Leibler (KL) divergence between new and old policies, as employed in one version of Proximal Policy Optimisation (PPO) \cite{Schulman2017}, leading to:
\begin{align*}
    K^{\ppo}(\theta) \coloneqq \expect_t \Big[ &\frac{\pi(a_t \mid s_t ; \theta)}{\pi(a_t \mid s_t ; \theta_{\old})} A_\theta(s_t,a_t)\\
    - &\kappa \cdot {\kl} ( \pi(s_t ; \theta) ~\Vert~ \pi(s_t; \theta_{\old}) ) \Big],
\end{align*}
where $\kappa$ is a scalar weight.
Such algorithms enjoy both state of the art performance and strong convergence guarantees \cite{Hsu2020,Liu2019b}.

\paragraph{Multi-Objective Reinforcement Learning.}
MORL is concerned with policy synthesis  under multiple objectives. This setting can be formalised as a \textit{multi-objective MDP} (MOMDP), which is a tuple $\langle S,A,T,I,\mathfrak{R},\gamma \rangle$ that is defined analogously to an MDP, but where $\mathfrak{R} : S \times A \times S \rightsquigarrow \mathbb{R}^m$ returns a vector of $m$ rewards, and $\gamma \in [0,1]^m$ defines $m$ discount rates. We define $R_i$ as $(s,a,s) \mapsto \mathfrak{R}(s,a,s)_i$.

\section{Lexicographic Reinforcement Learning}

In this section we present a family of value-based and policy-based algorithms that solve lexicographic multi-objective problems by learning a lexicographically optimal policy. Given a MOMDP $\mathcal{M}$ with $m$ rewards, we say that a policy $\pi$ is (globally) \textit{lexicographically $\epsilon$-optimal} if $\pi \in \Pi^\epsilon_m$, where $\Pi^\epsilon_0 = \Pi$ is the set of all policies in $\mathcal{M}$, $\Pi^\epsilon_{i+1} \coloneqq \{\pi \in \Pi^\epsilon_i \mid \max_{\pi' \in \Pi^\epsilon_i} J_i(\pi') - J_i(\pi) \leq \epsilon_i \}$, and $\mathbb{R}^{m-1} \ni \epsilon \succcurlyeq 0$.
We similarly write $\Theta^\epsilon_{i+1}$ to define global lexicographic $\epsilon$-optima for parametrised policies, but also $\tilde{\Theta}^\epsilon_{i+1} \coloneqq \{\theta \in \Theta^\epsilon_i \mid \max_{\theta' \in \nh_i(\theta)} J_i(\theta') - J_i(\theta) \leq \epsilon_i \}$ to define \emph{local} lexicographic $\epsilon$-optima, where $\nh_i(\theta) \subseteq \tilde{\Theta}^\epsilon_i$ is a compact local neighbourhood of $\theta$, and $\Theta^\epsilon_0 = \tilde{\Theta}^\epsilon_{-1} = \Theta$.
When $\epsilon = 0$ we drop it from our notation and refer to \textit{lexicographic optima} and \textit{lexicographically optimal policies} simpliciter.



\subsection{Value-Based Algorithms}


We begin by introducing bandit algorithms that take as input multiple $Q$-functions and converge to taking lexicographically optimal actions.


\begin{definition}[Lexicographic Bandit Algorithm]\label{def:lex_bandit} Let $S$ be a set of states, $A$ a set of actions, $Q_1, \dots, Q_m : S \times A \rightarrow \mathbb{R}$ a sequence of $Q$-functions, and $t \in \mathbb{N}$ a time parameter. A lexicographic bandit algorithm with tolerance $\tau \in \mathbb{R}_{> 0}$ is a function $\mathcal{B} : (S \times A \rightarrow \mathbb{R})^m \times S \times \mathbb{N} \rightsquigarrow A$, such that 
$$
\lim_{t\to\infty} \Pr(\mathcal{B}(Q_1, \dots, Q_m,s,t) \in \Delta^\tau_{s,m}) = 1,
$$
where $\Delta^\tau_{s,0} = A$ and $\Delta^\tau_{s,i+1} \coloneqq \{ a \in \Delta^\tau_{s,i} \mid Q_i(s,a) \geq \max_{a' \in \Delta^\tau_{s,i}} Q_i(s,a')-\tau\}$.
\end{definition}


Intuitively, a lexicographic bandit algorithm will, in the limit, pick an action $a$ such that $a$ maximises $Q_1$ (with tolerance $\tau$), and among all actions that do this, action $a$ also maximises $Q_2$ (with tolerance $\tau$), and so on. An example of a lexicographic bandit algorithm is given in Algorithm~\ref{alg:LeG}, where the exploration probabilities $\epsilon_{s,t}$ should satisfy $\lim_{t \rightarrow \infty}\epsilon_{s,t} = 0$ and $\sum_{t=0}^\infty\epsilon_{s,t} = \infty$ for all $s \in S$.






We can now introduce Algorithm~\ref{alg:VB-LRL} (VB-LRL), a value-based algorithm for lexicographic multi-objective RL. Here $\mathcal{B}$ is any lexicographic bandit algorithm. The rule for updating the $Q$-values (on line \ref{alg:q-update}) can be varied. We call the following update rule \textit{Lexicographic $Q$-Learning}:
\begin{align*}
Q_i(s, a) \gets \big(&1-\alpha_t(s,a)\big)\cdot Q_i(s, a) ~ + \\
&\alpha_t(s,a) \cdot \big(R_i(s,a,s') + \gamma_i \max_{a \in \Delta^\tau_{s,i}} Q_i(s', a)\big), 
\end{align*}
where $\Delta^\tau_{s,0} = A$, $\Delta^\tau_{s,i+1} \coloneqq \{ a \in \Delta^\tau_{s,i} \mid Q_i(s,a) \geq \max_{a' \in \Delta^\tau_{s,i}} Q_i(s,a') - \tau\}$, and $\tau \in \mathbb{R}_{> 0}$ is the tolerance parameter.\footnote{There are several places where VB-LRL makes use of a tolerance parameter $\tau$. In the main text of this paper, we assume that the same tolerance parameter is used everywhere, and that it is a constant. In the supplementary material, we relax these assumptions.} This rule is analogous to $Q$-Learning, but where the max-operator is restricted to range only over actions that (approximately) lexicographically maximise all rewards of higher priority. We can also use 
SARSA or Expected SARSA. 
Alternatively, we can adapt Double $Q$-Learning \cite{vanHasselt2010} for VB-LRL. To do this, we let the agent maintain two $Q$-functions $Q^A_i$, $Q^B_i$ for each reward. To update the $Q$-values, with probability $0.5$ we set: 
\begin{align*}
    &Q^A_i(s, a) \gets \big(1-\alpha_t(s,a)\big)\cdot Q^A_i(s, a) ~ + \\
    &\alpha_t(s,a)\cdot\Big(R_i(s,a,s')
    + \gamma_i \cdot Q^B_i \big(s', \argmax_{a'\in\Delta^\tau_{s,i}} Q^A_i(s',a') \big) \Big),
\end{align*}
and else perform the analogous update on $Q^B_i$, and let $Q_i(s,a) \coloneqq 0.5\big(Q^A_i(s,a) + Q^B_i(s,a)\big)$ in the bandit algorithm. Varying the bandit algorithm or $Q$-value update rule in VB-LRL produces a family of algorithms with different properties.\footnote{\label{footnote:re_gabor}The LRL algorithm in \cite{gabor1998} is equivalent to Algorithm~\ref{alg:VB-LRL} with Algorithm~\ref{alg:LeG}, Lexicographic $Q$-Learning, and $\tau = 0$.} We can now give our core result for Algorithm~\ref{alg:VB-LRL}.
All our proofs are included in the supplementary material.\footnote{Available at \url{https://github.com/lrhammond/lmorl}.}
\begin{theorem}\label{thm:convergence_short}
In any MOMDP $\mathcal{M}$, if VB-LRL uses a lexicographic bandit algorithm and either SARSA, Expected SARSA, or Lexicographic $Q$-Learning, then it will converge to a policy $\pi$ that is lexicographically optimal if:
\begin{enumerate}
    \item $S$ and $A$ are finite,
    \item All reward functions are bounded,
    \item Either $\gamma_1 , \dots,  \gamma_m < 1$, or every policy leads to a terminal state with probability one,
    \item The learning rates $\alpha_t(s,a) \in [0,1]$ satisfy the conditions $\sum_t \alpha_t(s,a) = \infty$ and $\sum_t \alpha_t(s,a)^2 < \infty$ with probability one, for all $s\in S$, $a \in A$,
    \item The tolerance $\tau$ satisfies the condition that $0 < \tau < \min_{i,s,a \neq a'} \vert q_i(s,a) - q_i(s,a') \vert$. 
\end{enumerate}
\end{theorem}

\begin{algorithm}
\caption{Lexicographic $\epsilon$-Greedy}
\begin{algorithmic}[1]
\Input $Q_1, \dots, Q_m$, $s$, $t$ 
\WithProb{$\epsilon_{s,t}$}{ $a \sim \unif(A)$}
\Else
    \State $\Delta \gets A$
    \For{$i \in \{1, \ldots, m\}$}
        \State $x \gets \max_{a' \in \Delta} Q_i(s,a')$
        \State $\Delta \gets \{ a \in \Delta \mid Q_i(s,a) \geq x-\tau\}$
    \EndFor
    \State $a \sim \unif(\Delta)$
\EndWithProb
\Return $a$
\end{algorithmic}
\label{alg:LeG}
\end{algorithm}

\begin{algorithm}
\caption{Value-Based Lexicographic RL}
\begin{algorithmic}[1]
\Input $\mathcal{M} = \langle S,A,T,I,\mathfrak{R},\gamma \rangle$ 
\State initialise $Q_1, \ldots, Q_m$, \quad $t \gets 0$, \quad $s \sim I$
\While{$Q_1, \ldots, Q_m$ have not converged}
    \State $a \gets \mathcal{B}(Q_1, \dots, Q_m, s, t)$ \Comment{Algorithm \ref{alg:LeG}}
    \State $s' \gets T(s,a)$
    \For{$i \in \{1, \ldots, m\}$}
        \State update $Q_i$ \label{alg:q-update}
    \EndFor
    \State \textbf{if} $s'$ is terminal \textbf{then} $s \sim I$ \textbf{else} $s \gets s'$
    \State $t \gets t + 1$
\EndWhile
\Return $\pi = s \mapsto \lim_{t \rightarrow \infty}\mathcal{B}(Q_1, \dots, Q_m, s, t)$
\end{algorithmic}
\label{alg:VB-LRL}
\end{algorithm}

We also show that VB-LRL with Lexicographic Double $Q$-Learning converges to a lexicographically optimal policy.
\begin{theorem}\label{thm:double_convergence_short}
In any MOMDP, if VB-LRL uses Lexicographic Double $Q$-Learning then it converges to a lexicographically optimal policy $\pi$ if conditions 1--5 in Theorem~\ref{thm:convergence_short} hold.
\end{theorem}
Condition 4 requires that the agent takes every action in every state infinitely often. Condition 5 is quite strong -- the upper bound on this range can in general not be determined \textit{a priori}. However, we expect VB-LRL to be well-behaved as long as $\tau$ is small. We motivate this intuition with a formal guarantee about the behaviour of VB-LRL for arbitrary $\tau$.



\begin{proposition}\label{thm:VB-LRK_with_arbitrary_slack}
In any MOMDP, if VB-LRL has tolerance $\tau > 0$, uses SARSA, Expected SARSA, or Lexicographic $Q$-Learning, and conditions 1--4 in Theorem~\ref{thm:convergence_short} are met, then:
\begin{enumerate}
    \item $J_1(\pi^*) - J_1(\pi_t) \leq \frac{\tau}{1-\gamma} - \lambda_t$, for some sequence $\{\lambda_t\}_{t \in \mathbb{N}}$ such that $\lim_{t \to \infty} \lambda_t = 0$,  
    \item $J_2(\pi^*) - J_2(\pi_t) \leq \frac{\tau}{1-\gamma} - \eta_t$, for some sequence $\{\eta_t\}_{t \in \mathbb{N}}$ such that $\lim_{t \to \infty} \eta_t = 0$,
\end{enumerate}
where $\pi_t$ is the policy at time $t$ and $\pi^*$ is a lexicographically optimal policy. 
\end{proposition}
Proposition~\ref{thm:VB-LRK_with_arbitrary_slack} shows that we can obtain guarantees about the limit behaviour of VB-LRL without prior knowledge of the MOMDP when we only have two rewards, or are primarily interested in the two most prioritised rewards. We discuss this issue further in the supplementary material. 
Note also that while Algorithm~\ref{alg:VB-LRL} is tabular, it is straightforward to combine it with function approximators, making it applicable to high-dimensional state spaces. 

\subsection{Policy-Based Algorithms}

We next introduce a family of lexicographic policy gradient algorithms. These algorithms use one objective function $K_i$ for each reward function, and update the parameters of $\pi(\cdot ; \theta)$ with a multi-timescale approach whereby we first optimise $\theta$ using $K_1$, then at a slower timescale optimise $\theta$ using $K_2$ while adding the condition that the loss with respect to $K_1$ remains bounded by its current value, and so on. To solve these problems we use the well-known Lagrangian relaxation technique  \cite{Bertsekas1999}.

Suppose that we have already optimised $\theta'$ lexicographically with respect to $K_1, \ldots, K_{i-1}$ and we wish to now lexicographically optimise $\theta$ with respect to $K_i$. 
Let $k_j := K_j(\theta')$ for each $j \in \{ 1, \ldots, i-1 \}$. Then we wish to solve the constrained optimisation problem given by:
\begin{align*}
    &\text{maximise} &&K_i(\theta),\\
    &\text{subject to} &&K_j(\theta) \geq k_j - \tau, ~~~~~ \forall~ j \in \{1, \ldots, i-1\}, 
\end{align*}
where $\tau > 0$ is a small constant tolerance parameter, included such that there exists some $\theta$ strictly satisfying the above constraints; in practice, while learning we set $\tau = \tau_t$ to decay as $t \rightarrow \infty$. This constraint qualification (Slater's condition \cite{Slater1950}) ensures that we may instead solve the dual of the problem by computing a saddle point $\min_{\lambda \succcurlyeq 0} \max_\theta L_i(\theta,\lambda)$ of the Lagrangian relaxation \cite{Bertsekas1999} where:
\begin{align*}
     L_i(\theta,\lambda) \coloneqq   K_i(\theta) + \sum^{i-1}_{j=1} \lambda_j \big(K_j(\theta) - k_j + \tau \big). 
\end{align*}
A natural approach would be to solve each optimisation problem for $L_i$, where $i \in \{1,\dots,m\}$, in turn. While this would lead to a correct solution, when the space of lexicographic optima for each objective function is large or diverse, this process may end up being slow and sample-inefficient. Our key observation here is that by instead updating $\theta$ at different timescales, we can solve this problem synchronously, guaranteeing that we converge to a lexicographically optimal solution as if done step-by-step under fixed constraints.

We set the learning rate $\eta$ of the Lagrange multiplier to $\eta^{i}$ after convergence with respect to the $i^\text{th}$ objective, and assume that for all learning rates $\iota \in \{\alpha, \beta^1, \ldots, \beta^m, \eta^0, \ldots, \eta^m \}$ and all $i \in \{1,\dots,m\}$ we have:
\begin{align*}
    &\iota_t \in [0,1],~ \sum^\infty_{t=0} \iota_t = \infty,~ \sum^\infty_{t=0} (\iota_t)^2 < \infty ~~~\mathrm{and}~~~ \\ &\lim_{t \rightarrow \infty} \frac{\beta^{i}_t}{\alpha_t} = \lim_{t \rightarrow \infty} \frac{\eta^i_t}{\beta^i_t} = \lim_{t \rightarrow \infty} \frac{\beta^{i}_t}{\eta^{i-1}_t} = 0.
\end{align*}

\begin{algorithm}
\caption{Policy-Based Lexicographic RL}
\label{alg:PB-LRL}
\begin{algorithmic}[1]
\Input $\mathcal{M} = \langle S,A,T,I,\mathfrak{R},\gamma \rangle$
\State initialise $\theta, w_1, \ldots, w_m, \lambda_1, \ldots, \lambda_m$
\State $t \gets 0$,
\quad $\eta \gets \eta^0$,
\quad $s \sim I$
\While{$\theta$ has not converged}
        \State $t \gets t + 1$,
        \quad $a \sim \pi(s)$,
        \quad $s' \sim T(s,a)$
        \For{$i \in \{1, \ldots, m\}$}
            \If{$\hat{K}_i(\theta)$ has not converged}
                {$\hat{k}_i \gets \hat{K}_i(\theta)$}
            \Else{ $\eta \gets \eta^i$ \label{alg:eta_update}}
            \EndIf
            \State update $w_i$ (if using a critic) and $\lambda_i$
        \EndFor
        \State update $\theta$
        \State \textbf{if} $s'$ is terminal \textbf{then} $s \sim I$ \textbf{else} $s \gets s'$
\EndWhile
\Return $\theta$
\end{algorithmic}
\end{algorithm}

We also assume that $\tau_t = o(\beta^{m}_t)$ in order to make sure that Slater's condition holds in the limit with respect to all learning rates. Using learning rates $\beta^i$ and $\eta = \eta^i$ we may compute a saddle point solution to each Lagrangian $L_i$ via the following (estimated) gradient-based updates:
\begin{align*}
    \theta &\gets \Gamma_\theta \bigg[\theta + \beta^i_t \Big( \nabla_\theta \hat{K}_i(\theta) + \sum^{i-1}_{j=1} \lambda_j \nabla_\theta \hat{K}_j(\theta) \Big) \bigg],\\
    \lambda_j &\gets \Gamma_\lambda \Big[\lambda_j + \eta_t \big( \hat{k}_j - \tau_t - \hat{K}_j(\theta) \big) \Big] ~ \forall~ j \in \{1,\ldots,i-1\},
\end{align*}
where $\Gamma_\lambda(\cdot) = \max(\cdot, 0)$, $\Gamma_\theta$ projects $\theta$ to the nearest point in $\Theta$, and $~\hat{}~$ is used to denote a Monte Carlo estimate. We next note that by collecting the terms involved in the updates to $\theta$ for each $i$, at time $t$ we are effectively performing the simple update $\theta \gets \Gamma_{\theta}[\theta + \nabla_\theta \hat{K}(\theta)]$, where: 
$$
\hat{K}(\theta) \coloneqq \sum^m_{i=1} c^i_t \hat{K}_i(\theta) ~~~~\text{   and   }~~~~ c^i_t \coloneqq \beta^i_t + \lambda_i \sum^m_{j=i+1}\beta^j_t,
$$ 
and where we assume that $\sum^m_{j=m+1}\beta^j_t = 0$. 
It is therefore computationally simple to formulate a lexicographic optimisation problem from any collection of objective functions by updating a small number of coefficients $c^i_t$ at each timestep and then linearly combining the objective functions.

Finally, in many policy-based algorithms we use a critic $V_i$ (or $Q_i$) to estimate each $\hat{K}_i$ and so must also update the parameters $w_i$ of each critic. This is typically done on a faster timescale using the learning rate $\alpha$, for instance via the TD(0) update for $V_i$ given by $w_i \gets w_i + \alpha_t \big( \delta^i_t \nabla_{w_i} V_i \big)$, where $\delta^i_t$ is the TD error for $V_i$ at time $t$ \cite{SuttonAndBarto}. A general scheme for policy-based LRL (PB-LRL) is shown in Algorithm \ref{alg:PB-LRL}, which may be instantiated with a wide range of objective functions $K_i$ and update rules for each $w_i$.


Below, we show that PB-LRL inherits the convergence guarantees of whatever (non-lexicographic) algorithm corresponds to the objective function used. Note that when $m = 1$, PB-LRL reduces to whichever algorithm is defined by the choice of objective function, such as A2C when using $K^{\ac}_1$, or PPO when using $K^{\ppo}_1$. By using a standard stochastic approximation argument \cite{Borkar2008} and proceeding by induction, we prove that any such algorithm that obtains a local (or global) $\epsilon$-optimum when $m = 1$ obtains a \emph{lexicographically} local (or global) $\epsilon$-optimum when the corresponding objective function is used in PB-LRL.

\begin{theorem}
    \label{thm:PB-LRL}
    Let $\mathcal{M}$ be a MOMDP, $\pi$ a policy that is  twice continuously differentiable in its parameters $\theta$, and assume that the same form of objective function is chosen for each $K_i$ and that each reward function $R_i$ is bounded. If using a critic, let $V_i$ (or $Q_i$) be (action-)value functions that are continuously differentiable in $w_i$ for $i \in \{1,\dots,m\}$ and suppose that if PB-LRL is run for $T$ steps there exists some limit point $w^*_i(\theta) = \lim_{T \rightarrow \infty} \expect_t [ w_i ]$ for each $w_i$ when $\theta$ is held fixed under some set of conditions $\mathcal{C}$ on $\mathcal{M}$, $\pi$, and each $V_i$. 
    If $\lim_{T \rightarrow \infty}  \expect_t [ \theta ] \in \Theta^\epsilon_1$ (respectively $\tilde{\Theta}^\epsilon_1$) under conditions $\mathcal{C}$ when $m = 1$, then for any fixed $m \in \mathbb{N}$ we have that $\lim_{T \rightarrow \infty}  \expect_t [ \theta ] \in \Theta^\epsilon_m$ (respectively $\tilde{\Theta}^\epsilon_m$), where each $\epsilon_i \geq 0$ is a constant that depends on the representational power of the parametrisations of $\pi$ (and $V_i$ or $Q_i$, if using a critic).
\end{theorem}

In the remainder of the paper, we consider two particular variants of Algorithm \ref{alg:PB-LRL}, in which we use $K^{\ac}_i$ and $K^{\ppo}_i$ respectively, for each $i$. We refer to the first as Lexicographic A2C (LA2C) and the second as Lexicographic PPO (LPPO). We conclude this section by combining Theorem \ref{thm:PB-LRL} with certain conditions $\mathcal{C}$ that are sufficient for the local and global convergence of A2C and PPO respectively, in order to obtain the following corollaries. The proofs of these corollaries contain further discussion and references regarding the conditions required in each case.

\begin{corollary}
    \label{cor:LA2C}
    Suppose that each critic is linearly parametrised as $V_i(s) = w_i^\top \phi(s)$ for some choice of state features $\phi$ and is updated using a semi-gradient TD(0) rule, and that:
    \begin{enumerate}
        \item $S$ and $A$ are finite, and each reward function $R_i$ is bounded,
        \item For any $\theta \in \Theta$, the induced Markov chain over $S$ is irreducible,
        \item For any $s\in S$ and $a \in A$, $\pi(a \mid s ; \theta)$ is twice continuously differentiable,
        \item Letting $\Phi$ be the $\vert S \vert \times c$ matrix with rows $\phi(s)$, then $\Phi$ has full rank (i.e.\ the features are independent), $c \leq \vert S \vert$, and there is no $w \in W$ such that $\Phi w = 1$.
    \end{enumerate}
    Then for any MOMDP with discounted or limit-average objectives, LA2C almost surely converges to a policy in $\tilde\Theta^\epsilon_m$.
\end{corollary}

\begin{corollary}
    \label{cor:LPPO}
    Let $\pi(a \mid s; \theta, \chi) \propto \exp\big(\chi^{-1} f(s, a ; \theta)\big)$ and suppose that both $f$ and the \emph{action-value} critics $Q_i$ are parametrised using two-layer neural networks (where $\chi$ is a temperature parameter), that a semi-gradient TD(0) rule is used to update $Q_i$, and that $Q_i$ replaces $A_i$ in the standard PPO loss $K^{\ppo}$, both of which updates use samples from the \emph{discounted} steady state distribution. Further, let us assume that:
    \begin{enumerate}
        \item $S$ is compact and $A$ is finite, with $S \times A \subseteq \mathbb{R}^d$ for some finite $d > 0$, and each reward function $R_i$ is bounded,
        \item The neural networks have widths $\mu_f$ and $\mu_{Q_i}$ respectively with ReLU activations, initial input weights drawn from a normal distribution with mean $0$ and variance $\frac{1}{d}$, and initial output weights drawn from $\mathrm{unif}([-1,1])$,
        \item We have that $q^\pi_i(\cdot,\cdot) \in \big\{Q_i(\cdot,\cdot; w_i) \mid w_i \in  \mathbb{R}^y \big\}$ for any  $\pi \in \Pi$,
        \item There exists $c > 0$ such that for any $z \in \mathbb{R}^d$ and $\zeta > 0$ we have that $\expect_\pi \big[ \mathbf{1} (\vert z^\top (s,a) \vert \leq \zeta  )  \big] \leq \frac{c\zeta}{\Vert z \Vert_2}$ for any $\pi \in \Pi$.
    \end{enumerate}
    Then for any MOMDP with discounted objectives, LPPO almost surely converges to a policy in $\Theta^\epsilon_m$. Furthermore, if the coefficient of the KL divergence penalty $\kappa > 1$ then $\lim_{\mu_f, \mu_{Q_i} \rightarrow \infty} \epsilon = 0$.
\end{corollary}

\begin{figure}[ht]
  \centering
    \subfloat[256 states]{{\includegraphics[width=0.24\textwidth]{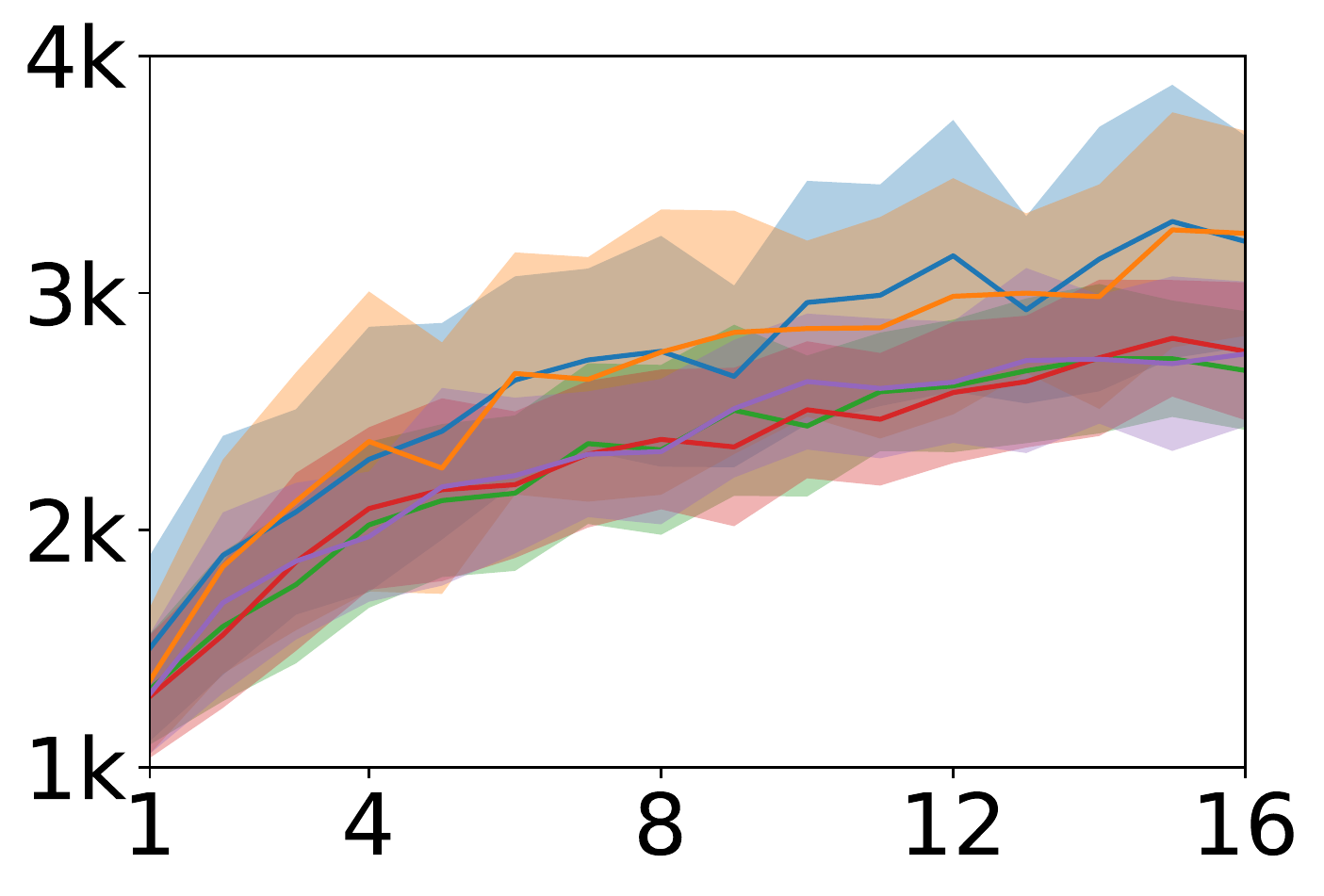}}}
    \subfloat[512 states]{{\includegraphics[width=0.24\textwidth]{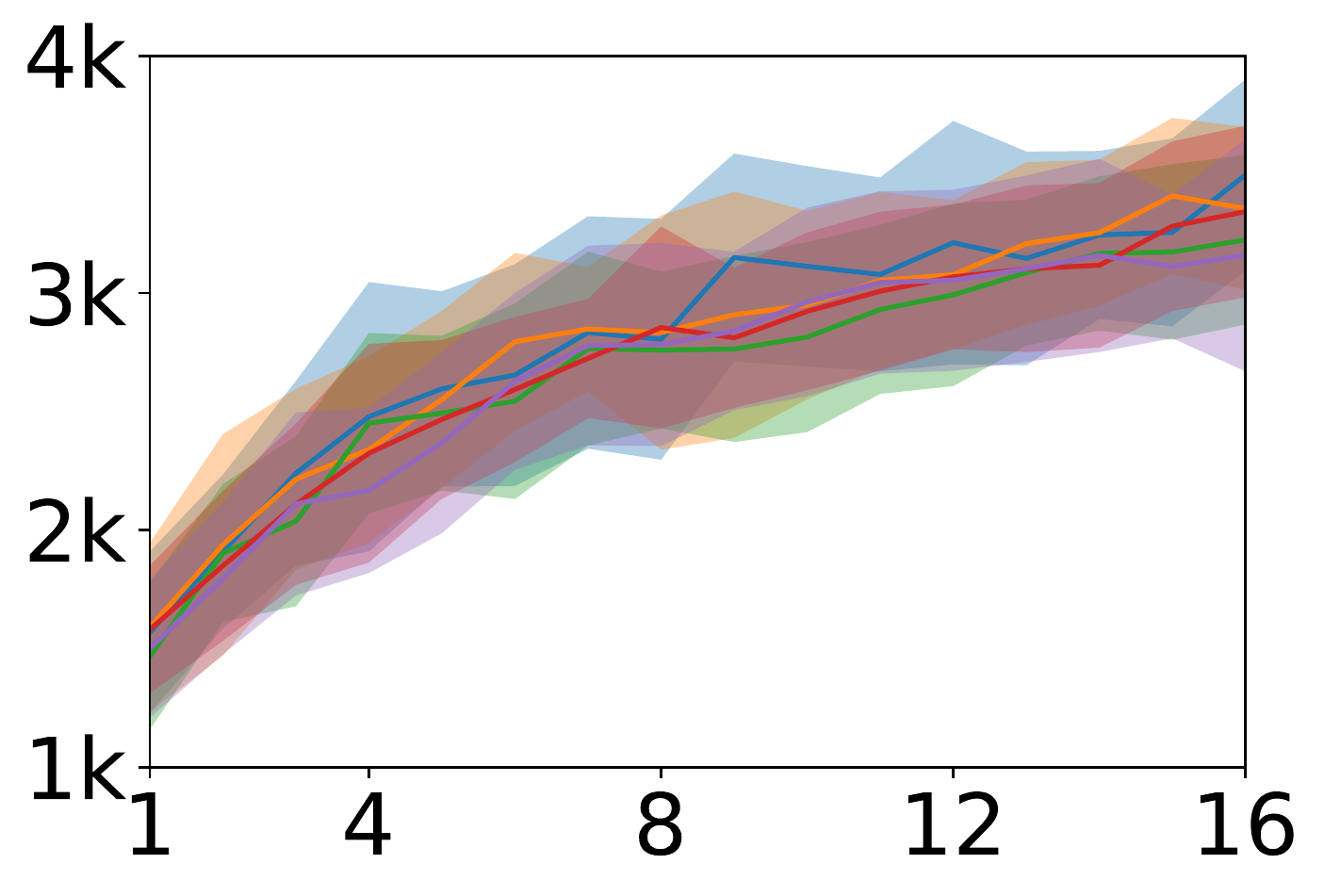}}}
    \caption{We plot the learning time of LRL as the number of episodes until convergence ($y$ axis) against the number of reward signals ($x$ axis). We use randomly generated MOMDPs with 256 or 512 states, four actions, and a varying number of rewards. For each trial we generate 30 MOMDPs, and record the number of episodes it takes for each
    agent's long-run average reward per episode to converge to a stable value. 
    The algorithms are Lexicographic $Q$-Learning (blue), Lexicographic Expected SARSA (orange), Lexicographic Double $Q$-Learning (green), Lexicographic PPO (red), and Lexicographic A2C (purple).
    }
    \label{fig:scaling_in_reward}
\end{figure}

\section{Experiments}
\label{sec:experiments}
In this section we evaluate our algorithms empirically. 
We first show how the learning time of LRL scales with the number of reward functions. We then compare the performance of VB-LRL and PB-LRL against that of other algorithms for solving constrained RL problems.
Further experimental details and additional experiments are described in the supplementary material, and documented in our codebase.\footnote{Available at \url{https://github.com/lrhammond/lmorl}.}


\subsection{Scaling with the Number of Rewards}
\label{sec:experiments:1}

Our first experiment (shown in Figure \ref{fig:scaling_in_reward}) shows how the learning time of LRL scales in the number of rewards. 
The data suggest that the learning time grows sub-linearly as additional reward functions are added, meaning that our algorithms can be used with large numbers of objectives.

%
%

\subsection{Lexicographic RL for Safety Constraints} 
\label{sec:experiments:2}



Many tasks are naturally expressed in terms of both a \textit{performance metric} and a \textit{safety constraint}. Our second experiment compares the performance of LRL against RCPO \cite{Tessler2019}, AproPO \cite{miryoosefi2019reinforcement}, and the actor-critic algorithm for VaR-constraints in \cite{chow2015riskconstrained}, in a number of environments with both a performance metric and a safety constraint. These algorithms synthesise slightly different kinds of policies, but are nonetheless sufficiently similar for a relevant comparison to be made. We use VB-LRL with a neural network and a replay buffer, which we call LDQN, and the PB-LRL algorithms we evaluate are LA2C and LPPO. The results are shown in Figure~\ref{fig:all}.


The CartSafe environment from \textit{gym-safety}\footnote{Available at \url{https://github.com/jemaw/gym-safety}.} is a version of the classic CartPole environment. 
The agent receives more reward the higher up the pole is, whilst incurring a cost if the cart is moved outside a safe region. Here the LRL algorithms, RCPO, and AproPO all learn quite safe policies, but VaR\_AC struggles. 
Of the safer policies LDQN gets the most reward (roughly matching DQN and A2C), followed by RCPO and AproPO, and then LA2C and LPPO. The latter two minimise cost more aggressively, and thus gain less reward. 


\begin{figure}[H]
  \centering
    \subfloat[CartSafe Reward]{{\includegraphics[width=0.24\textwidth]{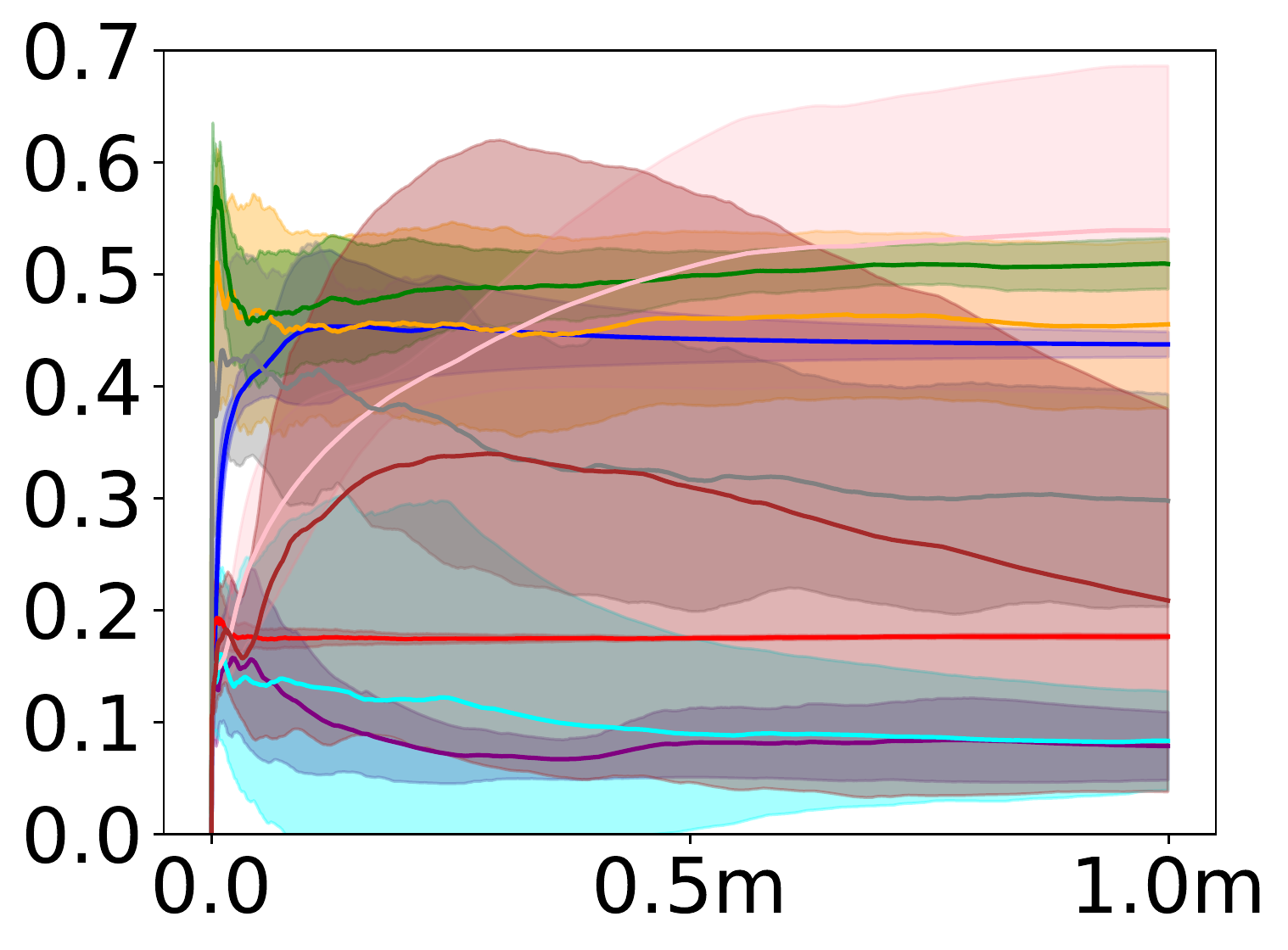}}}
    \subfloat[CartSafe Cost]{{\includegraphics[width=0.24\textwidth]{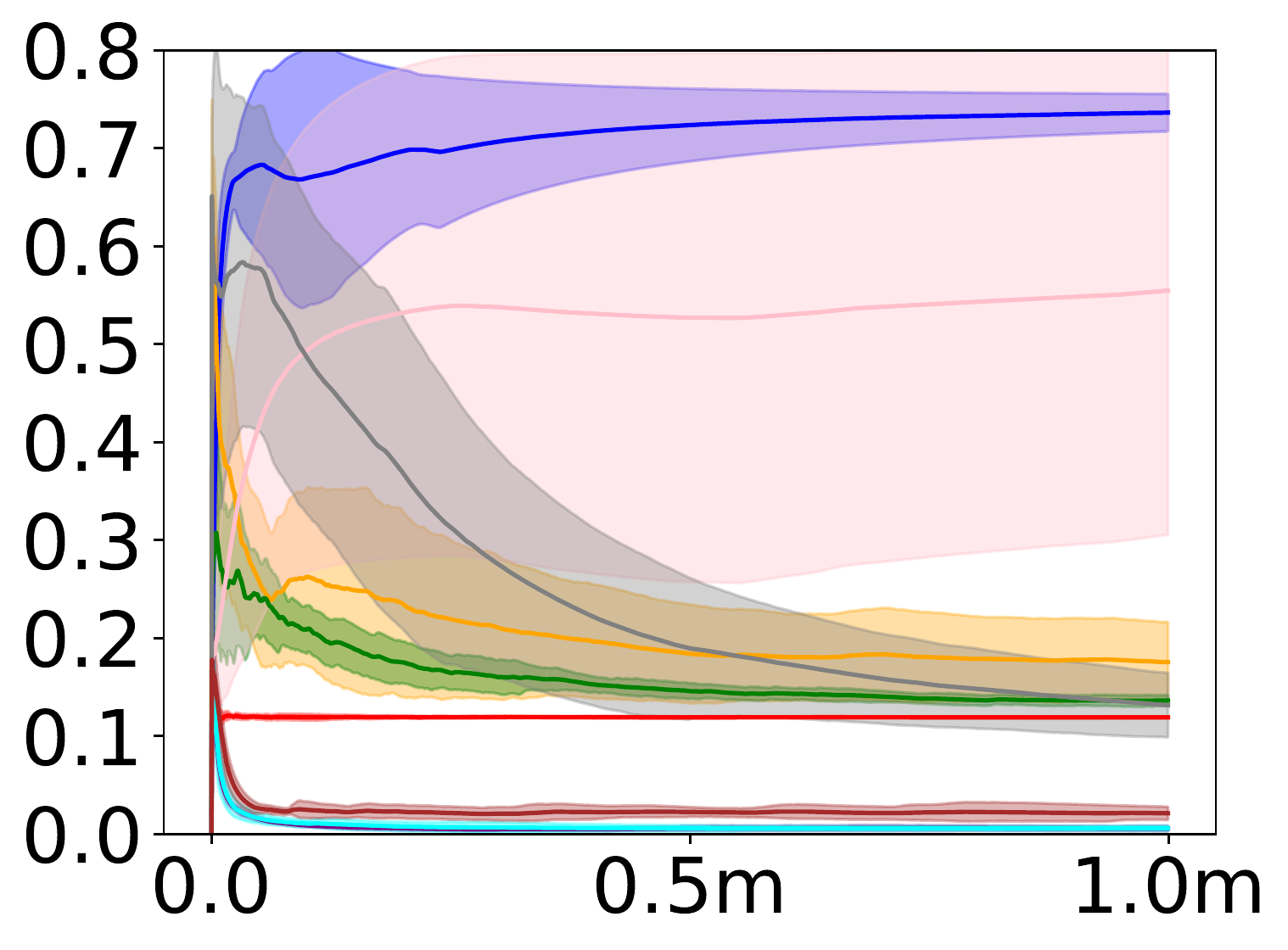}}}
    
    \subfloat[GridNav Reward]{{\includegraphics[width=0.24\textwidth]{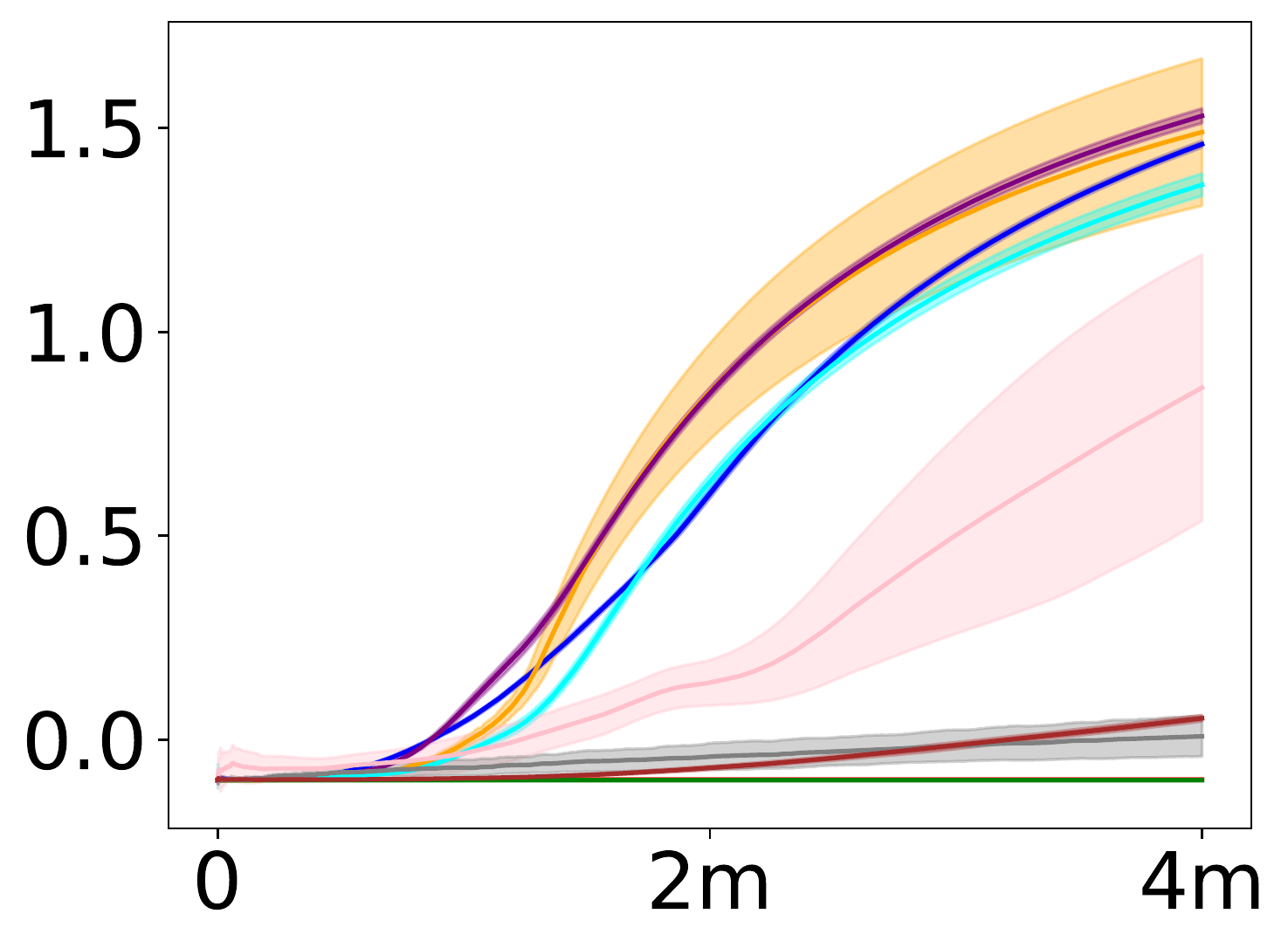}}}
    \subfloat[GridNav Cost]{{\includegraphics[width=0.24\textwidth]{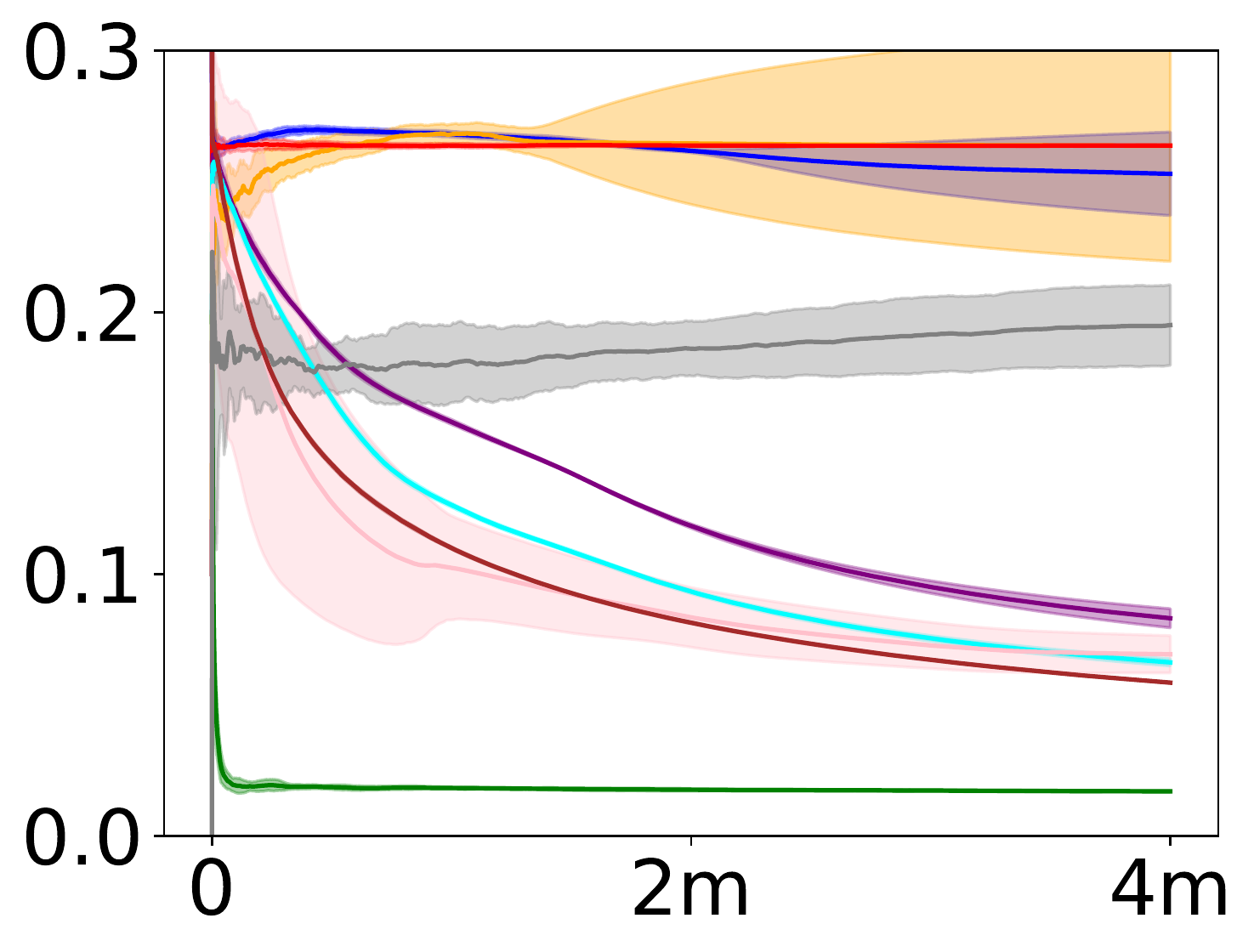}}}
    
    \subfloat[IntersectionEnv Reward]{{\includegraphics[width=0.24\textwidth]{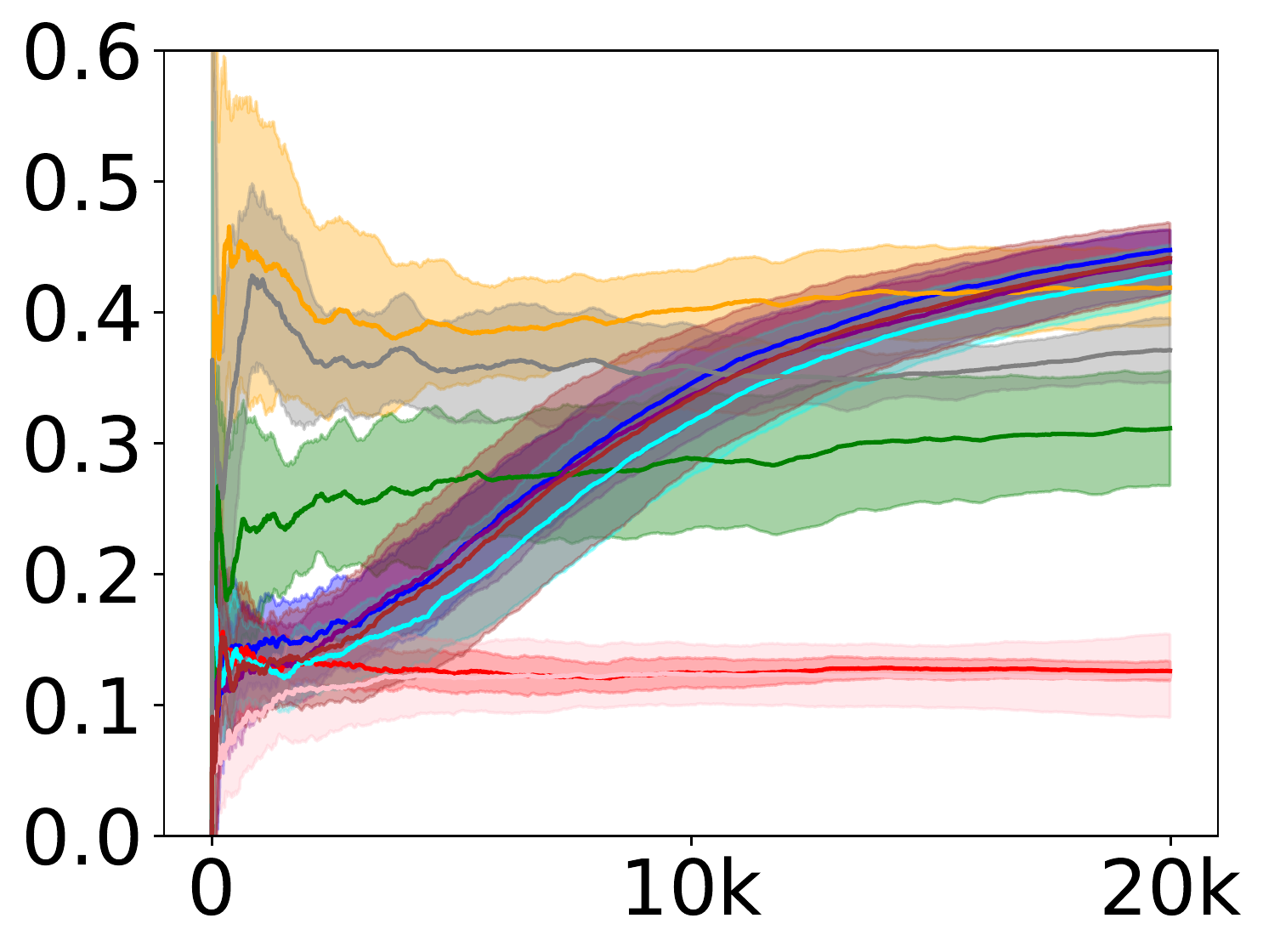}}}
    \subfloat[IntersectionEnv Cost]{{\includegraphics[width=0.24\textwidth]{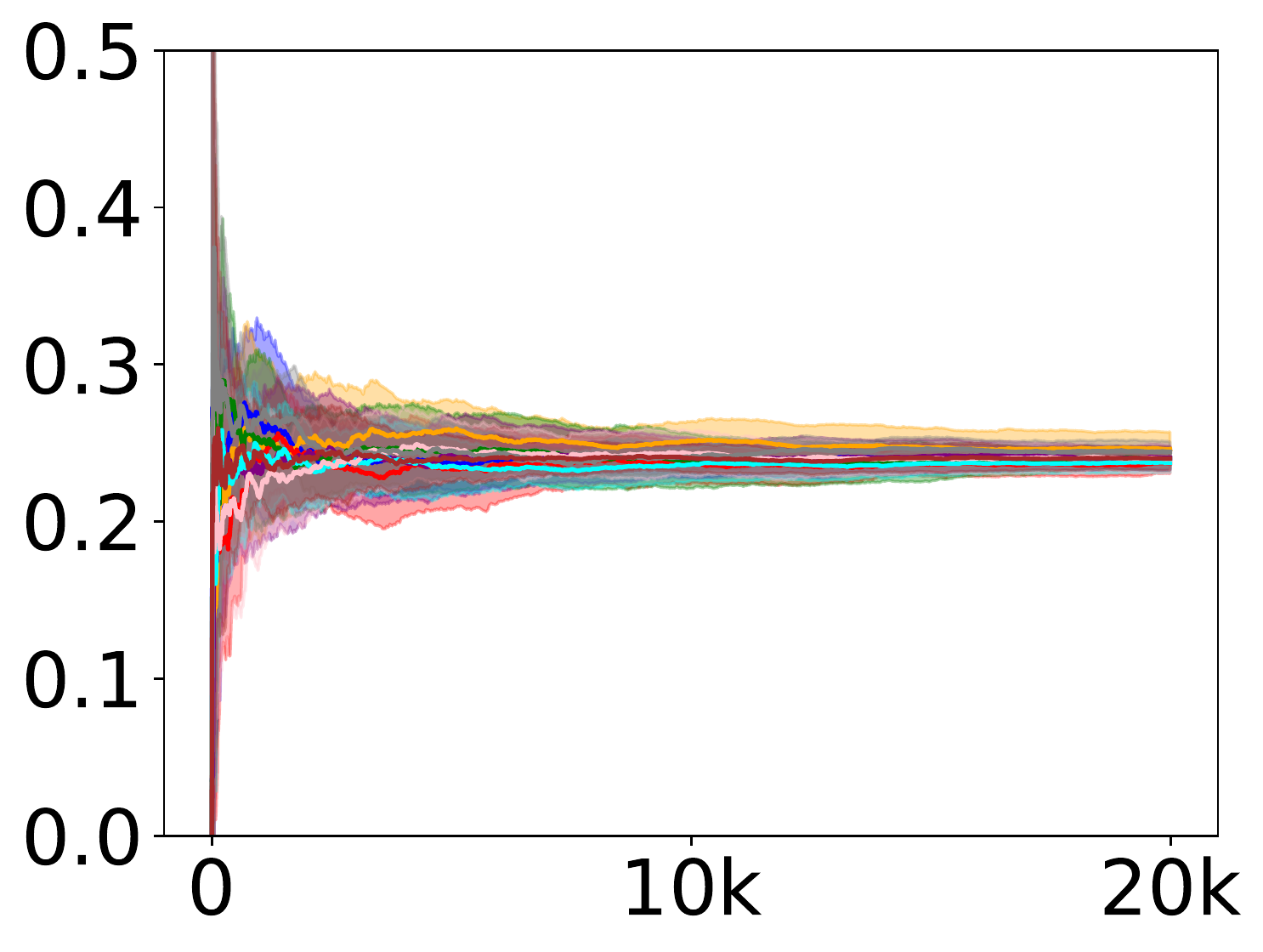}}}
    
     \subfloat{{\includegraphics[width=0.48\textwidth]{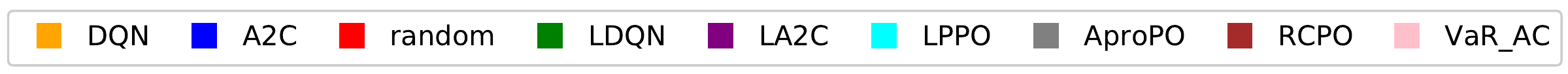}}}
    
    \caption{
    We plot the average reward and cost ($y$ axis) against the number of environment interactions ($x$ axis).
    In each environment, RCPO, AproPO, and VaR\_AC were tasked with maximising reward subject to a constraint on the cost, bounds on both the reward and cost, or a bound on the probability of the cost exceeding a certain constant.
    The LRL algorithms were tasked with minimising cost and, subject to that, maximising reward. Each algorithm was run ten times in each environment. 
    }
    \label{fig:all}
\end{figure}

The GridNav environment, again from \textit{gym-safety} (based on an environment in \cite{chow2018lyapunovbased}), is a large gridworld with a goal region and a number of \enquote{unsafe} squares. The agent is rewarded for reaching the goal quickly, and incurs a cost if it enters an unsafe square. Moreover, at each time step, the agent is moved in a random direction with probability 0.1. Here LDQN is the safest algorithm, but it also fails to obtain any reward. LA2C, LPPO, RCPO, and VaR\_AC are similar in terms of safety, but LA2C and LPPO obtain the most reward, VaR\_AC a fairly high reward, and RCPO a low reward. AproPO has low safety \emph{and} low reward. 

Finally, in the Intersection environment from \textit{highway-env}\footnote{Available at \url{https://github.com/eleurent/highway-env}.} the agent must guide a car through an intersection with dense traffic. We give the agent a reward of 10 if it reaches its destination, and a cost of 1 for each collision that occurs (which is slightly different from the environment's original reward structure). This task is challenging, and all the algorithms incur approximately the same cost as a random agent. However, they still manage to increase their reward, with LA2C and RCPO obtaining the most reward out of the constrained algorithms (roughly matching that of DQN and A2C). This shows that if optimising the first objective is too difficult, then the LRL algorithms fail gracefully by optimising the second objective, even if it has lexicographically lower priority.

\section{Discussion and Conclusions} 


We introduced two families of RL algorithms for solving lexicographic multi-objective problems, which are more general than prior work, and are justified both by their favourable theoretical guarantees and their compelling empirical performance against other algorithms for constrained RL. VB-LRL converges to a lexicographically optimal policy in the tabular setting, and PB-LRL inherits convergence guarantees as a function of the objectives used, leading to locally and globally lexicographically $\epsilon$-optimal policies in the case of LA2C and LPPO respectively. The learning time of the algorithms grows sub-linearly as reward functions are added, which is an encouraging result for scalability to larger problems. Further, when used to impose safety constraints, the LRL algorithms generally compare favourably to the state of the art, both in terms of learning speed and final performance.


We conclude by noting that in many situations, LRL may be preferable to constrained RL for reasons beyond its strong performance, as it allows one to solve different kinds of problems. For example, we might want a policy that is as safe as possible, but lack prior knowledge of what level of safety can be attained in the environment. LRL could also be used e.g.\ to guide learning by encoding prior knowledge in extra reward signals without the risk of sacrificing optimality with respect to the primary objective(s). These applications, among others, provide possible directions for future work.

\section*{Acknowledgments} 
Hammond acknowledges the support of an EPSRC Doctoral Training Partnership studentship (Reference: 2218880). 


\bibliographystyle{named}
\bibliography{refs}

\begin{thebibliography}{}

\bibitem[\protect\citeauthoryear{Achiam \bgroup \em et al.\egroup
  }{2017}]{achiam2017constrained}
Joshua Achiam, David Held, Aviv Tamar, and Pieter Abbeel.
\newblock Constrained policy optimization.
\newblock In {\em Proceedings of the 34th International Conference on Machine
  Learning}, pages 22--31, 2017.

\bibitem[\protect\citeauthoryear{Amari}{1998}]{Amari1998}
Shun-ichi Amari.
\newblock Natural gradient works efficiently in learning.
\newblock {\em Neural Computation}, 10(2):251--276, 1998.

\bibitem[\protect\citeauthoryear{Ben-Israel and Mond}{1986}]{BenIsrael1986}
A.~Ben-Israel and B.~Mond.
\newblock What is invexity?
\newblock {\em The Journal of the Australian Mathematical Society. Series B.
  Applied Mathematics}, 28(1):1--9, 1986.

\bibitem[\protect\citeauthoryear{Bertsekas}{1999}]{Bertsekas1999}
Dimitri Bertsekas.
\newblock {\em Nonlinear Programming}.
\newblock Athena Scientific, 1999.

\bibitem[\protect\citeauthoryear{Bhatnagar \bgroup \em et al.\egroup
  }{2009}]{Bhatnagar2009}
Shalabh Bhatnagar, Richard~S. Sutton, Mohammad Ghavamzadeh, and Mark Lee.
\newblock Natural actor–critic algorithms.
\newblock {\em Automatica}, 45(11):2471--2482, 2009.

\bibitem[\protect\citeauthoryear{Borkar}{2008}]{Borkar2008}
Vivek~S. Borkar.
\newblock {\em Stochastic Approximation}.
\newblock Hindustan Book Agency, 2008.

\bibitem[\protect\citeauthoryear{Chow \bgroup \em et al.\egroup
  }{2017}]{chow2015riskconstrained}
Yinlam Chow, Mohammad Ghavamzadeh, Lucas Janson, and Marco Pavone.
\newblock Risk-constrained reinforcement learning with percentile risk
  criteria.
\newblock {\em Journal of Machine Learning Research}, 18(1):6070--6120, 2017.

\bibitem[\protect\citeauthoryear{Chow \bgroup \em et al.\egroup
  }{2018}]{chow2018lyapunovbased}
Yinlam Chow, Ofir Nachum, Edgar Duenez-Guzman, and Mohammad Ghavamzadeh.
\newblock A lyapunov-based approach to safe reinforcement learning.
\newblock In {\em Proceedings of the 32nd International Conference on Neural
  Information Processing Systems}, pages 8103--8112, 2018.

\bibitem[\protect\citeauthoryear{Craven}{1981}]{Craven1981}
B.D. Craven.
\newblock Invex functions and constrained local minima.
\newblock {\em Bulletin of the Australian Mathematical Society},
  24(3):357--366, 1981.

\bibitem[\protect\citeauthoryear{Gábor \bgroup \em et al.\egroup
  }{1998}]{gabor1998}
Zoltán Gábor, Zsolt Kalmár, and Csaba Szepesvári.
\newblock Multi-criteria reinforcement learning.
\newblock In {\em Proceedings of the 15th International Conference on Machine
  Learning}, pages 197--205, 1998.

\bibitem[\protect\citeauthoryear{Hanson}{1981}]{Hanson1981}
Morgan~A Hanson.
\newblock On sufficiency of the kuhn-tucker conditions.
\newblock {\em Journal of Mathematical Analysis and Applications},
  80(2):545--550, 1981.

\bibitem[\protect\citeauthoryear{Hasselt}{2010}]{vanHasselt2010}
Hado~V. Hasselt.
\newblock Double q-learning.
\newblock In {\em Proceedings of the 24th International Conference on Neural
  Information Processing Systems}, pages 2613--2621, 2010.

\bibitem[\protect\citeauthoryear{Hsu \bgroup \em et al.\egroup
  }{2020}]{Hsu2020}
Chloe Ching-Yun Hsu, Celestine Mendler-Dünner, and Moritz Hardt.
\newblock Revisiting design choices in proximal policy optimization.
\newblock {\em arXiv:2009.10897}, 2020.

\bibitem[\protect\citeauthoryear{Kakade}{2001}]{Kakade2001}
Sham Kakade.
\newblock A natural policy gradient.
\newblock In {\em Proceedings of the 14th International Conference on Neural
  Information Processing Systems}, pages 1531--1538, 2001.

\bibitem[\protect\citeauthoryear{Konda and Tsitsiklis}{2000}]{Konda2000}
Vijay~R. Konda and John~N. Tsitsiklis.
\newblock Actor-critic algorithms.
\newblock In {\em Proceedings of the 13th International Conference on Neural
  Information Processing Systems}, pages 1008--1014, 2000.

\bibitem[\protect\citeauthoryear{Lesser and Abate}{2018}]{LA18}
K.~Lesser and A.~Abate.
\newblock Multi-objective optimal control with safety as a priority.
\newblock {\em IEEE Transactions on Control Systems Technology},
  26(3):1015--1027, 2018.

\bibitem[\protect\citeauthoryear{{Liu} \bgroup \em et al.\egroup
  }{2015}]{Liu2015}
C.~{Liu}, X.~{Xu}, and D.~{Hu}.
\newblock Multiobjective reinforcement learning: A comprehensive overview.
\newblock {\em IEEE Transactions on Systems, Man, and Cybernetics: Systems},
  45(3):385--398, 2015.

\bibitem[\protect\citeauthoryear{Liu \bgroup \em et al.\egroup
  }{2019}]{Liu2019b}
Boyi Liu, Qi~Cai, Zhuoran Yang, and Zhaoran Wang.
\newblock Neural trust region/proximal policy optimization attains globally
  optimal policy.
\newblock In {\em Proceedings of the 33rd International Conference on Neural
  Information Processing Systems}, pages 10564--10575, 2019.

\bibitem[\protect\citeauthoryear{Miryoosefi \bgroup \em et al.\egroup
  }{2019}]{miryoosefi2019reinforcement}
Sobhan Miryoosefi, Kiant{\'{e}} Brantley, Hal~Daum{\'{e}} III, Miroslav
  Dud{\'{\i}}k, and Robert~E. Schapire.
\newblock Reinforcement learning with convex constraints.
\newblock In {\em Proceedings of the 33rd International Conference on Neural
  Information Processing Systems}, pages 14070--14079, 2019.

\bibitem[\protect\citeauthoryear{Mitten}{1974}]{Mitten1974}
L.~G. Mitten.
\newblock Preference order dynamic programming.
\newblock {\em Management Science}, 21(1):43--46, 1974.

\bibitem[\protect\citeauthoryear{Paternain \bgroup \em et al.\egroup
  }{2019}]{Paternain2019}
Santiago Paternain, Luiz F.~O. Chamon, Miguel Calvo{-}Fullana, and Alejandro
  Ribeiro.
\newblock Constrained reinforcement learning has zero duality gap.
\newblock In {\em Proceedings of the 33rd International Conference on Neural
  Information Processing Systems}, pages 7553--7563, 2019.

\bibitem[\protect\citeauthoryear{Polymenakos \bgroup \em et al.\egroup
  }{2019}]{PAR19}
K.~Polymenakos, A.~Abate, and S.~Roberts.
\newblock Safe policy search using gaussian process models.
\newblock In {\em Proceedings of the 18th International Conference on
  Autonomous Agents and Multiagent Systems}, pages 1565--1573, 2019.

\bibitem[\protect\citeauthoryear{Rentmeesters \bgroup \em et al.\egroup
  }{1996}]{Rentmeestersa1996}
M.J. Rentmeesters, W.K. Tsai, and Kwei-Jay Lin.
\newblock A theory of lexicographic multi-criteria optimization.
\newblock In {\em Proceedings of the 2nd {IEEE} International Conference on
  Engineering of Complex Computer Systems}, 1996.

\bibitem[\protect\citeauthoryear{Roijers \bgroup \em et al.\egroup
  }{2013}]{Roijers2013}
D.~M. Roijers, P.~Vamplew, S.~Whiteson, and R.~Dazeley.
\newblock A survey of multi-objective sequential decision-making.
\newblock {\em Journal of Artificial Intelligence Research}, 48:67--113, 2013.

\bibitem[\protect\citeauthoryear{Rummery and Niranjan}{1994}]{Rummery1994}
Gavin Rummery and Mahesan Niranjan.
\newblock On-line q-learning using connectionist systems.
\newblock Technical report, University of Cambridge, 1994.

\bibitem[\protect\citeauthoryear{Schulman \bgroup \em et al.\egroup
  }{2017}]{Schulman2017}
John Schulman, Filip Wolski, Prafulla Dhariwal, Alec Radford, and Oleg Klimov.
\newblock Proximal policy optimization algorithms.
\newblock {\em arXiv:1707.06347}, 2017.

\bibitem[\protect\citeauthoryear{Singh \bgroup \em et al.\egroup
  }{2000}]{Singh2000}
Satinder Singh, Tommi Jaakkola, Michael~L. Littman, and Csaba Szepesvári.
\newblock Convergence results for single-step on-policy reinforcement-learning
  algorithms.
\newblock {\em Machine Learning}, 38:287--308, 2000.

\bibitem[\protect\citeauthoryear{Slater}{1950}]{Slater1950}
Morton Slater.
\newblock Lagrange multipliers revisited.
\newblock Cowles Commission Discussion Paper No. 403, 1950.

\bibitem[\protect\citeauthoryear{Sutton and Barto}{2018}]{SuttonAndBarto}
Richard~S. Sutton and Andrew~G. Barto.
\newblock {\em Reinforcement Learning: An Introduction}.
\newblock The MIT Press, 2018.

\bibitem[\protect\citeauthoryear{Sutton \bgroup \em et al.\egroup
  }{1999}]{Sutton1999}
Richard~S. Sutton, David McAllester, Satinder Singh, and Yishay Mansour.
\newblock Policy gradient methods for reinforcement learning with function
  approximation.
\newblock In {\em Proceedings of the 12th International Conference on Neural
  Information Processing Systems}, pages 1057--1063, 1999.

\bibitem[\protect\citeauthoryear{Tessler \bgroup \em et al.\egroup
  }{2019}]{Tessler2019}
Chen Tessler, Daniel~J. Mankowitz, and Shie Mannor.
\newblock Reward constrained policy optimization.
\newblock In {\em Proceedings of the 7th International Conference on Learning
  Representations}, 2019.

\bibitem[\protect\citeauthoryear{Thomas \bgroup \em et al.\egroup
  }{2013}]{Thomas2013}
Philip~S. Thomas, William Dabney, Sridhar Mahadevan, and Stephen Giguere.
\newblock Projected natural actor-critic.
\newblock In {\em Proceedings of the 26th International Conference on Neural
  Information Processing Systems}, pages 2337--2345, 2013.

\bibitem[\protect\citeauthoryear{Thomas}{2014}]{Thomas2014}
Philip~S. Thomas.
\newblock Bias in natural actor-critic algorithms.
\newblock In {\em Proceedings of the 31st International Conference on
  International Conference on Machine Learning}, pages 441--448, 2014.

\bibitem[\protect\citeauthoryear{Tsitsiklis and {Van
  Roy}}{1997}]{Tsitsiklis1997}
J.N. Tsitsiklis and B.~{Van Roy}.
\newblock An analysis of temporal-difference learning with function
  approximation.
\newblock {\em {IEEE} Transactions on Automatic Control}, 42(5):674--690, 1997.

\bibitem[\protect\citeauthoryear{van Seijen \bgroup \em et al.\egroup
  }{2009}]{Seijen2009}
Harm van Seijen, Hado van Hasselt, Whiteson Shimon, and Marco Wiering.
\newblock A theoretical and empirical analysis of expected sarsa.
\newblock {\em IEEE Symposium on Adaptive Dynamic Programming and Reinforcement
  Learning}, pages 177--184, 2009.

\bibitem[\protect\citeauthoryear{Watkins}{1986}]{watkins1989}
Chris Watkins.
\newblock {\em Learning from Delayed Rewards}.
\newblock PhD thesis, University of Cambridge, 1986.

\bibitem[\protect\citeauthoryear{Wray and
  Zilberstein}{2015}]{Wray2015MultiObjectivePW}
Kyle~Hollins Wray and Shlomo Zilberstein.
\newblock Multi-objective pomdps with lexicographic reward preferences.
\newblock In {\em Proceedings of the 24th International Conference on
  Artificial Intelligence}, pages 1719--1725, 2015.

\end{thebibliography}

\setcounter{proposition}{0}
\setcounter{theorem}{0}
\setcounter{corollary}{0}
\setcounter{definition}{0}

\newpage



\section{General Form of VB-LRL}

There are several places where VB-LRL makes use of a tolerance parameter $\tau$. In the main text of this paper, we assume that the same tolerance parameter is used everywhere, and that it is a constant. Here we provide a more general form of VB-LRL, which is also the form that will be used in our proofs. First, let a \emph{tolerance function} $\tau : S \times \{1 \dots m\} \times \mathbb{N} \times (S \times A \rightarrow \mathbb{R})^m \rightarrow \mathbb{R}_{> 0}$ be a function that takes as input a state, reward priority, time step, and sequence of $Q$-functions, and returns a (positive) variable corresponding to a tolerance. In other words, $\tau(s,i,t,Q_1, \dots, Q_m)$ is the tolerance for the $i^\text{th}$ reward when selecting an action in state $s$ at time $t$, given the current $Q$-functions. 
We also require that $\lim_{t \rightarrow \infty}\tau(s,i,t,Q_1, \dots, Q_m)$ exists.
We will revisit the specification of VB-LRL using such tolerance functions, starting with a more general definition of lexicographic bandit algorithms.


\begin{definition}[Lexicographic Bandit Algorithm] Let $S$ be a set of states, $A$ a set of actions, $Q_1, \dots, Q_m : S \times A \rightarrow \mathbb{R}$ a sequence of $Q$-functions, and $t \in \mathbb{N}$ a time parameter.
Then a lexicographic bandit algorithm with tolerance function $\tau_\mathcal{B} : S \times \{1 \dots m\} \times \mathbb{N} \times (S \times A \rightarrow \mathbb{R})^m \rightarrow \mathbb{R}_{> 0}$ is a function $\mathcal{B} : (S \times A \rightarrow \mathbb{R})^m \times S \times \mathbb{N} \rightsquigarrow A$, such that: 
$$
\lim_{t\to\infty} \Pr(\mathcal{B}(Q_1, \dots, Q_m,s,t) \in \Delta^{\tau_\mathcal{B}}_{s,m}) = 1,
$$
where $\Delta^{\tau_\mathcal{B}}_{s,0} = A$ and $\Delta^{\tau_\mathcal{B}}_{s,i+1} \coloneqq \{ a \in \Delta^{\tau_\mathcal{B}}_{s,i} \mid Q_i(s,a) \geq \max_{a' \in \Delta^{\tau_\mathcal{B}}_{s,i}} Q_i(s,a')- \lim_{t\to\infty} {\tau_\mathcal{B}}(s,i+1,t,Q_1, \dots, Q_m)\}$.
\end{definition}

We next generalise the definition of lexicographic $Q$-learning in a similar way. 
Given a tolerance function $\tau_Q$, we say that lexicographic $Q$-learning with tolerance $\tau_Q$ is the update rule:
\begin{align*}
Q_i(s, a) \gets \big(&1-\alpha_t(s,a)\big)\cdot Q_i(s, a) ~ + \\
&\alpha_t(s,a) \cdot \big(R_i(s,a,s') + \gamma_i \max_{a \in \Delta^{\tau_Q,t}_{s',i}} Q_i(s', a)\big), 
\end{align*}
where $\Delta^{\tau_Q,t}_{s,0} = A$ and $\Delta^{\tau_Q,t}_{s,i+1} \coloneqq \{ a \in \Delta^{\tau_Q,t}_{s,i} \mid Q_i(s,a) \geq \max_{a' \in \Delta^{\tau_Q,t}_{s,i}} Q_i(s,a')- \tau_Q(s,i+1,t,Q_1, \dots, Q_m)\}$.
Note that $\tau_\mathcal{B}$ and $\tau_Q$ may, but need not, be the same.



These generalisations may seem excessive, but they let us set the tolerance in a much more flexible way.
For example, if we do not know the scale of the reward, then it might be beneficial to express the tolerance as a \emph{proportion} rather than a constant, such as by setting $\tau(s,i,t,Q_1, \dots, Q_m) \coloneqq \sigma \cdot |\max_{a \in \Delta} Q_i(s,a)|$, where $\sigma \in (0,1)$. 
If we do this, then the tolerance will depend on $s$, $i$, and $Q_1, \dots, Q_m$.
We might also wish  $\sigma$ to decrease over time, in which case the tolerance will additionally depend on $t$. To give another example, we could implement the tolerance by casting the $Q$-values to low-precision floats, letting the rounding error correspond to the tolerance.
The benefit of the more general tolerance functions we define above is that they easily can capture these (and similar) kinds of setups.

\section{VB-LRL Convergence Proof}

We now prove that VB-LRL converges to a lexicographically optimal policy, making use of the following lemma.



\begin{lemma}
\label{lemma:convergence_lemma}
Let $\langle \zeta_t , \beta_t, F_t \rangle$ 
be a stochastic process where $\zeta_t , \beta_t, F_t : X \rightarrow \mathbb{R}$ satisfy:
$$
\beta_{t+1}(x) = \big(1-\zeta_t(x_t)\big)\cdot \beta_t(x_t)+\zeta_t(x_t)\cdot F_t(x_t),
$$
with $x_t \in X$ and $t \in \mathbb{N}$. Let $P_t$ be a sequence of increasing $\sigma$-fields  such  that $\zeta_0$ and $\beta_0$ are $P_0$-measurable and $\zeta_t$, $\beta_t$ and $F_{t-1}$ are $P_t$-measurable, and $t\geq1$. Then $\beta_t$ converges to $0$ with probability one if:
\begin{enumerate}
    \item $X$ is finite,
    \item $\zeta_t(x_t) \in [0,1]$ and $\forall x \neq x_t$ we have $\zeta_t(x) = 0$,
    \item $\sum_t \zeta_t(x_t) = \infty$ and $\sum_t \zeta_t(x_t)^2 < \infty$ with probability one,
    \item $\mathrm{Var}\{F_t(x_t) \mid P_t\} \leq K(1+\kappa \Vert \beta_t \Vert_\infty )^2$ for some $K \in \mathbb{R}$ and $\kappa \in [0,1)$,
    \item $\Vert \mathbb{E}\{F_t \mid P_t\}\Vert_\infty \leq \kappa \Vert \beta_t \Vert_\infty + c_t$, where $c_t \rightarrow 0$ with probability one as $t \rightarrow \infty$,
\end{enumerate}
where $\Vert \cdot \Vert_\infty$ is a (potentially weighted) maximum norm.  
\end{lemma}

\begin{proof}
See \cite{Singh2000}.
\end{proof}


We next give the general statement of the Theorem:

\begin{theorem}
\label{thm:convergence}
In any MOMDP $\mathcal{M}$, if VB-LRL uses SARSA, Expected SARSA, or Lexicographic $Q$-Learning, then it will converge to a policy $\pi$ that is lexicographically optimal if:
\begin{enumerate}
    \item $S$ and $A$ are finite,
    \item All reward functions are bounded,
    \item Either $\gamma_1 , \dots,  \gamma_m < 1$, or every policy leads to a terminal state with probability one,
    \item The learning rates $\alpha_t(s,a) \in [0,1]$ satisfy the conditions $\sum_t \alpha_t(s,a) = \infty$ and $\sum_t \alpha_t(s,a)^2 < \infty$ with probability one, for all $s\in S$, $a \in A$,
    \item The tolerance function $\tau_\mathcal{B}$ used in the lexicographic bandit algorithm satisfies the condition that $0 < \lim_{t \rightarrow \infty} \tau_\mathcal{B}(s,i,t,Q_1, \dots, Q_m) < \min_{a \neq a'} \vert q_i(s,a) - q_i(s,a') \vert$; if lexicographic $Q$-learning is used then its tolerance function $\tau_\mathcal{Q}$ must also satisfy this condition. 
\end{enumerate}
\end{theorem}
Note that condition 4 requires that the agent visits every state-action pair infinitely many times. 
Before giving the proof, we need to provide a few definitions.
We say that the $Q$-functions $Q_1 , \dots,  Q_m$ are \textit{lexicographically optimal} if they correspond to the $q$-values of (globally) lexicographically optimal policies, and we will also use $q_1^l , \dots,  q_m^l$ to denote these $Q$-functions. We define \textit{lexmax} as:
$$
\lexmax_a q_i^l(s,a) = \max_{a \in \Delta_{s,i}} q_i^l(s,a),
$$
with $\Delta_{s,1} = A$ and $\Delta_{s,i+1} = \argmax_{a \in \Delta_{s,i}} q_{i+1}(s,a)$. 
Moreover, given a function $\sigma : S \times \{1 \dots m\} \times (S \times A \rightarrow \mathbb{R})^m \rightarrow \mathbb{R}_{> 0}$ that takes as input a state, reward priority, and sequence of $Q$-functions, and returns a non-negative real value, we define $\sigma$-\textit{slack lexmax} as:
$$
\slexmax Q_i(s,a) = \max_{a \in \Delta_{s,i}^\sigma} Q_i(s,a),
$$
where $\Delta_{s,0}^\sigma = A$ and $\Delta_{s,i+1}^\sigma = \{ a \in \Delta_{s,i}^\sigma \mid Q_i(s,a) \geq \max_{a' \in \Delta_{s,i}^\sigma} Q_i(s,a') - \sigma(s,i,Q_1, \dots, Q_m)\}$. In other words, $\lexmax_a q_i^l(s,a)$ is the maximal value of $q_i^l$ at $s$ if $a$ is chosen from the actions that lexicographically maximise $q_1^l , \dots,  q_{i-1}^l$, and $\slexmax_a Q_i(s,a)$ is the maximal value of $Q_i$ at $s$ if $a$ is chosen from the actions that lexicographically maximise $Q_1 , \dots,  Q_{i-1}$ with tolerance $\sigma$. 
Note that $\sigma$ does not depend on the time $t$, whereas the tolerance functions ($\tau$) which we defined earlier are allowed to depend on $t$.
A $Q$-value update rule is a \textit{temporal difference} (TD) \textit{update rule} if it has the following form:
\begin{align*}
Q_{i}(s_t, a_t) \gets &(1-\alpha_t(s_t,a_t))\cdot Q_{i}(s, a) + \\ &\alpha_t(s_t,a_t) \cdot (r_{i,t} + \gamma_i \hat{v}_{n,t}(s_{t+1})),
\end{align*}
where $\hat{v}_{n,t}(s)$ estimates the value on $s$ under $R_n$. Note that Lexicographic $Q$-Learning, SARSA, and Expected SARSA all are TD update rules. We use $Q_{i,t}$ to refer to the agent's $Q$-function for the $i$'th reward at time $t$. Note that $q_i^l$ is the \enquote{true} $Q$-function for the $i$'th reward, whereas $Q_{i,t}$ is the agent's estimate of $q_i^l$ at time $t$. We can now finally give the proof.

\begin{proof}
The proof proceeds by induction on the reward signals. For $n\in \{ 0 \dots m\}$, let P($n$) signify that $Q_{1,t} , \dots,  Q_{n,t}$ converge to $q_1^l , \dots,  q_n^l$ as $t \rightarrow \infty$. The base case $\mathrm{P}(0)$ holds vacuously. As for the inductive step $\mathrm{P}(n) \rightarrow \mathrm{P}(n+1)$, let:
\begin{itemize}
    \item $X = S \times A$,
    \item $\zeta_t(s,a) = \alpha_t(s,a)$,
    \item $\beta_t(s,a) = Q_{n+1,t}(s,a) - q_{n+1}^l(s,a)$, 
    \item $F_t(s,a) = r_{n+1,t} + \gamma_{n+1} \hat{v}_{n+1,t+1}(s_{t+1}) - q_{n+1}^l(s,a)$. 
\end{itemize}
Note that $\hat{v}_{n+1,t+1}(s_{t+1})$ comes from the $Q$-value update rule. Now assumptions 1, 2, and 4 of this theorem imply conditions 1--4 in Lemma \ref{lemma:convergence_lemma}. It remains to show that condition 5 holds, which we can do algebraically. If Lexicographic $Q$-Learning is used then let :
$$
\sigma(s,i,Q_{1,t} \dots Q_{m,t}) = \lim_{t' \rightarrow \infty}\tau_Q(s,i,t',Q_{1,t} \dots Q_{m,t}),
$$ 
where $\tau_Q$ is the tolerance function used in the lexicographic $Q$-learning update rule. Otherwise, if SARSA or Expected SARSA is used, let 
$$
\sigma(s,i,Q_{1,t} \dots Q_{m,t}) = \lim_{t' \rightarrow \infty}\tau_\mathcal{B}(s,i,t',Q_{1,t} \dots Q_{m,t}),
$$
where $\tau_\mathcal{B}$ is the tolerance function used in the lexicographic bandit algorithm. 
It is important to note that $\sigma$ is defined as the limit of $\tau_Q$ or $\tau_\mathcal{B}$ when the time parameter \emph{that $\tau_Q$ or $\tau_\mathcal{B}$ takes as input} goes to $\infty$, and that this limit is \emph{not} applied to the other arguments of $\tau_Q$ or $\tau_\mathcal{B}$. In other words, it is not a typo that the $Q_{1,t} \dots Q_{m,t}$ are parameterised by $t$ rather than $t'$ (indeed, we have not yet shown that $\lim_{t \rightarrow \infty} Q_{i,t}$ even exists).
With this noted, we begin as follows:

\begin{align*}
     &\Vert \mathbb{E}\{F_t \mid P_t\} \Vert_\infty \\
     = \max_{s,a} &\Bigg| \mathbb{E}\Big[r_{n+1,t} + \gamma_{n+1} \hat{v}_{n+1,t+1}(s_{t+1}) - q_{n+1}^l(s,a) \Big]\Bigg|\\
     = \max_{s,a}&\Bigg| \mathbb{E}\Big[r_{n+1,t} + \gamma_{n+1} \slexmax_{a'} Q_{n+1,t+1}(s_{t+1},a') \\ & - q_{n+1}^l(s,a) + \gamma_{n+1}\hat{v}_{n+1,t+1}(s_{t+1}) \\
     &- \gamma_{n+1} \slexmax_{a'} Q_{n+1,t+1}(s_{t+1},a')\Big]\Bigg|\\
     \leq \max_{s,a}&\Bigg|\mathbb{E}\Big[r_{n+1,t} + \gamma_{n+1} \slexmax_{a'} Q_{n+1,t+1}(s_{t+1},a') \\
     &- q_{n+1}^l(s,a) \Big]\Bigg| + \\
      \max_{s,a}&\Bigg| \mathbb{E}\Big[\gamma_{n+1} \hat{v}_{n+1,t+1}(s_{t+1})- \\
     &\gamma_{n+1} \slexmax_{a'} Q_{n+1,t+1}(s_{t+1},a')\Big]\Bigg|.
\end{align*}
The second term is bounded above by:
$$
\max_{s}\Bigg| \mathbb{E}\Big[\hat{v}_{n+1,t+1}(s)- \slexmax_a Q_{n+1,t+1}(s,a)\Big]\Bigg|. 
$$
Let $k_t$ denote this expression --- we will argue that $k_t \rightarrow 0$ as $t \rightarrow \infty$. First recall that :
$$
\slexmax_a Q_{n+1,t+1}(s,a) = \max_{a \in \Delta_{s,n+1}^\sigma} Q_{n+1,t+1}(s,a).
$$
If VB-LRL uses Lexicographic $Q$-Learning (with $\tau_Q$) then :
$$
\hat{v}_{n+1,t+1}(s) = \max_{a \in \Delta^{\tau_Q,t+1}_{s,n+1}} Q_{n+1,t+1}(s, a),
$$
and also recall that:
\begin{alignat*}{1}
    &\sigma(s,n+1,Q_{1,t+1} \dots Q_{m,t+1}) \\
    = \lim_{t' \rightarrow \infty}&\tau_Q(s,n+1,t',Q_{1,t+1} \dots Q_{m,t+1}).
\end{alignat*}
If assumption 5 holds, and the inductive assumption holds, then this implies that there is a time $t'$ such that:
$$
\Delta_{s,n+1}^\sigma = \Delta^{\tau_Q,t''}_{s,n+1},
$$
for all $t'' \geq t'$, which also means that $k_{t''} = 0$ for all $t'' \geq t'$. Hence, if VB-LRL uses Lexicographic $Q$-Learning then $k_t \rightarrow 0$ as $t \rightarrow \infty$. 
If instead VB-LRL uses SARSA or Expected SARSA then we have that:
\begin{alignat*}{1}
    &\mathbb{E}\left[\hat{v}_{n+1,t+1}(s)\right] \\
    = &\mathbb{E}_{a \sim \mathcal{B}(Q_{1,t+1}, \dots, Q_{m,t+1},s,t+1)}\left[Q_{n+1,t+1}(s,a)\right],
\end{alignat*}
where $\mathcal{B}$ is the bandit algorithm that is used, 
and:
\begin{alignat*}{1}
    &\sigma(s,n+1,Q_{1,t+1} \dots Q_{m,t+1}) \\
    = \lim_{t' \rightarrow \infty}&\tau_\mathcal{B}(s,n+1,t',Q_{1,t+1} \dots Q_{m,t+1}),
\end{alignat*}
where $\tau_\mathcal{B}$ is the tolerance of $\mathcal{B}$.
From the properties of lexicographic bandit algorithms, we have that:
$$
\lim_{t\to\infty} \Pr(\mathcal{B}(Q_1, \dots, Q_m,s,t) \in \Delta^{\tau_\mathcal{B}}_{s,m}) = 1.
$$
Moreover, since in this case $\sigma = \lim_{t \rightarrow \infty} \tau_\mathcal{B}$, we straightforwardly have that: $\Delta^{\sigma}_{s,n+1} = \Delta^{\tau_\mathcal{B}}_{s,n+1}$. This means that:
$$
\lim_{t'\to\infty} \Pr(\hat{v}_{n+1,t'}(s) = \max_{a \in \Delta_{s,n+1}^\sigma} Q_{n+1,t+1}(s,a)) = 1.
$$
Since the action space $A$ is finite, this also means that
$$
\lim_{t'\to\infty} \mathbb{E}\left[\hat{v}_{n+1,t'}(s) - \max_{a \in \Delta_{s,n+1}^\sigma} Q_{n+1,t+1}(s,a)\right] = 0.
$$
In other words, if VB-LRL uses SARSA or Expected SARSA then $k_t \rightarrow 0$ as $t \rightarrow \infty$. 
This exhausts all cases, so we can conclude that $k_t \rightarrow 0$ as $t \rightarrow \infty$.
We thus have:
\begin{alignat*}{2}
    \Vert& \mathbb{E}\{F_t \mid P_t\} \Vert_\infty \\
    \leq \max_{s,a}\Bigg|&\mathbb{E}\Big[r_{n+1,t} \\
    &+ \gamma_{n+1} \slexmax_{a'} Q_{n+1,t}(s_{t+1},a') \\
    &- q_{n+1}^l(s,a) \Big]\Bigg| + k_t \\
    = \max_{s,a}\Bigg|&\mathbb{E}_{s' \sim T(s,a)} \Big[R_{n+1}(s,a,s') \\
    & + \gamma_{n+1} \slexmax_{a'} Q_{n+1,t}(s',a') \\
    & - R_{n+1}(s,a,s') - \gamma_{n+1} \lexmax_{a'} q_{n+1}^l(s',a') \Big]\Bigg|\\
    &+ k_t\\
    = \max_{s,a}\Bigg|& \mathbb{E}_{s' \sim T(s,a)} \Big[ \slexmax_{a'} Q_{n+1,t}(s',a')\\
    & - \lexmax_{a'} q_{n+1}^l(s',a') \Big] \Bigg| \cdot \gamma_{n+1} + k_t.
\end{alignat*}
Recall that:
$$
\lexmax_a q_{n+1}^l(s,a) = \max_{a \in \Delta_{s,n+1}} q_{n+1}^l(s,a),
$$
$$
\slexmax_a Q_{n+1,t}(s,a) = \max_{a \in \Delta_{s,n+1}^\sigma} Q_{n+1,t}(s,a).
$$
The inductive assumption states that $Q_{1,t} , \dots,  Q_{n,t}$ converge to $q_1^l , \dots,  q_n^l$ as $t \rightarrow \infty$. 
This, together with assumption 5, implies that there is a time $t'$ such that $\Delta_{s,n+1} = \Delta_{s,n+1}^\sigma$ for all $t'' \geq t'$.
This in turn implies that :
\begin{alignat*}{1}
  &\slexmax_{a} Q_{n+1,t''}(s,a) - \lexmax_{a} q_{n+1}^l(s,a) \\
  \leq &\max_a | Q_{n+1,t''}(s,a) - q_{n+1}^l(s,a) |,
\end{alignat*}
for all $t'' \geq t'$. We can therefore continue as follows:
\begin{align*}
    \leq & \gamma_{n+1} \max_{s,a}\Bigg| \mathbb{E}_{s' \sim T(s,a)}\Big[ \max_{a'} \Big| Q_{n+1,t}(s',a') \\
    & - q_{n+1}^l(s',a') \Big|
    + d_t(s') \Big] \Bigg| + k_t\\
    \leq & \gamma_{n+1} \max_{s,a} \mathbb{E}_{s' \sim T(s,a)} \Big[ \max_{a'} \Big| \beta(s', a')\Big|\Big] + c_t, 
\end{align*}
where $c_t = \max_s d_t(s) + k_t$, and $d_t(s)$ is an error correction term given by: 
\begin{align*}
    d_t(s) = \max \big(0, &\slexmax_{a} Q_{n+1,t}(s,a) \\
    &- \lexmax_{a} q_{n+1}^l(s,a)\\ 
    &- \max_a | Q_{n+1,t}(s,a) - q_{n+1}^l(s,a)|\big). 
\end{align*}
Note that $d_t(s)$ is $0$ after $t'$, so $c_t \rightarrow 0$ as $t \rightarrow \infty$. Now if $\gamma_{n+1} < 1$ we furthermore have:
\begin{align*}
    \leq & \gamma_{n+1} \max_{s,a}\left| \beta(s,a) \right| + c_t = \gamma_{n+1} \Vert \beta_t \Vert_\infty + c_t. 
\end{align*}
This means that:
$$
\Vert \mathbb{E}\{F_t \mid P_t\} \Vert_\infty \leq \gamma_{n+1} \Vert \beta_t \Vert_\infty + c_t, 
$$
where $\gamma_{n+1} \in [0,1)$ and $c_t \rightarrow 0$ as $t \rightarrow \infty$. Thus by Lemma~\ref{lemma:convergence_lemma} we have that $Q_{n+1}$ converges to $q_{n+1}^l$, and so $\mathrm{P}(n+1)$ holds.

Now suppose $\gamma_{n+1} = 1$, and that every policy leads to a terminal state with probability one. For this case we will need to use a specific maximum norm. Let:
\begin{itemize}
    \item $S_1$ be the set of terminal states in $\mathcal{M}$,
    \item $S_{i+1} = \{s \not\in S_1 \cup \dots \cup S_i \mid \forall a\in A: \\ \mathbb{P}\left[T(s,a) \in S_1 \cup \dots \cup S_i\right] > 0 \}$,
    \item $\eta$ be the largest integer such that $S_\eta \neq \varnothing$,
    \item $S_{a;b} = S_a \cup S_{a+1} \cup \dots \cup S_b$ (where $a,b \in \mathbb{N}$ with $a < b$),
    \item $d : S \rightarrow \mathbb{N}$ be the function such that $d(s) = i$ if $s \in S_i$,
    \item $\epsilon = \min_{\rho \in \{2 , \dots,  \eta\}} \min_{a \in A} \min_{s \in S_d} \\  \mathbb{P}\left[T(s,a) \in S_1 \cup \dots \cup S_{\rho-1}\right]$,
    \item If $\epsilon = 1$ then $w_i = i$, else $w_1 = 1$,
            and $w_{i+1} = 0.5\left(w_i + \min_{j\in \{1 , \dots,  d\}} \\ \left( w_j + \frac{\epsilon}{1-\epsilon}(w_j - w_{j-1})\right) \right)$,
    \item $\Vert \cdot \Vert^W_\infty$ be the maximum norm over $S \times A$ where the weight of $\langle s,a \rangle$ is $w_{d(s)}$.
\end{itemize}
Intuitively, $S_1 , \dots,  S_\eta$ represent different \enquote{distances} to a terminal state. For example, if $s \in S_i$ and an agent is at $s$ then the agent will reach a terminal state with positive probability after $i-1$ steps, regardless of which actions it selects. The constant $\eta$ gives the maximal \enquote{distance}, which must be finite if every policy leads to a terminal state with probability one. The constant $\epsilon$ gives a (positive) lower bound on the probability of moving \enquote{closer} to a terminal state after taking an action, and $w_1 , \dots,  w_\eta$ are the weights we will use to define our maximum norm. 

By straightforward induction we have that $w_1 < \dots < w_\eta$ and that:
\begin{equation*}
    \left(\frac{w_\eta}{w_i}\right)(1-\epsilon) + \left(\frac{w_{i-1}}{w_i}\right)\epsilon < 1. 
\end{equation*}
Let $\kappa = \max_{i \in \{1 , \dots,  \eta\}} \left(\frac{w_\eta}{w_i}\right)(1-\epsilon) + \left(\frac{w_{i-1}}{w_i}\right)\epsilon$. We can now go back to the inequality:
\begin{align*}
    & \gamma_{n+1} \max_{s,a} \sum_{s' \in S} \mathbb{P}\left[T(s,a) = s'\right]\max_{a'} \Big| \beta(s', a')\Big| + c_t\\
    \leq & \max_{s,a} \sum_{s' \in S} \mathbb{P}\left[T(s,a) = s'\right]\max_{a'} \Big| \beta(s', a')\Big| \\
    &\cdot w_{d(s')} \left(\frac{w_{d(s')}}{w_{d(s')}}\right) + c_t\\
    \leq & \max_{s,a} \left( w_{d(s)}\Big|\beta(s, a)\Big| \right) \cdot \\
    &\max_{s,a} \Bigg[ \sum_{s' \in S} \mathbb{P}\left[T(s,a) = s'\right] \left(\frac{w_{d(s')}}{w_{d(s')}}\right) \Bigg] + c_t. 
\end{align*}
By factoring the sum, and using the fact that $w_1 < \dots < w_\eta$, we obtain: 
\begin{alignat*}{2}
    \leq & \Vert \beta_t \Vert^W_\infty \cdot \max_{s,a} \Bigg[ &&\sum_{s' \in S_{1;d(s)-1}} \mathbb{P}\left[T(s,a) = s'\right] \left(\frac{w_{d(s')-1}}{w_{d(s')}}\right) \\
    & &&+ \sum_{s' \in S_{d(s);\eta}} \mathbb{P}\left[T(s,a) = s'\right] \left(\frac{w_\eta}{w_{d(s')}}\right) \Bigg]\\
    &+ c_t\\
    = & \Vert \beta_t \Vert^W_\infty \cdot \max_{s,a} \Bigg[ &&\sum_{s' \in S_{1;d(s)-1}} \mathbb{P}\left[T(s,a) = s'\right] \\
    & &&\left(\frac{w_{d(s')-1}}{w_{d(s')}} - \frac{w_\eta}{w_{d(s')}}\right)\\
    & &&+ \sum_{s' \in S} \mathbb{P}\left[T(s,a) = s'\right] \left(\frac{w_\eta}{w_{d(s')}}\right) \Bigg] + c_t\\
    \leq & \Vert \beta_t \Vert^W_\infty \cdot \max_{s,a} \Bigg[ && \epsilon \left(\frac{w_{d(s')-1}}{w_{d(s')}} - \frac{w_\eta}{w_{d(s')}}\right) + \left(\frac{w_\eta}{w_{d(s')}}\right) \Bigg]\\
    & + c_t\\
    \leq & \Vert \beta_t \Vert^W_\infty \cdot \kappa + c_t. &&
\end{alignat*}
This means that:
$$
\Vert \mathbb{E}\{F_t \mid P_t\} \Vert_\infty \leq \kappa \Vert \beta_t \Vert^W_\infty + c_t, 
$$
where $\kappa \in [0,1)$ and $c_t \rightarrow 0$ as $t \rightarrow \infty$. Thus Lemma~\ref{lemma:convergence_lemma} implies that $Q_{n+1}$ converges to $q_{n+1}^l$, so $\mathrm{P}(n+1)$ holds. 

This exhausts all cases, which means that the inductive step holds, and so $Q_1 , \dots,  Q_m$ must all converge to lexicographically optimal values. From the properties of lexicographic bandit algorithms, and assumption 5, we then see that LRL must converge to a lexicographically optimal policy. 
\end{proof}

We next provide a separate proof that VB-LRL converges to a lexicographically optimal policy if it is used with Lexicographic Double $Q$-Learning.
\begin{lemma}
\label{lemma:double_q_learning}
In any MDP, if conditions 1--4 in Theorem~\ref{thm:convergence} are satisfied then Double $Q$-Learning converges to an optimal policy.
\end{lemma}
\begin{proof}
See \cite{vanHasselt2010}.
\end{proof}

\begin{theorem}
In any MOMDP $\mathcal{M}$, if VB-LRL is used with Lexicographic Double $Q$-Learning and conditions 1--5 in Theorem~\ref{thm:convergence} are satisfied then it will converge to a lexicographically optimal policy.
\end{theorem}
\begin{proof}
We prove this by induction on the reward signals. For $n \in \{0 \dots m\}$, let P($n$) signify that $Q^A_{1,t}, \dots, Q^A_{n,t}$ and $Q^B_{1,t}, \dots, Q^B_{n,t}$ converge to $q_1^l, \dots, q_n^l$ as $t \rightarrow \infty$. 


The base case P($0$) is vacuously true.
For the inductive step $\mathrm{P}(n) \rightarrow \mathrm{P}(n+1)$
we have that assumption 5 and the inductive assumption together imply that there is a time $t'$ such that for all $t'' \geq t'$ we have that:
$$
\Delta^{\tau,t''}_{s,n+1} = \Delta_{s,n+1}.
$$
Let $\mathcal{M}^-$ be the MOMDP that is created if in each state of $\mathcal{M}$ the action set is restricted to $\Delta_{s,n+1}$. Note that after time $t'$ the agent will update $Q^A_{n+1,t}$ and $Q^B_{n+1,t}$ like a Double $Q$-Learning agent with reward $R_{n+1}$ running in $\mathcal{M}^-$ (starting from a particular initialisation).
Thus, by Lemma~\ref{lemma:double_q_learning}, $Q^A_{n+1,t}$ and $Q^B_{n+1,t}$ converge to the optimal $Q$-values for $R_{n+1}$ in $\mathcal{M}^-$. These values correspond to the lexicographically optimal values for $R_{n+1}$ in $\mathcal{M}$, and so $\mathrm{P}(n) \rightarrow \mathrm{P}(n+1)$ holds.
We thus have that $Q_{1,t}, \dots, Q_{m,t}$ converge to $q_1^l, \dots, q_m^l$ as $t \rightarrow \infty$. Assumption 5 then implies that Lexicographic Double $Q$-learning converges to a lexicographically optimal policy. This completes the proof.
\end{proof}

\section{VB-LRL Tolerance}

In this section, we prove a formal guarantee about the limit behaviour of VB-LRL with respect to its two most prioritised reward signals, which can be obtained without any prior knowledge of the reward functions. We also further comment on the problem of selecting the tolerance parameters. 

We begin by arguing that although the convergence of $\tau_\mathcal{B}$ and $\tau_Q$ to $0$ is not technically sufficient to guarantee convergence to a lexicographically optimal policy, such a policy should be found as long as these tolerances approach $0$ sufficiently slowly. The reason for this insufficiency is that if $\tau_\mathcal{B}$ and $\tau_Q$ approach $0$ faster than $Q_1, \dots, Q_m$ approach their limit values, then this might lead the sets $\Delta^{\tau,t}_{s,i}$ -- which are quantified over in the $Q$-value update rule -- to forever remain too small, after a certain time. This is because we might have $q_i^l(s,a_1) = q_i^l(s,a_2)$ and a tolerance that shrinks faster than $Q_{i,t}(s,a_1) - Q_{i,t}(s,a_2)$; in that case, $a_1$ or $a_2$ might erroneously be considered suboptimal according to $R_i$ during the updating of $Q_{i+1}, \dots, Q_m$. 

If the tolerance parameters decrease more slowly than the $Q_i$-values converge, however, then this will not happen. It should therefore be possible to reliably obtain convergence in practice by simply letting $\tau_\mathcal{B}$ and $\tau_Q$ decrease to $0$ very slowly.\footnote{Alternatively, one could perhaps design a scheme whereby the rate at which $\tau_\mathcal{B}$ and $\tau_Q$ decrease depends on the rate at which $Q_1, \dots, Q_m$ change. We conjecture that such a scheme could be used to derive a formal guarantee of VB-LRL's convergence to a lexicographically optimal policy with high probability, but in this work we did not investigate this possibility due to the extra computational and theoretical overhead it would introduce.} Moreover, we also expect VB-LRL to in practice be likely to converge to a policy that is approximately lexicographically optimal as long as the limit values of $\tau_\mathcal{B}$ and $\tau_Q$ are small (but positive). Evidence for this is provided by Proposition \ref{thm:VB-LRL_with_arbitrary_slack}, for which we require the following lemma.



\begin{lemma}
\label{lemma:slack_optimisation_bound}
Let $\pi$ and $\pi'$ be two policies in an MDP $\mathcal{M} = \langle S,A,T,I,R,\gamma \rangle$, and let $q_{\pi}(s, \pi'(s)) \geq \max_a q_{\pi}(s,a) - \delta$ for all $s \in S$. Then $J(\pi') \geq J(\pi) - \frac{\delta}{1-\gamma}$. Moreover, this bound is tight.
\end{lemma}
\begin{proof}
Let $\pi_n$ be the (non-stationary) policy that selects its first $n$ actions according to $\pi$, and all its subsequent actions according to $\pi'$. By a straightforward inductive argument we have that $J(\pi_n) \geq J(\pi) - \sum_{i=1}^n \gamma ^n \delta$, for all $n \in \mathbb{N}$. Then, since $\lim_{n \rightarrow \infty} J(\pi_n) = J(\pi')$ and $\lim_{n \rightarrow \infty} J(\pi) - \sum_{i=1}^n \gamma^n\delta = J(\pi) - \frac{\delta}{1-\gamma}$, we have that $J(\pi') \geq J(\pi) - \frac{\delta}{1-\gamma}$.

To see that the bound is tight, consider the MDP $\mathcal{M} = \langle \{s\},\{a_1,a_2\},T,I,R,\gamma \rangle$ where $R(s,a_1,s) = \delta$, $R(s,a_2,s) = 0$, $T(s,\cdot) = s$, $I(s) = 1$, and the policies $\pi : s \mapsto a_1$ and $\pi' : s \mapsto a_2$.
\end{proof}

\begin{proposition}
\label{thm:VB-LRL_with_arbitrary_slack}
Suppose VB-LRL either uses SARSA or Expected SARSA with any lexicographic bandit algorithm with tolerance $\tau_\mathcal{B}$, or that it uses Lexicographic $Q$-Learning and Lexicographic $\epsilon$-Greedy with tolerance $\tau_Q$ and $\tau_\mathcal{B}$. Moreover, suppose that with probability one there exists a $t \in \mathbb{N}$ such that:
$$
\tau_\mathcal{Q}(s,i,t',Q_1, \dots, Q_m) \leq \tau_\mathcal{B}(s,i,t'+1,Q_1, \dots, Q_m),
$$
for all $t' \geq t$, and further that conditions 1--4 in Theorem~\ref{thm:convergence} are satisfied. Then, in any MOMDP, we have that:

\begin{enumerate}
    \item $J_1(\pi^*) - J_1(\pi_t) \leq \frac{\max \tau_\mathcal{B}^1}{1-\gamma} - \lambda_t$, for some sequence $\{\lambda_t\}_{t \in \mathbb{N}}$ such that $\lim_{t \to \infty} \lambda_t = 0$,  
    \item $J_2(\pi^*) - J_2(\pi_t) \leq \frac{\max \tau_\mathcal{B}^2}{1-\gamma} - \eta_t$, for some sequence $\{\eta_t\}_{t \in \mathbb{N}}$ such that $\lim_{t \to \infty} \eta_t = 0$,
\end{enumerate}
where $\pi^*$ is a lexicographically optimal policy, 
and $\max \tau_\mathcal{B}^i$ is $\max_{s,Q_1, \dots, Q_m} \lim_{t \rightarrow \infty} \tau_\mathcal{B}(s,i,t,Q_1, \dots, Q_m)$.
\end{proposition}

\begin{proof}[Sketch]
The updating of $Q_{1,t}$ is unaffected by the slack parameter, so $Q_{1,t}$ converges to $q_1^l$. This, together with Lemma~\ref{lemma:slack_optimisation_bound} and the properties of lexicographic bandit algorithms, implies the first point of the proposition.

Since $Q_{1,t}$ converges to $q_1^l$, we have that there is a time $t'$ such that $\Delta_{s,1}^{\tau_Q,t''} \subseteq \Delta_{s,1}$ for all $s$ and all $t'' \geq t'$. This means that eventually $Q_{2,t}$ will always be updated in the direction of values that are at least as high as the corresponding $q_2^l$-values. We therefore have that $Q_{2,t}(s,a) \geq q_2^l(s,a) - \eta_t$ for some $\eta_t$ where $\lim_{t \to \infty} \eta_t = 0$ (this can be shown via Lemma~\ref{lemma:convergence_lemma}, but we omit the full derivation for the sake of brevity). This, together with the conditions on $\tau_\mathcal{B}$ and $\tau_Q$, Lemma~\ref{lemma:slack_optimisation_bound}, and the properties of lexicographic bandit algorithms, implies the second point of the proposition.
\end{proof}

Note that the conditions that Proposition~\ref{thm:VB-LRL_with_arbitrary_slack} imposes on $\tau_\mathcal{B}$ and $\tau_Q$ are sufficient, rather than necessary, and that there might be additional choices of $\tau_\mathcal{B}$ and $\tau_Q$ that also lead to the described behaviour.
Note also that Proposition~\ref{thm:VB-LRL_with_arbitrary_slack} does not state that VB-LRL's policy $\pi$ will converge -- it establishes certain conditions on the limit \textit{value} of $\pi$, but does not rule out the possibility that $\pi$ itself continues to oscillate.
This is presumably unlikely, but might perhaps happen if 
$q_i(s,a_1) = q_i(s,a_2) - \lim_{t \rightarrow \infty}\tau_Q(s,i,t,Q_1, \dots, Q_m)$, $a_1 = \argmax_a q_i(s,a)$, and $q_{i+1}(s,a_1) < q_{i+1}(s,a_2)$, for some $s$, $a_1$, $a_2$, and $i$. 

Note also that Proposition~\ref{thm:VB-LRL_with_arbitrary_slack} cannot be strengthened to make any claims about $J_3(\pi), \dots, J_m(\pi)$ without additional assumptions. This is because while we eventually have that $\Delta_{s,0}^{\tau_Q,t} = \Delta_{s,0}$ and $\Delta_{s,1}^{\tau_Q,t} \subseteq \Delta_{s,1}$, there is no guaranteed relationship between $\Delta_{s,i}^{\tau_Q,t}$ 
and $\Delta_{s,i}$ for  $i \geq 3$. To see this, consider what happens as we increase the tolerance. The $Q_1$-values will remain accurate. For $Q_2$, we now have fewer restrictions on what actions we can choose, which means that the $Q_2$-values should increase (or stay the same). However, this might lead the algorithm to choose an action $a$ where $a \not\in \Delta_{s,1}$, with $q_1^l(s,a) \approx \max_a q_1^l(s,a)$ and $q_2^l(s,a) > \max_{a \in \Delta_{s,1}} q_2^l(s,a)$. Then $q_3^l(s,a), \dots, q_m^l(s,a)$ could have any arbitrary values relative to those of the lexicographically optimal actions. 

\section{PB-LRL Convergence Proof}

The proof of PB-LRL's convergence proceeds via a standard multi-timescale stochastic approximation argument; see Chapter 6 in \cite{Borkar2008} for a full exposition. We begin by restating, for reference, the update rules for the parameters $\theta$ and $\lambda$ associated with Lagrangian $L_i$:
\begin{align}
    \theta &\gets \Gamma_\theta \bigg[\theta + \beta^i_t \Big( \nabla_\theta \hat{K}_i(\theta) + \sum^{i-1}_{j=0} \lambda_j \nabla_\theta \hat{K}_j(\theta) \Big) \bigg],\label{pblrl-theta}\\
    \lambda_j &\gets \Gamma_\lambda \Big[\lambda_j + \eta_t \big( \hat{k}_j - \tau_t - \hat{K}_j(\theta) \big) \Big] ~~~~~ \forall~ j < i.\label{pblrl-lambda}
\end{align}
We also restate the conditions on our learning rates $\iota \in \{\alpha, \beta^1, \ldots, \beta^m, \eta^0, \ldots, \eta^m \}$, where for all $i \in \{1,\ldots,m \}$ we have:
\begin{align*}
&\iota_t \in [0,1],~ \sum^\infty_{t=0} \iota_t = \infty,~ \sum^\infty_{t=0} (\iota_t)^2 < \infty ~~~\text{and} \\ &\lim_{t \rightarrow \infty} \frac{\beta^{i}_t}{\alpha_t} = \lim_{t \rightarrow \infty} \frac{\eta^i_t}{\beta^i_t} = \lim_{t \rightarrow \infty} \frac{\beta^{i}_t}{\eta^{i-1}_t} = 0.
\end{align*}
Our main theorem for PB-LRL is then given as follows, where we make some minor observations before beginning the proof.




\begin{theorem}
    Let $\mathcal{M}$ be a MOMDP, $\pi$ a policy that is  twice continuously differentiable in its parameters $\theta$, and assume that the same form of objective function is chosen for each $K_i$ and that each reward function $R_i$ is bounded. If using a critic, let $V_i$ (or $Q_i$) be (action-)value functions that are continuously differentiable in $w_i$ for $i \in \{1,\dots,m\}$ and suppose that if PB-LRL is run for $T$ steps there exists some limit point $w^*_i(\theta) = \lim_{T \rightarrow \infty} \expect_t [ w_i ]$ for each $w_i$ when $\theta$ is held fixed under conditions $\mathcal{C}$ on $\mathcal{M}$, $\pi$, and each $V_i$. 
    If $\lim_{T \rightarrow \infty}  \expect_t [ \theta ] \in \Theta^\epsilon_1$ (respectively $\tilde{\Theta}^\epsilon_1$) under conditions $\mathcal{C}$ when $m = 1$, then for any fixed $m \in \mathbb{N}$ we have that $\lim_{T \rightarrow \infty}  \expect_t [ \theta ] \in \Theta^\epsilon_m$ (respectively $\tilde{\Theta}^\epsilon_m$), where each $\epsilon_i \geq 0$ is a constant that depends on the representational power of the parametrisations of $\pi$ (and $V_i$ or $Q_i$, if using a critic).
\end{theorem}

    In the proof below we use $\epsilon_\theta \geq 0$ to quantify the representational power of our policy parametrisation as follows:
    $$\min_{\theta \in \Theta} \max_{\pi,s} \sum_a \vert \pi(a \mid s) - \pi(a \mid s; \theta) \vert \leq \epsilon_\theta.$$
    where, as we will see, each $\epsilon_i$ is a function of $\epsilon_\theta$. In the tabular setting or when $\pi$ is represented using, say, an over-parametrised neural network, then we have that $\epsilon_\theta = 0$ \cite{Paternain2019}.
    
    Similarly, for any fixed $\theta$ it will not always be possible to guarantee that the approximation $V^*_i$ of each $v^i_\theta$ defined by $w^*_i(\theta)$ will have zero error and thus, if each estimated objective $\hat{K}_i$ is computed using $V_i$, we may only prove convergence to within a small neighbourhood of lexicographic optima.\footnote{In the following remarks and proof, for ease of exposition, we focus primarily on the case of a value rather than action-value critic, but analogous arguments apply in the case of the latter.} In particular, let us define $\tilde{J}_i(\theta) = \sum_s I(s) V^*_i(s)$. When $\max_{\theta,s} \expect_\theta \big[ \vert v_\theta^i(s) -  V^*_i(s) \vert \big] \leq \epsilon_{w_i}$ we have that $\max_{\theta \in \Theta^\epsilon_{i-1}} J_i(\theta) - J_i(\theta^*) \leq \epsilon_{w_i}$ where $\theta^* \in \argmax_{\theta \in \Theta^\epsilon_{i-1}} \tilde{J}_i(\theta)$. As such, for the majority of the proof we consider convergence with respect to each $\tilde{J}_i$ instead of each $J_i$.

    
    
 
    Finally, as noted in the main body of the paper, when $m=1$ then PB-LRL reduces to whatever algorithm is defined by the choice of objective function, such as A2C when using $L^{\ac}_1$, or PPO when using $L^{\ppo}_1$. As, by assumption, PB-LRL converges to an optimum (either local or global) of $\tilde{J}_1$ by following (an estimate of) the gradient of $K_1$, we see that $\nabla_\theta K_1(\theta)$ serves as a faithful \emph{surrogate} for the true policy gradient. Thus (by a simple inductive argument), in the general setting we may in fact focus on convergence according to the objectives $K_1, \dots, K_m$ as opposed to $\tilde{J}_1, \dots, \tilde{J}_m$, safe in the knowledge that a lexicographic optimum with respect to the former is sufficient for a lexicographic optimum with respect to the latter.

\begin{proof}

    The proof proceeds by induction on $i \in \{1,\ldots,m\}$ for some fixed $m$, using a multi-timescale approach. Due to the learning rates chosen, we may consider those more slowly updated parameters fixed for the purposes of analysing the convergence of the more quickly updated parameters \cite{Borkar2008}. Thus, for the updates to each vector $w_i$ the other parameters $\theta$ and $\lambda$ may be viewed as fixed. By assumption, under conditions $\mathcal{C}$ we have that this process thus converges to some fixed $w^*_i(\theta)$ for each $w_i$.
    
    
    Let us next consider, as the base case $i = 1$ for our inductive argument, the convergence of $\theta$ with respect to the 
    Lagrangian $L_1(\theta, \lambda)$ objective. As all of the critic parameters are updated on a faster timescale we may view them as having converged for the purpose of our analysis. Further, by the above argument we have that each $w_i = w^*_i(\theta)$. Similarly, as the updates to $\theta$ with respect to $L_j(\theta, \lambda)$ for $j > i = 1$ occur at a slower timescale we may view $\theta$ as static with respect to these latter updates. Note that at this step we may also ignore updates to the sets of Lagrange multipliers $\lambda$ as they do not feature in $L_1(\theta, \lambda) = K_1(\theta)$. As such, this reduces to the case with $m=1$ where, by assumption, we have that $\theta$ converges to either a local or global optimum of $K_1$ (and thus an $\epsilon$-optimum, for $\epsilon_1 = \epsilon_{w_1}$, of $J_i$) under conditions $\mathcal{C}$, as required.

    Let us next assume as our inductive hypothesis, not only that $\theta$ has converged with respect to objectives $K_1,\ldots,K_{i-1}$ (as entailed by of our choice of learning rates), but has converged, under conditions $\mathcal{C}$, to some policy in $\Theta^\epsilon_{i-1}$. Then we have that $\hat{k}_j := \hat{K}_j(\theta)$ has also converged for each $j \in \{1,\ldots,i-1\}$
    and that $\eta = \eta^i$ is our current learning rate for the Lagrange multipliers
    (see lines 6 and 7 of Algorithm 3).
    Note that as the updates to $\theta$ with respect to $L_j(\theta, \lambda)$ for $j > i$ occur at a slower timescale we may view $\theta$ as static with respect to these latter updates, and as updates for $j < i$ occur at a faster timescale then we may assume that $\theta$ has converged with respect to each such $L_j(\theta, \lambda)$. 
    
    As before, we have that the Lagrange multipliers $\lambda$ are updated more slowly and the critic parameters $w^*_i(\theta)$ more quickly so may view each of them as static. Thus, when using a critic or when estimating the gradient using another unbiased method (such as Monte Carlo estimates generated from trajectories) then $\expect_\theta [\hat{K}_i (\theta)] = K_i(\theta)$ for each $i \in \{1,\ldots,m\}$ and so:
    \begin{align*}
    &\expect_\theta \Big[ \nabla_\theta \hat{K}_i(\theta) + \sum^{i-1}_{j=0} \lambda_j \nabla_\theta \hat{K}_j(\theta) \Big] \\ &= \nabla_\theta K_i(\theta) + \sum^{i-1}_{j=0} \lambda_j \nabla_\theta K_j(\theta) = \nabla_\theta L_i(\theta, \lambda).
    \end{align*}
    In other words, the update rule (\ref{pblrl-theta}) uses an unbiased estimate of the gradient of the Lagrangian $L_i$ with respect to $\theta$ and can thus be seen as a discretisation of the ODE:
    $$\dot{\theta^t} = \Gamma_\theta \big( \nabla_\theta L_i(\theta^t,\lambda) \big),$$
    where $t$ indexes the values of $\theta$ over time and the projection operator $\Gamma_\theta$ ensures that iterates governed by this ODE remain in a compact set. 
    
    Given that gradient descent on $K_1$ converges to a locally or globally optimal stationary point when $m = 1$, then it must be the case that $K_1$ is either locally or globally \emph{invex} respectively \cite{BenIsrael1986}, where recall that a differentiable function $f : \mathbb{R}^n \rightarrow \mathbb{R}$ is (globally) invex if and only if there exists a function $g : \mathbb{R}^n \times \mathbb{R}^n \rightarrow \mathbb{R}^n$ such that $f(x_1) - f(x_2) \geq g(x_1,x_2)^\top \nabla f(x_2)$ for all $x_1, x_2 \in \mathbb{R}^n$ \cite{Hanson1981}, with a generalisation to local invexity being straightforward \cite{Craven1981}. Further, as each objective $K_i$ is of the same form (by assumption) then each $K_i$ is similarly locally or globally invex, and thus so too is each $L_i$ in $\theta$, being a linear combination of $K_1,\ldots,K_{i-1}$ with the addition of a single scalar term. Hence, the stationary point $\theta^*(\lambda)$ of the above ODE for fixed Lagrange multipliers $\lambda$ is such that $L_i(\theta^*(\lambda),\lambda)$ is is either a local or global maximum of $L_i$ respectively.
    
    What remains is to consider the convergence of said Lagrange multipliers at a slower timescale given by $\eta = \eta_{i}$. As the actor and critics are updated at a faster timescale we may again view them as having converged to $\theta^*(\lambda)$ and $w^*_i(\theta^*)$ respectively. Similarly to above we have that:
    $$\expect_{\theta} [ \hat{k}_j - \tau_t - \hat{K}_j(\theta) ] = k_j - \tau_t - K_j(\theta) = \nabla_{\lambda_j} L_i(\theta, \lambda),$$
    for each $j \in \{1,\ldots,m-1\}$. As such, the gradient update in (\ref{pblrl-lambda}) uses unbiased estimates and thus forms a discrete approximation of the following ODE:
    $$\dot{\lambda^{t'}_j} = - \nabla_{\lambda_j} L_i(\theta(\lambda^{t'}),\lambda^{t'}),$$
    where $t'$ indexes the values of each $\lambda_j$ over time, $\theta(\lambda^{t'})$ is the limit of the $\theta$ recursion with static parameters $\lambda^{t'}$, and $\Gamma_\lambda$ ensures that $\lambda \succcurlyeq 0$. As shown above, $\theta^t(\lambda^{t'}) \rightarrow \theta^*$, and by a similar argument we have $\lambda_j^{t'} (\theta^*) \rightarrow \lambda_j^*$, where $\lambda_j^{t'} (\theta)$ is the limit of the $\lambda^{t'}_j$ recursion with static parameters $\theta$ and thus that $ L_i(\theta^*,\lambda^*)$ is a global minimum (as $L_i$ is convex in $\lambda$), given $\theta^*$. 
    
    Thus, $(\theta^*,\lambda^*)$ forms a local or global saddle point of the Lagrangian $L_i$, representing a solution to the dual formulation of our original constrained optimisation problem. Furthermore, it can be shown that (given Slater's condition) the duality gap between this solution and a locally or globally optimal solution to the primal formulation is bounded above by \cite{Paternain2019}:
    $$\epsilon'_\theta \coloneqq \epsilon_\theta \cdot \bigg( \max_{s,a}R_i(s,a) + \Vert \lambda_\xi \Vert_1 \max_{j<i,s,a} \frac{R_j(s,a)}{1-\gamma_j} \bigg),$$
    where:
    \begin{align*}
    &\lambda_\xi = \argmin_{\lambda \succcurlyeq 0} \max_{\theta}~ L_i(\theta,\lambda) - \xi \sum_{j=1}^i \lambda_j ~~~\text{ and }~~~  \\ &\xi = \max_{j<i,s,a} \frac{\epsilon_\theta R_j(s,a)}{1-\gamma_j}.
    \end{align*}
    Given our approximation of $V_i$ by $w_i$, and thus our convergence with respect to $\tilde{J}_i$ instead of $J_i$, we therefore have $\theta$ converges to within $\epsilon_i \coloneqq \epsilon'_\theta + \epsilon_{w_i} + \lim_{t \rightarrow \infty} \tau_t = \epsilon'_\theta + \epsilon_{w_i}$ of a local or global lexicographic optimum for $J_1,\ldots,J_i$, i.e.\ that $\lim_{T \rightarrow \infty}  \expect_t [ \theta ] \in \Theta^\epsilon_{i}$ or $\tilde{\Theta}^\epsilon_{i}$ respectively. The proof of this inductive step completes our overall argument. It can readily be seen the procedure concludes when solving, in effect, the (relaxed version of) the constrained optimisation problem for the $m^\text{th}$ objective, and hence that $ \lim_{T \rightarrow \infty}  \expect_t [ \theta ] \in \Theta^\epsilon_m$ or $\tilde{\Theta}^\epsilon_{i}$, as required. 
\end{proof}

\begin{corollary}
    Suppose that each critic is linearly parametrised as $V_i(s) = w_i^\top \phi(s)$ for some choice of state features $\phi$ and is updated using a semi-gradient TD(0) rule, and that:
    \begin{enumerate}
        \item $S$ and $A$ are finite, and each reward function $R_i$ is bounded,
        \item For any $\theta \in \Theta$, the induced Markov chain over $S$ is irreducible,
        \item For any $s\in S$ and $a \in A$, $\pi(a \mid s ; \theta)$ is twice continuously differentiable,
        \item Letting $\Phi$ be the $\vert S \vert \times c$ matrix with rows $\phi(s)$, then $\Phi$ has full rank (i.e.\ the features are independent), $c \leq \vert S \vert$, and there is no $w \in W$ such that $\Phi w = 1$.
    \end{enumerate}
    Then for any MOMDP with discounted or limit-average objectives, LA2C almost surely converges to a policy in $\tilde\Theta^\epsilon_m$.
\end{corollary}

\begin{proof}
     Note that the critic updates are equivalent to the well-known Linear Semi-Gradient TD(0) algorithm \cite{SuttonAndBarto} under the state distribution $d^\theta_i$ that depends on $\gamma_i$,\footnote{Said dependence is often ignored in the literature, we refer the reader to \cite{Thomas2014} for further discussion.} given by: 
     $$d^\theta_i(s) := \lim_{T \rightarrow \infty} \frac{1}{\sum^T_{t=0} \gamma_i^t} \sum^T_{t=0} \gamma^t_i \Pr(s_t = s \vert s_0, \theta) I (s_0).$$
     Under the assumed conditions this algorithm is known to converge with probability one to some $w_i^*(\theta)$ where $V_i = \Phi w^*_i(\theta)$ is the (unique) TD fixed point minimising the mean squared Bellman error. In the discounted case, see \cite{Tsitsiklis1997} for the original proof (Theorem 1) and \cite{SuttonAndBarto} for a more recent exposition (Section 9.4). In the limit-average case see, for one particularly clear and concise proof, Lemma 5 in \cite{Bhatnagar2009}. 
     
     Similarly, under the conditions listed above the standard A2C algorithm is known to converge asymptotically to a local $\epsilon_{w_1}$-optimum $J_1$ when $m=1$. See \cite{Bhatnagar2009} for the limit-average setting, the proof for which can be readily adapted to the discounted setting assuming that proper discounting is applied to the various gradient updates in order to reflect $d^\theta_1$, and given that the definition of $J_1$ is based on a specific initial distribution $I$, rather than the steady state distribution under the policy $\theta$ \cite{Sutton1999,Thomas2014}. With these facts in hand we may apply Theorem \ref{thm:PB-LRL} to obtain our result.
\end{proof}

\begin{corollary}
    Let $\pi(a \mid s; \theta, \chi) \propto \exp\big(\chi^{-1} f(s, a ; \theta)\big)$ and supposed that both $f$ and the \emph{action-value} critics $Q_i$ are parametrised using two-layer neural networks (where $\chi$ is a temperature parameter), that a semi-gradient TD(0) rule is used to update $Q_i$, and that $Q_i$ replaces $A_i$ in the standard PPO loss $\hat{K}^{\ppo}$, both of which updates use samples from the \emph{discounted} steady state distribution. Further, let us assume that:
    \begin{enumerate}
        \item $S$ is compact and $A$ is finite, with $S \times A \subseteq \mathbb{R}^d$ for some finite $d > 0$, and each reward function $R_i$ is bounded,
        \item The neural networks have widths $\mu_f$ and $\mu_{Q_i}$ respectively with ReLU activations, initial input weights drawn from a normal distribution with mean $0$ and variance $\frac{1}{d}$, and initial output weights drawn from $\mathrm{unif}([-1,1])$,
        \item We have that $q^\pi_i(\cdot,\cdot) \in \big\{Q_i(\cdot,\cdot; w_i) \mid w_i \in  \mathbb{R}^y \big\}$ for any  $\pi \in \Pi$,
        \item There exists $c > 0$ such that for any $z \in \mathbb{R}^d$ and $\zeta > 0$ we have that $\expect_\pi \big[ \mathbf{1} (\vert z^\top (s,a) \vert \leq \zeta  )  \big] \leq \frac{c\zeta}{\Vert z \Vert_2}$ for any $\pi \in \Pi$.
    \end{enumerate}
    Then for any MOMDP with discounted objectives, LPPO almost surely converges to a policy in $\Theta^\epsilon_m$. Furthermore, if the coefficient of the KL divergence penalty $\kappa > 1$ then $\lim_{\mu_f, \mu_{Q_i} \rightarrow \infty} \epsilon = 0$.
\end{corollary}
    
\begin{proof}
    Our proof makes use of recent results about the global optimality of PPO when using over-parametrised networks. See \cite{Liu2019b} for the original detailed result that we make use of here and \cite{Hsu2020} for a more intuitive exposition based on the theory of multiplicative weights. Our arguments are essentially identical to those found in \cite{Liu2019b} with minor adaptions to accommodate: a) our use of diminishing rather than constant learning rates; b) our focus on asymptotic convergence; and c) our assumption that samples are taken from the \emph{discounted} steady-state distribution $d^\theta_i(s)$ defined in the previous proof. We refer the reader to this paper for a full exposition.
    
    The aforementioned result relies on bounding the error for a generic network $u(\cdot, \cdot ; \omega)$ of the same form as that representing $f$ and each $Q_i$ based on the updates given by:
    \begin{align*}
        \omega_{t+1} \gets \Gamma_\omega \Big[ & \omega_t - \iota \big( u(s,a;\omega_t) - g(s,a) - h \cdot u(s',a';\omega_t) \big) \\ &\nabla_\omega u(s,a;\omega_t)  \Big]
    \end{align*}
    for a fixed learning rate $\iota$, and some fixed function $g$ and constant $h$, where $\Gamma_\omega [ \omega' ] \coloneqq \argmin_{\omega \in \mathcal{B}^0(r_u)} \Vert \omega - \omega' \Vert_2$ projects $\omega'$ to the ball $\mathcal{B}^0(r_u) \coloneqq \{\omega \mid \Vert \omega_0 - \omega \Vert_2 \leq r_u  \}$ of radius $r_u$ about the initial parameters $\omega_0$. If we instead replace $\iota$ with a series of learning rates $\{ \iota_t \}_{t \in \mathbb{N}}$ such that $\iota_t \in [0,1]$, $\sum^\infty_{t=0} \iota_t = \infty$, and $\sum^\infty_{t=0} (\iota_t)^2 < \infty$ then at time $t$ we have the following error bound:
    \begin{align*}
        &\expect_{\nu} \Big[ \big( u(s,a;\omega_t) - u^0(s,a;\omega^*) \big)^2 \mid \omega_t \Big] \leq\\
        &\frac{1}{\iota_t(1 - h) - 4\iota^2_t} \Big( \Vert \omega_t - \omega^* \Vert^2_2 - \expect_\nu \big[ \Vert \omega_{t+1} - \omega^* \Vert^2_2 \mid \omega_t \big] \\ & + C_1(r_u) \iota^2_t \Big) + C_2(r_u, \mu_u),
    \end{align*}
    for some choice of distribution $\nu$ over $S \times A$, and where $u^0$ is the local linearisation of $u$ at $\psi_0$, $\omega^*$ is such that:
    \begin{align*}
    &\expect_\nu \Big[ \big(u^0(s,a;\omega^*) - g(s,a) - h \cdot  u^0(s',a';\omega^*)\big) \\ &\nabla_\omega \big(u^0(s,a;\omega^*) \Big] ^\top (\omega - \omega^*) \geq 0,
    \end{align*} 
    for any $\omega \in \mathcal{B}^0(r_u)$, and $C_1(r_u)$ and $C_2(r_u, \mu_u)$ are constants of size $O(r^{2}_u)$ and $O(r^{5/2}_u \mu_u^{1/4} + r^{3}_u \mu_u^{1/2})$ respectively. Then telescoping up to timestep $T$ and dividing by $\sum^{T-1}_{t=0} \big(\iota_t(1 - h) - 4\iota^2_t \big)$ we have:
    \begin{align*}
        &\expect_{\nu,I_{\omega},t} \Big[ \big( u(s,a;\omega_t) - u^0(s,a;\omega^*) \big)^2  \Big]\\
        &\leq \frac{\Big( \expect_{I_{\omega}} [ \Vert \omega_0 - \omega^* \Vert^2_2 ] + \sum^{T-1}_{t=0}  C_1(r_u) \iota^2_t \Big)}{\sum^{T-1}_{t=0} \big(\iota_t(1 - h) - 4\iota^2_t \big)} + C_2(r_u, \mu_u)\\
        &\leq \frac{\Big( r_u + \sum^{T-1}_{t=0}  C_1(r_u) \iota^2_t \Big)}{\sum^{T-1}_{t=0} \big(\iota_t(1 - h) - 4\iota^2_t \big)}  + C_2(r_u, \mu_u),
    \end{align*}
    where $t$ is sampled from $[0,T-1]$ with probability $\propto \iota_t(1 - h) - 4\iota^2_t$, $\omega_0$ is sampled from $I_{\omega}$ (where $I_{\omega}$ is the distribution described by condition 2 of the corollary), and in the second inequality we use the fact that $\Vert \omega_0 - \omega^* \Vert^2_2 \leq r_u$. By using different instantiations of $g$, $h$, $r_u$, and $\iota$, we may instantiate the update rules for both $w_1$ and $\theta$, as described in \cite{Liu2019b}. 
    
    Let the error bounds for these cases be given by $\epsilon_{Q_1}(T)$ and $\epsilon_{f}(T)$ respectively, and note that due to the conditions on $\alpha$ and $\beta^1$ then we have $\lim_{T \rightarrow \infty} \epsilon_{Q_1}(T) = C_2(r_{Q_1}, \mu_{Q_1})$ and $\lim_{T \rightarrow \infty} \epsilon_f(T) = C_2(r_f, \mu_f)$ respectively. Supposing that $\pi$ and $\kappa$ are updated every $T$ steps (where recall that $\kappa$ is the coefficient of the KL divergence penalty term in the PPO update rule), and that this is repeated $K$ times, then we have the following bound:
    \begin{align*}
        &J_1(\theta^*) -  \expect_k [J_1(\theta_k) ]\\
        &\leq \epsilon^{T,K}_1 \coloneqq
    \frac{\log \vert A \vert + M \sum^{K-1}_{k=0} ( \kappa^{-2}_k + \epsilon_k ) }{(1 - \gamma_1) \sum^{K-1}_{k=0} \kappa^{-1}_k},
    \end{align*}
    where $k$ is sampled from $[0,K-1]$ with probability $\propto \kappa^{-1}_k$ and we have:
    \begin{align*}
        M &\coloneqq 2 \big( \expect_{\theta^*} [\max_{a \in A} Q(s,a;w_0)^2] + r^2_f \big),\\
        \epsilon_k &\coloneqq \chi^{-1}_{k+1} \epsilon_{f}(T) \phi^*_k + \kappa^{-1}_k \epsilon_{Q_1}(T) \omega^*_k + \vert A \vert \chi^{-1}_{k+1} \epsilon_{f}(T)^2,
    \end{align*}
    for constants $\phi^*_k$ and $\omega^*_k$ that represent density ratios -- see equation 4.2 in \cite{Liu2019b} for their precise definitions. Letting $\epsilon_1 \coloneqq \lim_{T,K \rightarrow \infty} \epsilon^{T,K}_1$ allows us to apply Theorem 3, where note that if $\kappa_k > 1$ then for each $i$ we have:
    \begin{align*}
        & \mu_{Q_i}, \mu_f \rightarrow \infty \\
        ~\Rightarrow~ & \epsilon_{w_i}, \epsilon_\theta, C_2(r_{Q_i}, \mu_{Q_i}), C_2(r_f, \mu_f) \rightarrow 0  \\ 
        ~\Rightarrow~ & \epsilon_k \rightarrow 0 \\
        ~\Rightarrow~ & \epsilon_i \rightarrow 0,
    \end{align*}
    which thus results in a global lexicographic optimum.
\end{proof}

\section{Hyperparameters}

In this section we list the hyperparameters used in Section \ref{sec:experiments}, as well as other noteworthy implementation details. 

In Section \ref{sec:experiments:1}, the VB-LRL algorithms are tabular, with a learning rate of $\alpha_t = 0.01$. All VB-LRL algorithms use Lexicographic $\epsilon_t$ Greedy with $\epsilon_t = 0.05$, and slack $\tau(s,i) = \xi(s,i,t) = 0.01\cdot \max_{a\in\Delta^\tau_{s,i}}Q_i(s,a)$. The PB-LRL algorithms take the state as a one-hot encoded vector, and use fully-connected neural networks without hidden layers or bias terms. The code used to generate the MOMDPs is based on \emph{MDPToolBox}\footnote{Available at \url{https://github.com/sawcordwell/pymdptoolbox}.} but with Gaussian noise ($\sigma = 0.2$) added to the reward. 

In the CartSafe environment, RCPO's constraint was $J_2(\pi) \leq 10$, AproPO's constraints were $J_1(\pi) \in [90,350]$, $J_2(\pi)\in[0,10]$, and VaR\_AC's constraint was to keep the cost under 10 with probability at least 0.95. In the GridNav environment, RCPO's constraint was $J_2(\pi) \leq 5$, AproPO's constraints were $J_1(\pi)\in [90,100]$, $J_2(\pi)\in [0,5]$, and VaR\_AC's constraint was to keep the cost under 5 with probability at least 0.95. Finally, in the Intersection environment, RCPO's constraint was $J_2(\pi) \leq 1$, AproPO's constraints were $J_1(\pi)\in [9,10]$, $J_2(\pi)\in[0,1]$, and VaR\_AC's constraint was to keep the cost under 1 with probability at least 0.95.
The lexicographic algorithms were trying to minimise cost and, subject to that, maximise reward. Each algorithm was run ten times in each environment.

In Section \ref{sec:experiments:2}, we use VB-LRL with Lexicographic $\epsilon_t$-Greedy with $\epsilon_t = 0.05$, Lexicographic $Q$-Learning, and slack $\tau(s,i) = \xi(s,i,t) = 0.05\cdot \max_{a\in\Delta^\tau_{s,i}}Q_i(s,a)$ in GridNav and Intersection, and $0.001\cdot \max_{a\in\Delta^\tau_{s,i}}Q_i(s,a)$ in CartSafe. All agents use fully connected networks, and for the algorithms with parameterised policies, separate networks were used for prediction and control. AproPO's best-response oracle was a DQN algorithm. In GridNav and Intersection, the algorithms update their networks every 32 steps, and use 32 as their batch size, but in CartSafe they update every 8 steps, and have a batch size of 8. DQN and LDQN have a replay buffer of size 100,000, and AproPO performs an update every 8 episodes in GridNav and Intersection, and every episode in CartSafe. The networks were updated with the Adam optimiser, with a learning rate of 0.001 in CartSafe and Intersection, and 0.01 in GridNav. In the CartSafe environment the agents use networks with two hidden layers with 24 neurons each, in the GridNav environment they use networks without hidden layers or bias terms, and in the Intersection environment they use networks with two hidden layers with 512 neurons each.

\section{A Hyperparameter Study for VB-LRL's Slack}

Here we present some additional data that investigates the practical significance of VL-LRL's slack parameter. In particular, we run LDQN in CartSafe, and vary $\tau(s,i) = \xi(s,i,t)$ from 0 to 1 in steps of 0.2. The result of this experiment is shown in Figure~\ref{fig:tolerance_study}.  We can see that the choice of slack parameter does have an impact on the performance of the algorithm, but that this impact is fairly gradual. The hyperparameters used in this experiment were similar to those used for the experiment in Section \ref{sec:experiments:2}.

\begin{figure}
  \centering
    \subfloat[Reward]{{\includegraphics[width=0.24\textwidth]{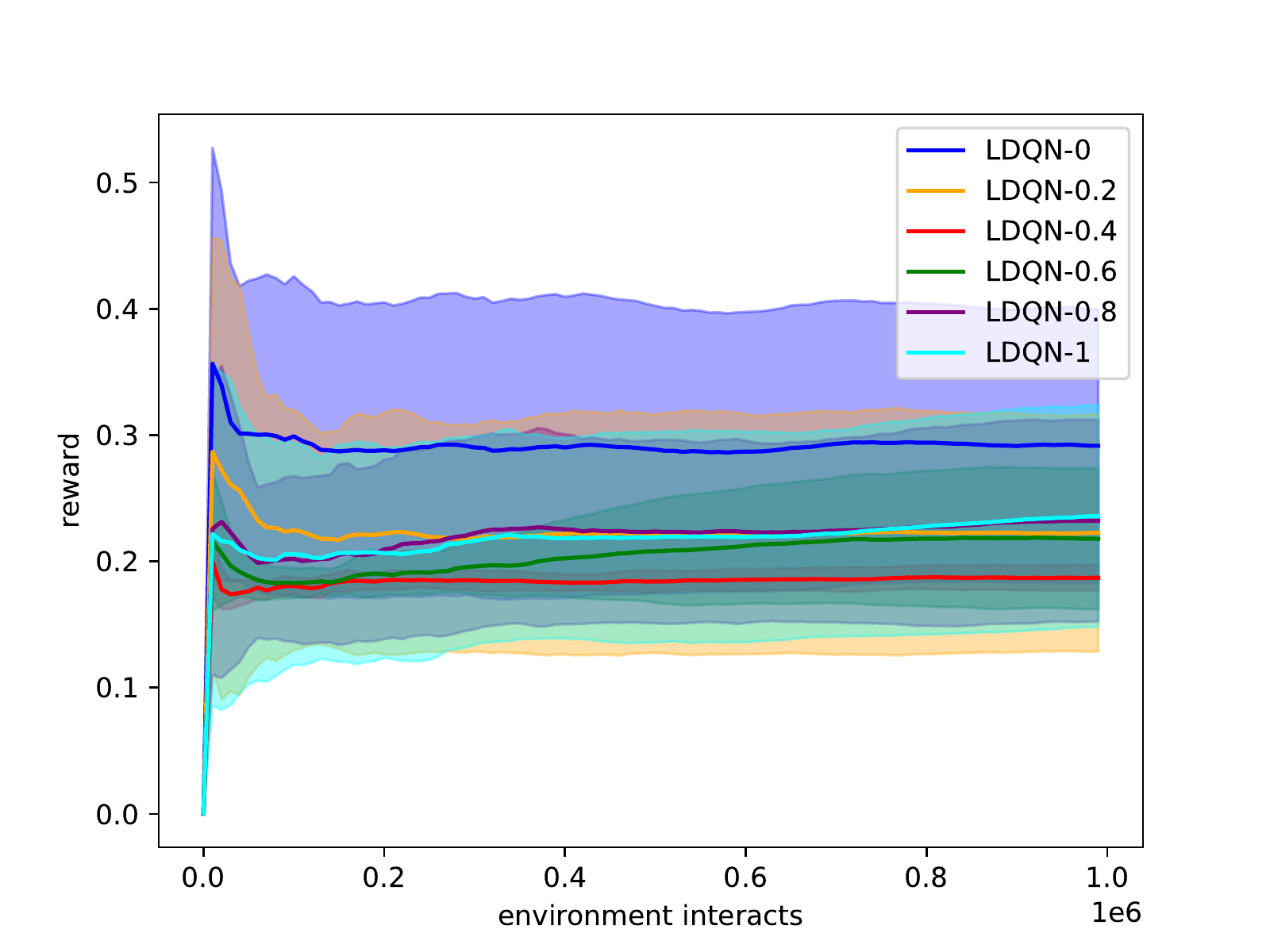}}}
    \subfloat[Cost]{{\includegraphics[width=0.24\textwidth]{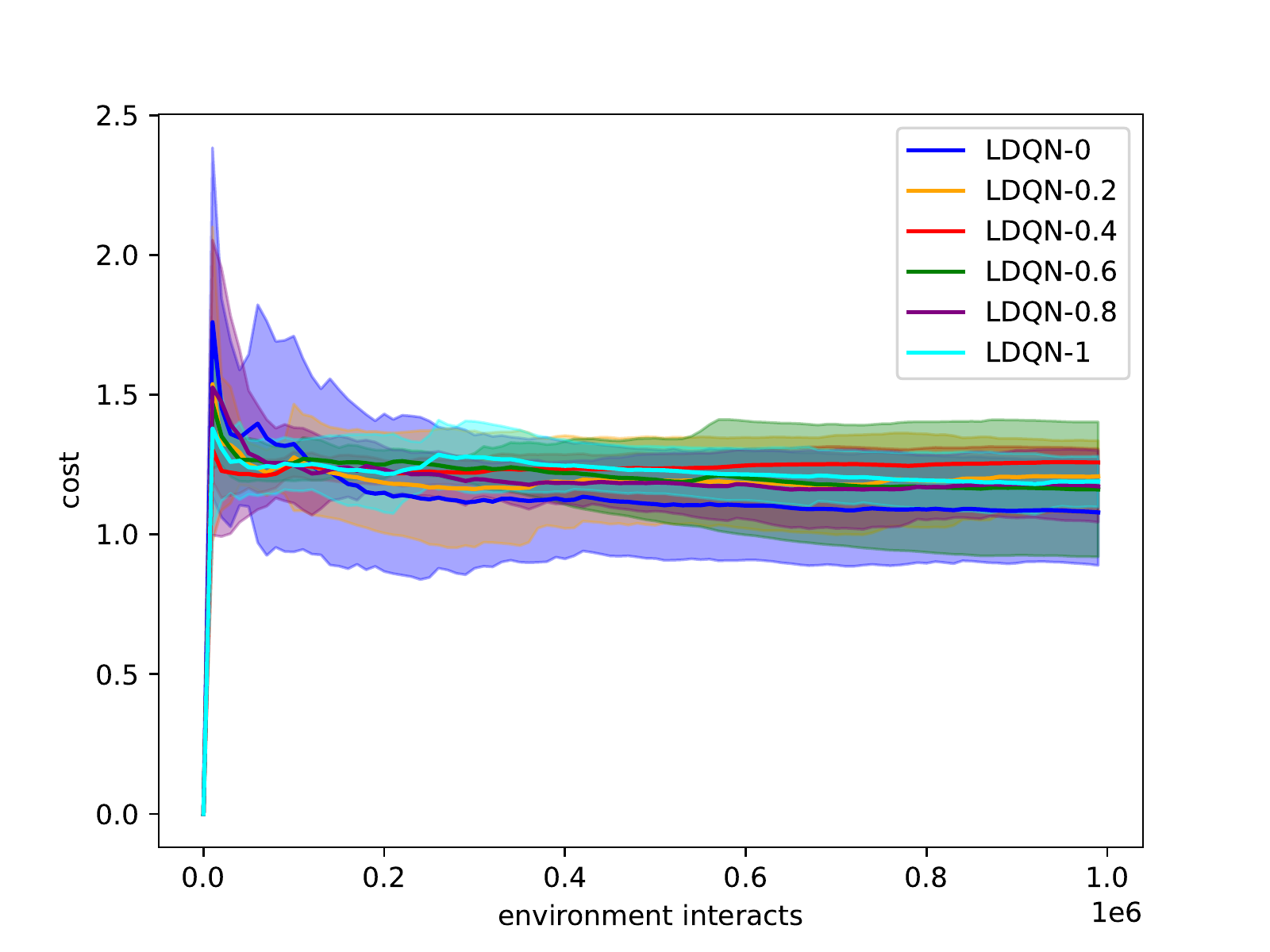}}}
    \caption{This experiment studies the impact of LDQN's tolerance parameter in the CartSafe environment.}
    \label{fig:tolerance_study}
\end{figure}

\section{A Comparison of Different Settings of PB-LRL}

Here we present data that compares different settings of of PB-LRL. Specifically, we run LPPO and LA2C in CartSafe, with and without second-order derivatives in the loss, and with and without sequential optimisation of the objectives, as opposed to the simultaneous method presented in the paper. We also include LDQN for comparison. The result of this experiment is shown in Figure~\ref{fig:lac_extras}. We can see that these versions of PB-LRL all are quite similar in terms of the cost they attain, and that LA2C achieves higher reward than LPPO. Moreover, it seems that the use of second-order terms or sequential optimisation has an impact on the performance, but that this impact is not that large. The other hyperparameters used in this experiment were similar to those used for the experiment in Section \ref{sec:experiments:2}.

\begin{figure}
  \centering
    \subfloat[Reward]{{\includegraphics[width=0.24\textwidth]{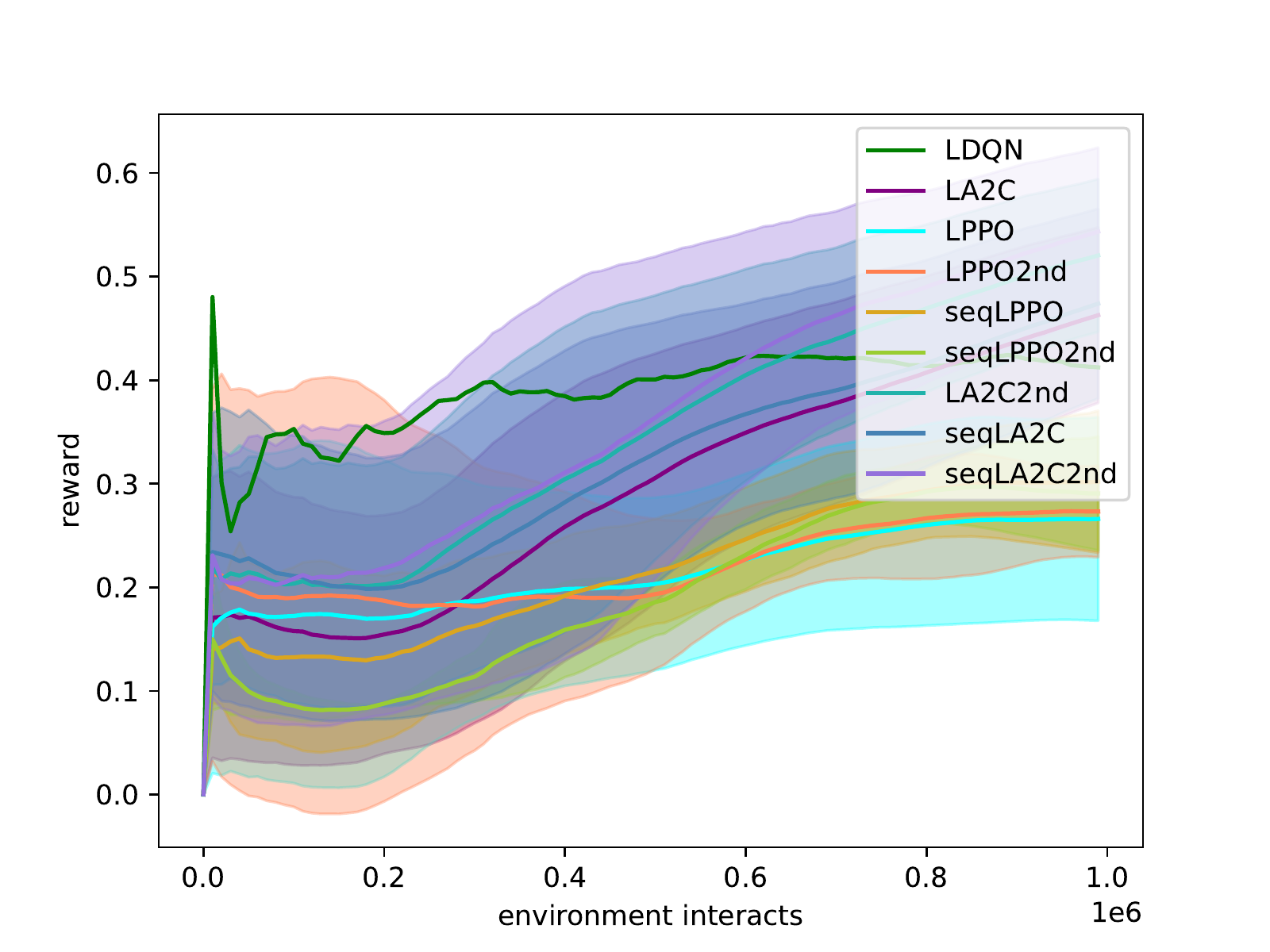}}}
    \subfloat[Cost]{{\includegraphics[width=0.24\textwidth]{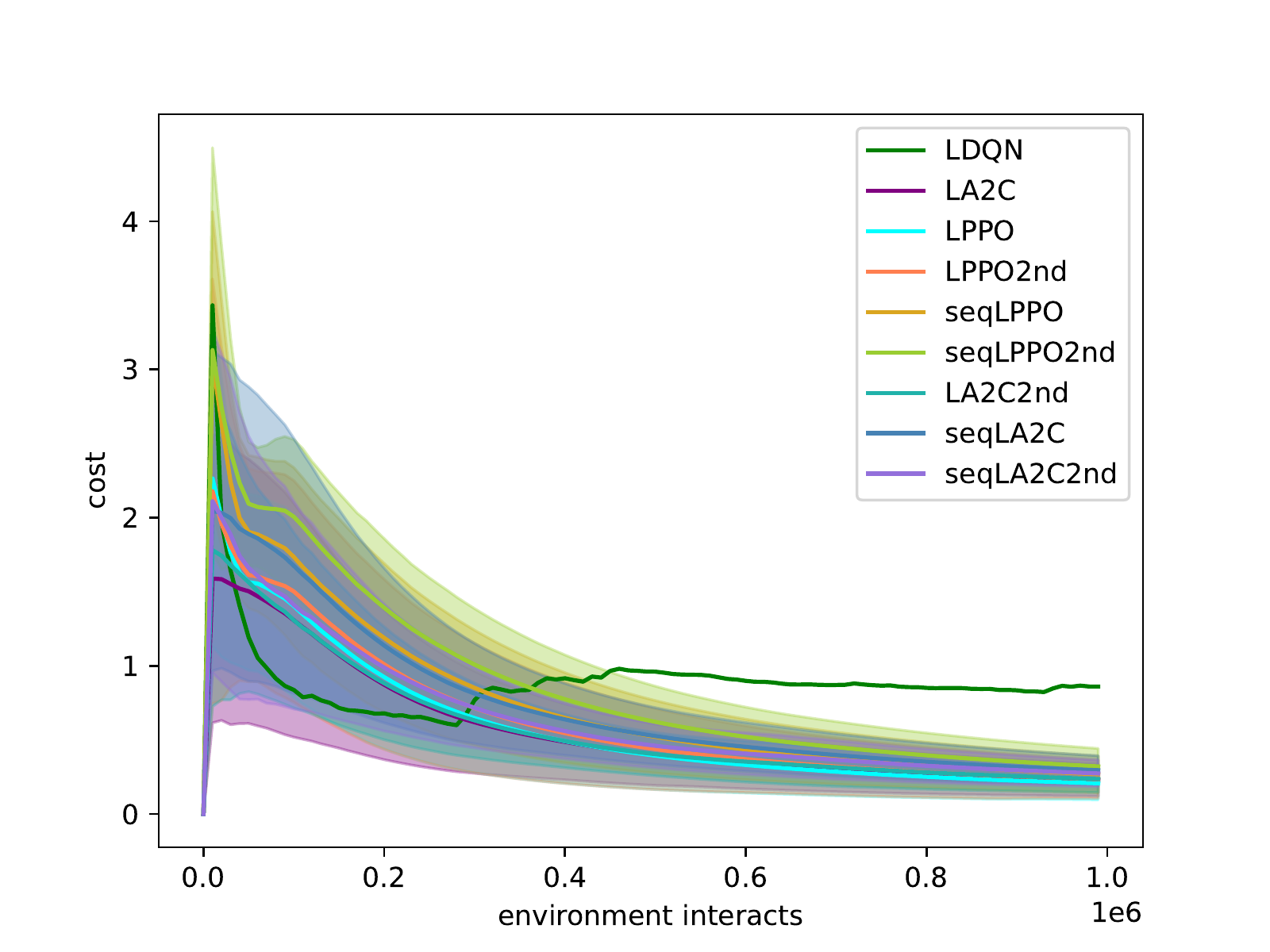}}}
    \caption{This experiment compares different versions of PB-LRL in the CartSafe environment.}
    \label{fig:lac_extras}
\end{figure}

\section{Policy Visualisations}

Here we present some additional experimental data. Figure~\ref{fig:cartsafe_policy_visualisations} displays visualisations of the policies learned by all the algorithms in the CartSafe environment during the experiment presented in Section \ref{sec:experiments}, and Figure~\ref{fig:gridnav_policy_visualisations} does the same for the GridNav environment. We have not created any visualisations of the policies learned in the Intersection environment, since it would be difficult to display such a policy as a static 2D plot.

Looking at the plots for the CartSafe environment, we can see that A2C and VaR\_AC seem to reliably learn to simply go straight to the right or straight to the left. LA2C and LPPO are good at staying inside the safe region, but also don't get a lot of reward. LDQN learns policies that generate more reward, but sometimes strays outside the safe region, and RCPO will sometimes learn policies that are safe but overly cautious, and sometimes learn policies that get more reward, but sometimes go outside the safe region. Note that this environment allows for policies that get high reward, but never stray outside the safe region.

Looking at the plots for the GridNav environment, we can see that LA2C, LPPO, and VaR\_AC all seem to mostly reach the goal. LDQN mostly stays around the start square, without reaching the goal. RCPO wanders around the board a lot, but only sometimes reaches the goal. AproPO seems to have a preference for going in the same direction, and sometimes reaches the goal.

\newpage

\begin{figure}[h]
    \centering
\stackunder[5pt]{
		\stackunder[5pt]{\subfloat{{\includegraphics[width=0.04\textwidth]{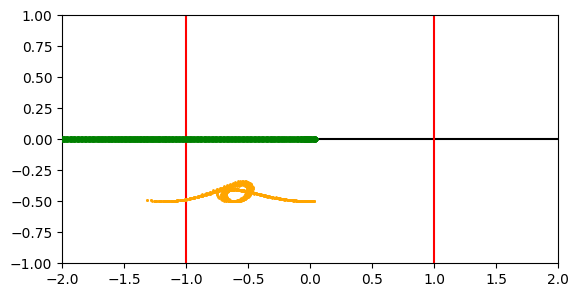}}}}{\tiny{09231}}
		\stackunder[5pt]{\subfloat{{\includegraphics[width=0.04\textwidth]{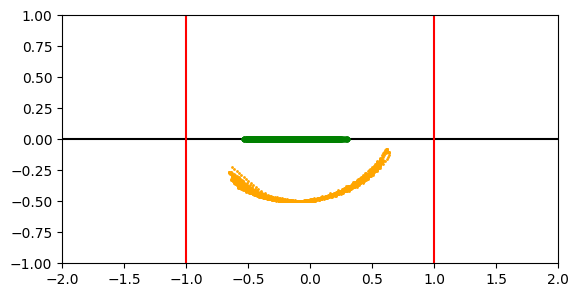}}}}{\tiny{72498}}
		\stackunder[5pt]{\subfloat{{\includegraphics[width=0.04\textwidth]{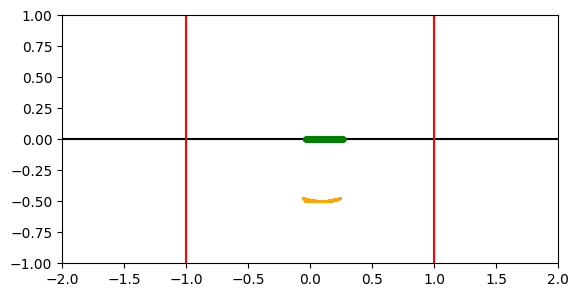}}}}{\tiny{76473}}
		\stackunder[5pt]{\subfloat{{\includegraphics[width=0.04\textwidth]{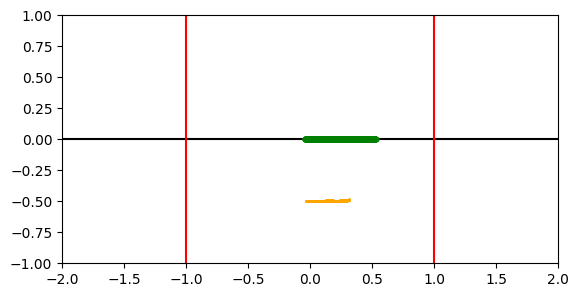}}}}{\tiny{77119}}
		\stackunder[5pt]{\subfloat{{\includegraphics[width=0.04\textwidth]{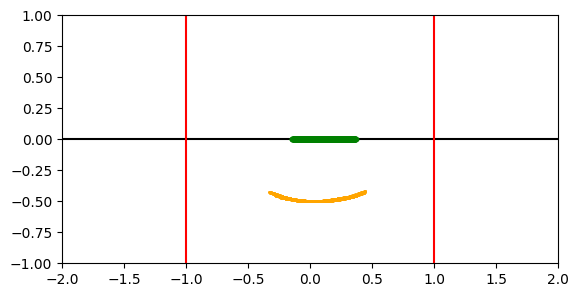}}}}{\tiny{77305}}
		\stackunder[5pt]{\subfloat{{\includegraphics[width=0.04\textwidth]{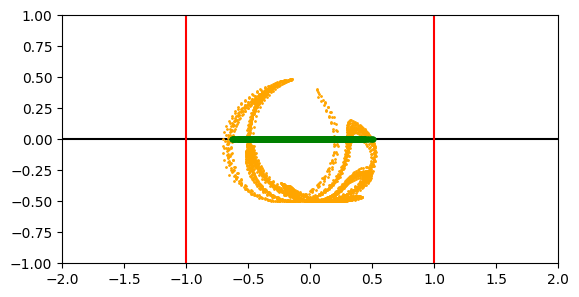}}}}{\tiny{78146}}
		\stackunder[5pt]{\subfloat{{\includegraphics[width=0.04\textwidth]{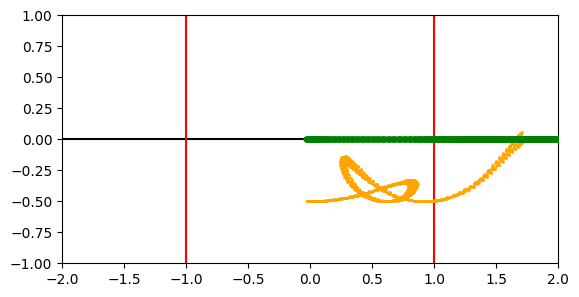}}}}{\tiny{82822}}
		\stackunder[5pt]{\subfloat{{\includegraphics[width=0.04\textwidth]{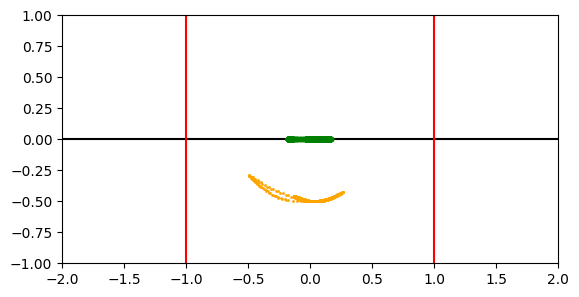}}}}{\tiny{86686}}
		\stackunder[5pt]{\subfloat{{\includegraphics[width=0.04\textwidth]{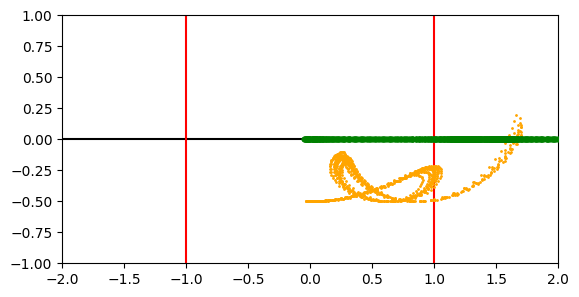}}}}{\tiny{94373}}
		\stackunder[5pt]{\subfloat{{\includegraphics[width=0.04\textwidth]{figures/CartSafe/Policies/DQN-94373-policy.png}}}}{\tiny{94373}}
		\label{fig:cartsafe_dqn}
	}{\tiny{Policy visualisations for DQN in the CartSafe environment.}}
\stackunder[5pt]{
		\stackunder[5pt]{\subfloat{{\includegraphics[width=0.04\textwidth]{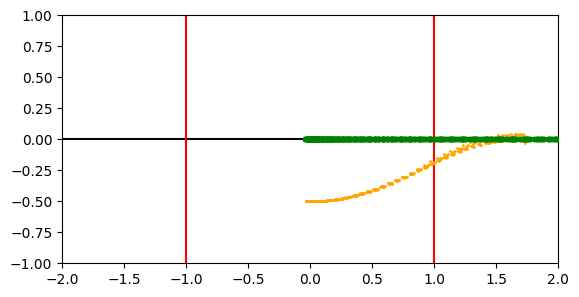}}}}{\tiny{04338}}
		\stackunder[5pt]{\subfloat{{\includegraphics[width=0.04\textwidth]{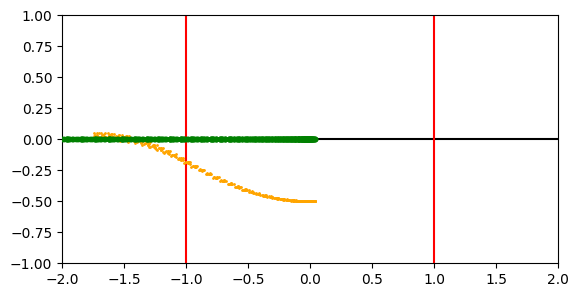}}}}{\tiny{04464}}
		\stackunder[5pt]{\subfloat{{\includegraphics[width=0.04\textwidth]{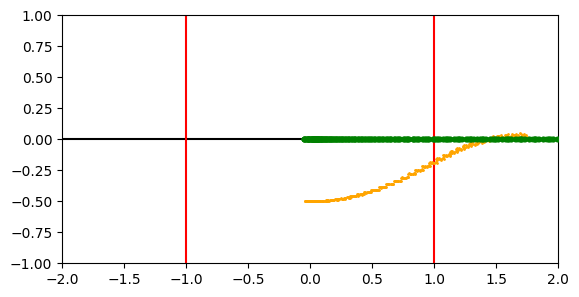}}}}{\tiny{04512}}
		\stackunder[5pt]{\subfloat{{\includegraphics[width=0.04\textwidth]{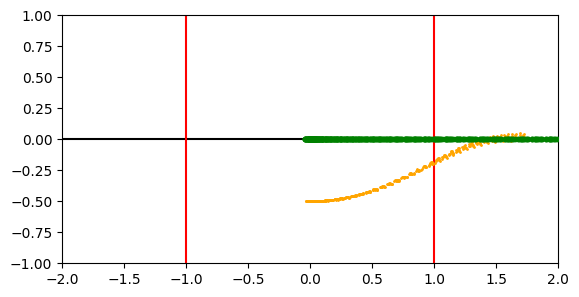}}}}{\tiny{04548}}
		\stackunder[5pt]{\subfloat{{\includegraphics[width=0.04\textwidth]{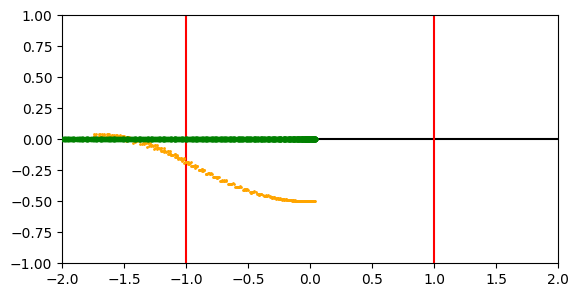}}}}{\tiny{04639}}
		\stackunder[5pt]{\subfloat{{\includegraphics[width=0.04\textwidth]{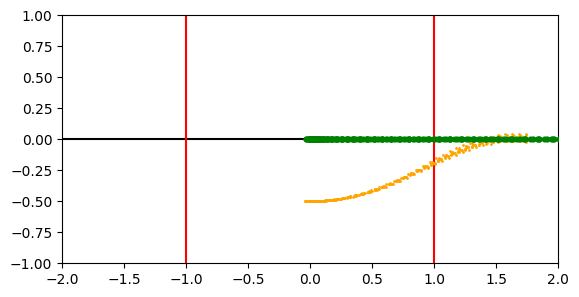}}}}{\tiny{04952}}
		\stackunder[5pt]{\subfloat{{\includegraphics[width=0.04\textwidth]{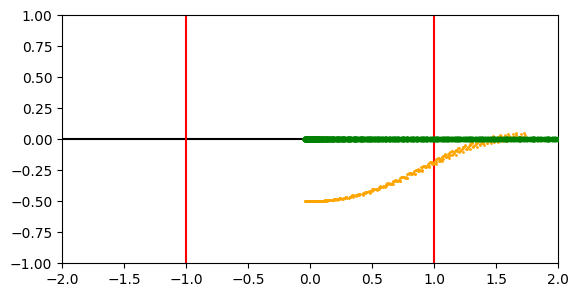}}}}{\tiny{44986}}
		\stackunder[5pt]{\subfloat{{\includegraphics[width=0.04\textwidth]{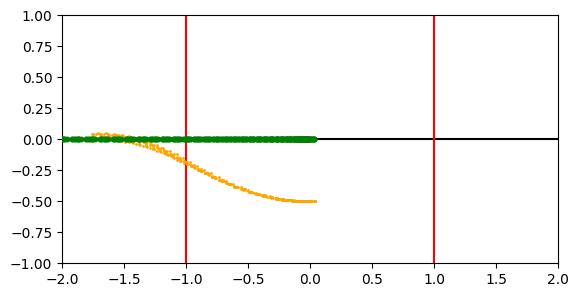}}}}{\tiny{45117}}
		\stackunder[5pt]{\subfloat{{\includegraphics[width=0.04\textwidth]{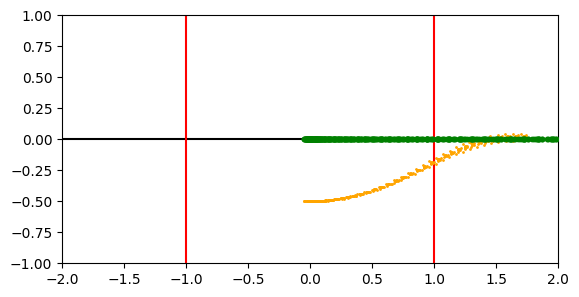}}}}{\tiny{48894}}
		\stackunder[5pt]{\subfloat{{\includegraphics[width=0.04\textwidth]{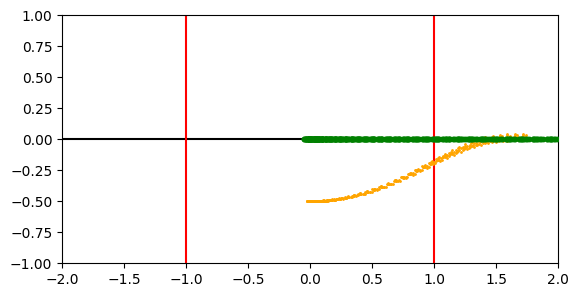}}}}{\tiny{99956}}
		\label{fig:cartsafe_ac}
	}{\tiny{Policy visualisations for A2C in the CartSafe environment.}}
\stackunder[5pt]{
        \stackunder[5pt]{\subfloat{{\includegraphics[width=0.04\textwidth]{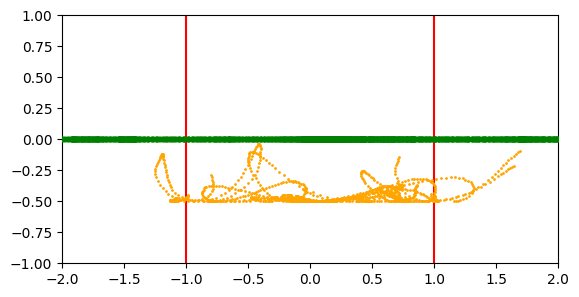}}}}{\tiny{16902}}
		\stackunder[5pt]{\subfloat{{\includegraphics[width=0.04\textwidth]{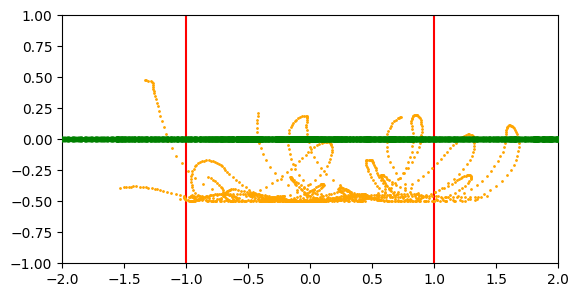}}}}{\tiny{24395}}
		\stackunder[5pt]{\subfloat{{\includegraphics[width=0.04\textwidth]{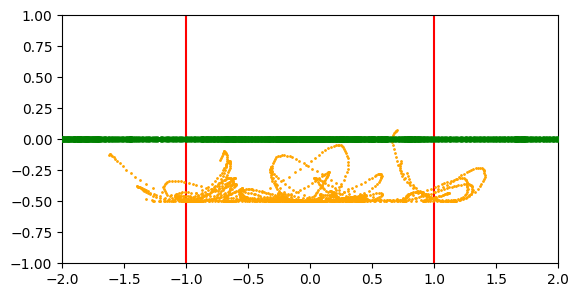}}}}{\tiny{27076}}
		\stackunder[5pt]{\subfloat{{\includegraphics[width=0.04\textwidth]{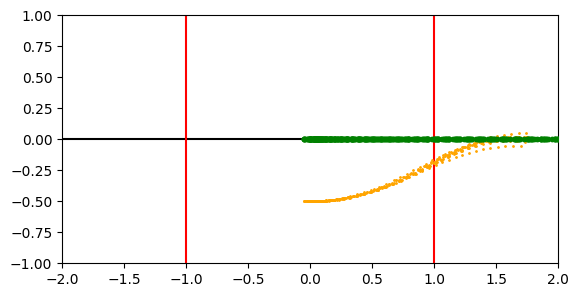}}}}{\tiny{29703}}
		\stackunder[5pt]{\subfloat{{\includegraphics[width=0.04\textwidth]{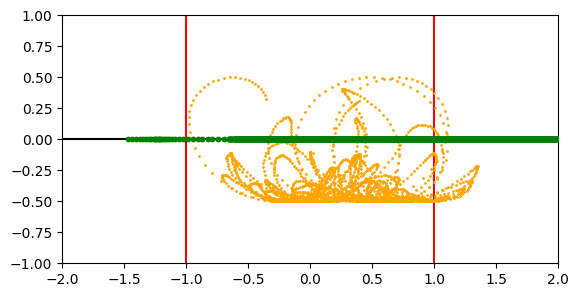}}}}{\tiny{34374}}
		\stackunder[5pt]{\subfloat{{\includegraphics[width=0.04\textwidth]{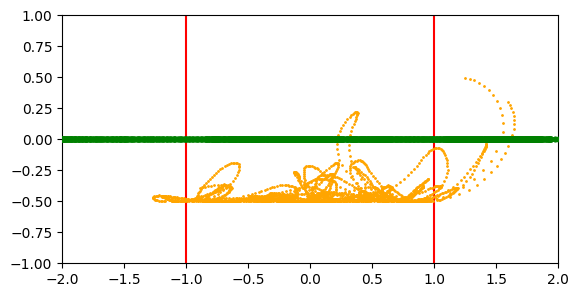}}}}{\tiny{39561}}
		\stackunder[5pt]{\subfloat{{\includegraphics[width=0.04\textwidth]{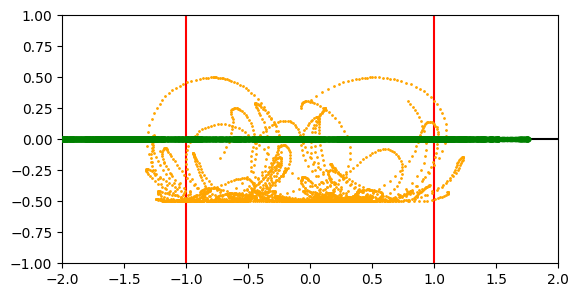}}}}{\tiny{76288}}
		\stackunder[5pt]{\subfloat{{\includegraphics[width=0.04\textwidth]{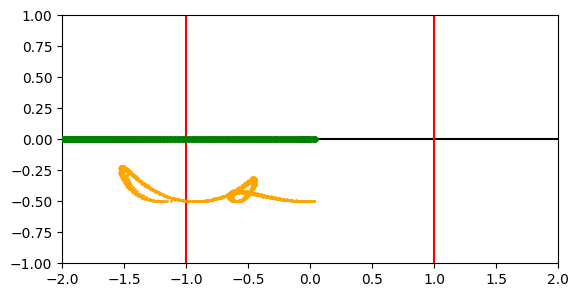}}}}{\tiny{76601}}
		\stackunder[5pt]{\subfloat{{\includegraphics[width=0.04\textwidth]{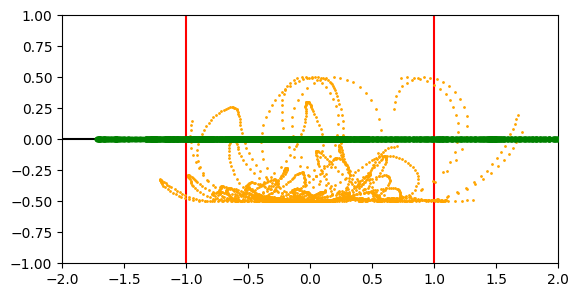}}}}{\tiny{95442}}
		\stackunder[5pt]{\subfloat{{\includegraphics[width=0.04\textwidth]{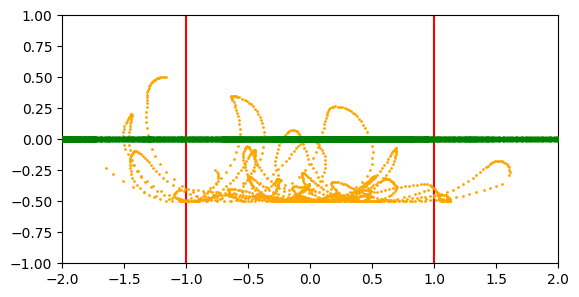}}}}{\tiny{99713}}
		\label{fig:cartsafe_ldqn}
	}{\tiny{Policy visualisations for LDQN in the CartSafe environment.}}
\stackunder[5pt]{
		\stackunder[5pt]{\subfloat{{\includegraphics[width=0.04\textwidth]{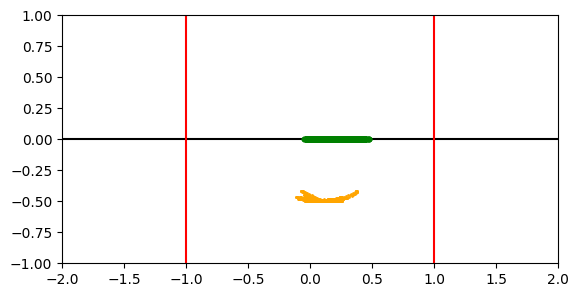}}}}{\tiny{04383}}
		\stackunder[5pt]{\subfloat{{\includegraphics[width=0.04\textwidth]{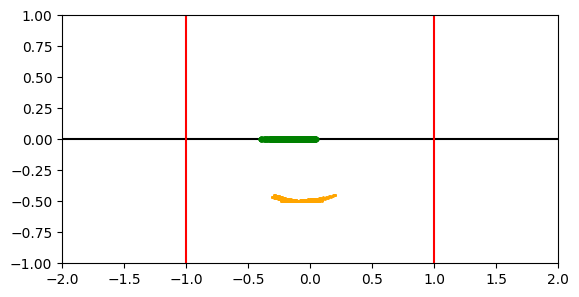}}}}{\tiny{04917}}
		\stackunder[5pt]{\subfloat{{\includegraphics[width=0.04\textwidth]{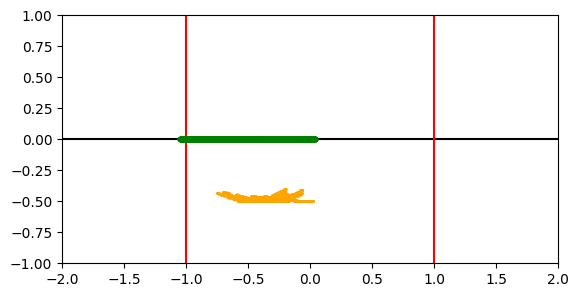}}}}{\tiny{05185}}
		\stackunder[5pt]{\subfloat{{\includegraphics[width=0.04\textwidth]{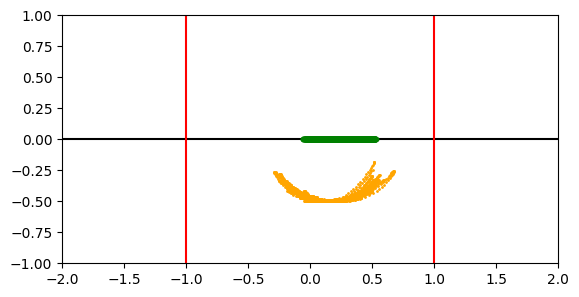}}}}{\tiny{20007}}
		\stackunder[5pt]{\subfloat{{\includegraphics[width=0.04\textwidth]{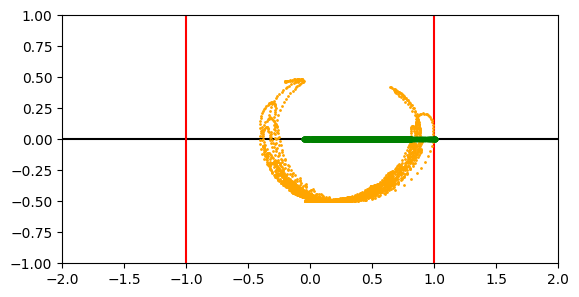}}}}{\tiny{44552}}
		\stackunder[5pt]{\subfloat{{\includegraphics[width=0.04\textwidth]{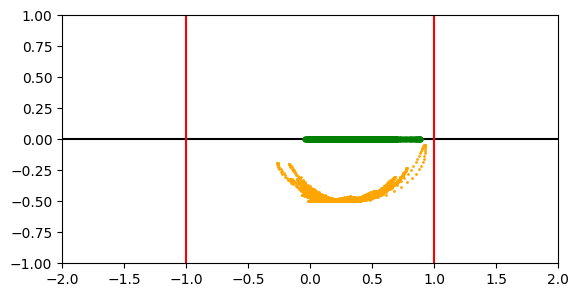}}}}{\tiny{45421}}
		\stackunder[5pt]{\subfloat{{\includegraphics[width=0.04\textwidth]{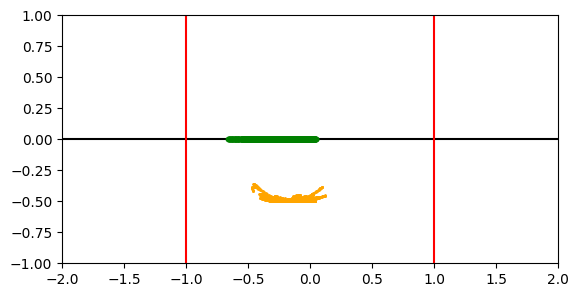}}}}{\tiny{50823}}
		\stackunder[5pt]{\subfloat{{\includegraphics[width=0.04\textwidth]{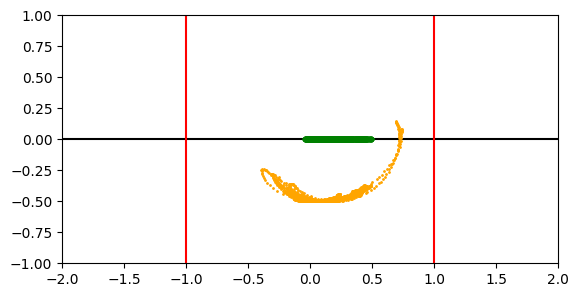}}}}{\tiny{80867}}
		\stackunder[5pt]{\subfloat{{\includegraphics[width=0.04\textwidth]{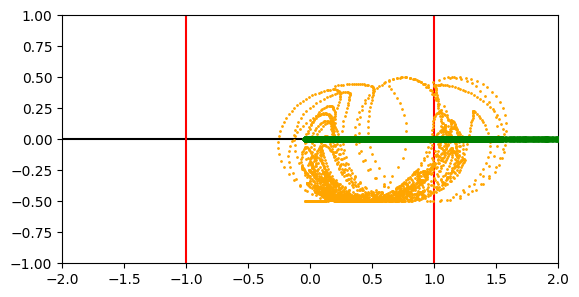}}}}{\tiny{83733}}
		\stackunder[5pt]{\subfloat{{\includegraphics[width=0.04\textwidth]{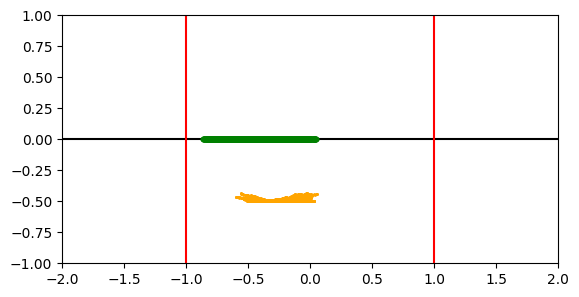}}}}{\tiny{98053}}
		\label{fig:cartsafe_la2c}
	}{\tiny{Policy visualisations for LA2C in the CartSafe environment.}}
\stackunder[5pt]{
		\stackunder[5pt]{\subfloat{{\includegraphics[width=0.04\textwidth]{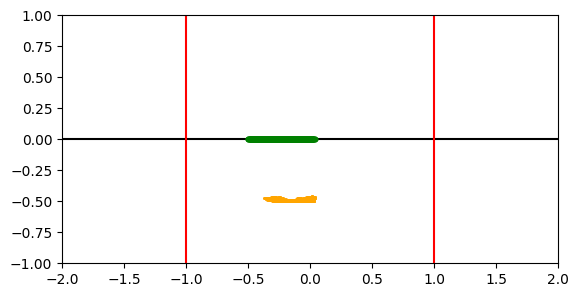}}}}{\tiny{04496}}
		\stackunder[5pt]{\subfloat{{\includegraphics[width=0.04\textwidth]{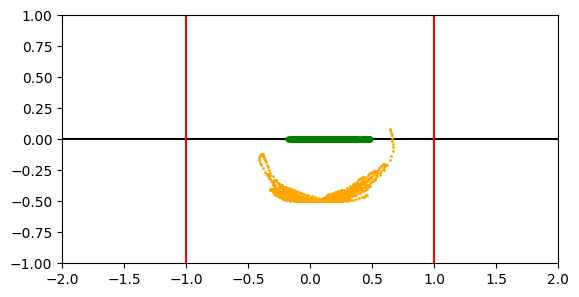}}}}{\tiny{04512}}
		\stackunder[5pt]{\subfloat{{\includegraphics[width=0.04\textwidth]{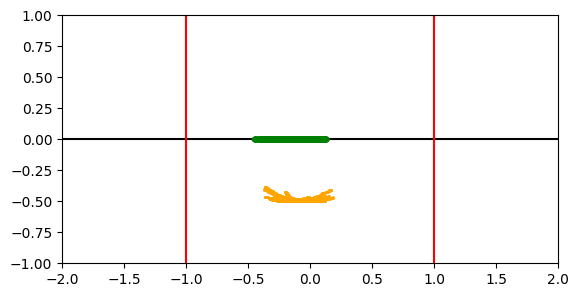}}}}{\tiny{08716}}
		\stackunder[5pt]{\subfloat{{\includegraphics[width=0.04\textwidth]{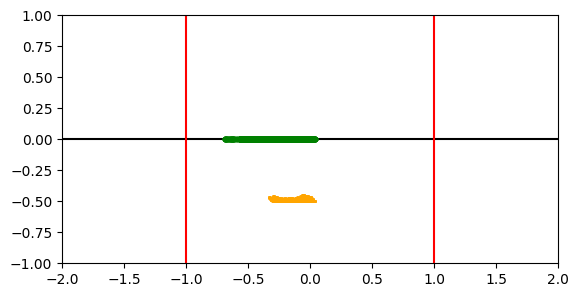}}}}{\tiny{26224}}
		\stackunder[5pt]{\subfloat{{\includegraphics[width=0.04\textwidth]{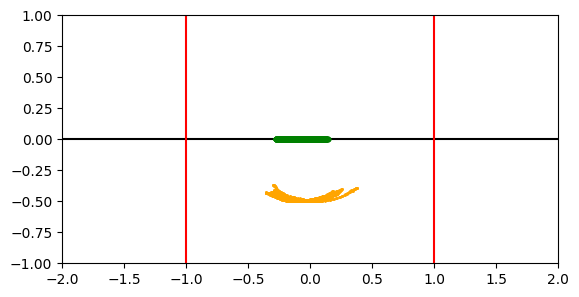}}}}{\tiny{38364}}
		\stackunder[5pt]{\subfloat{{\includegraphics[width=0.04\textwidth]{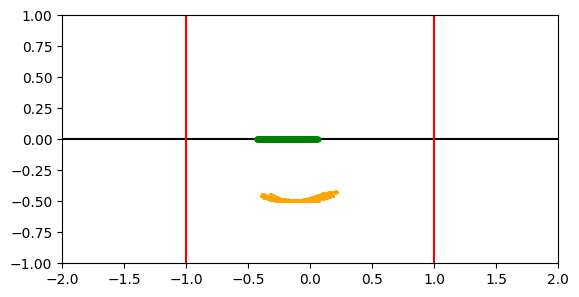}}}}{\tiny{44798}}
		\stackunder[5pt]{\subfloat{{\includegraphics[width=0.04\textwidth]{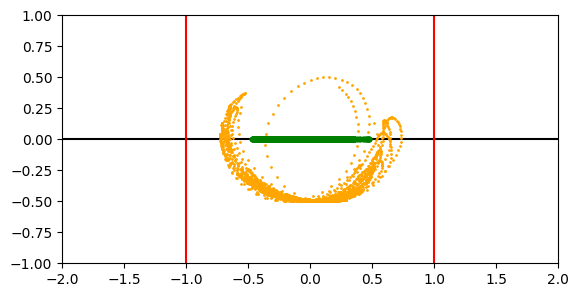}}}}{\tiny{45074}}
		\stackunder[5pt]{\subfloat{{\includegraphics[width=0.04\textwidth]{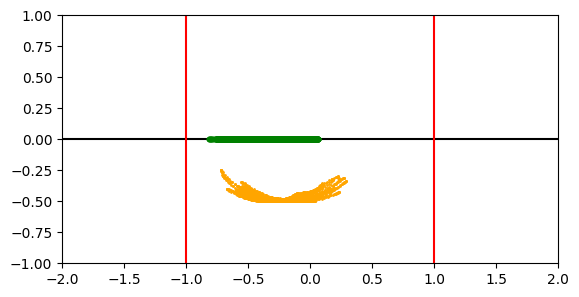}}}}{\tiny{60068}}
		\stackunder[5pt]{\subfloat{{\includegraphics[width=0.04\textwidth]{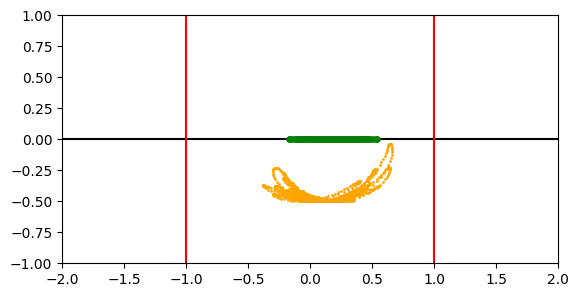}}}}{\tiny{60438}}
		\stackunder[5pt]{\subfloat{{\includegraphics[width=0.04\textwidth]{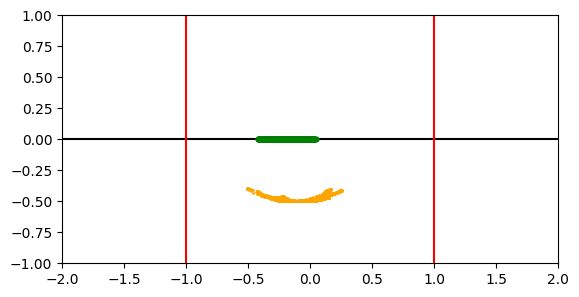}}}}{\tiny{60441}}
		\label{fig:cartsafe_lppo}
	}{\tiny{Policy visualisations for LPPO in the CartSafe environment.}}
\stackunder[5pt]{
		\stackunder[5pt]{\subfloat{{\includegraphics[width=0.04\textwidth]{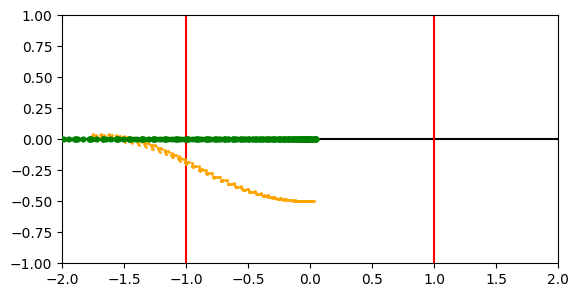}}}}{\tiny{02066}}
		\stackunder[5pt]{\subfloat{{\includegraphics[width=0.04\textwidth]{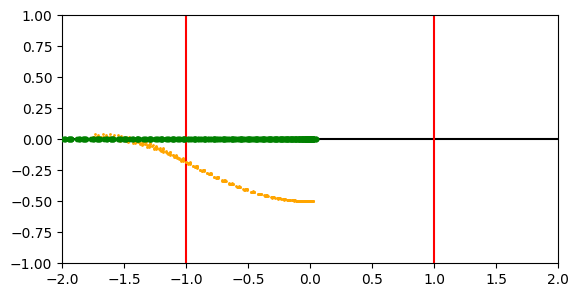}}}}{\tiny{04317}}
		\stackunder[5pt]{\subfloat{{\includegraphics[width=0.04\textwidth]{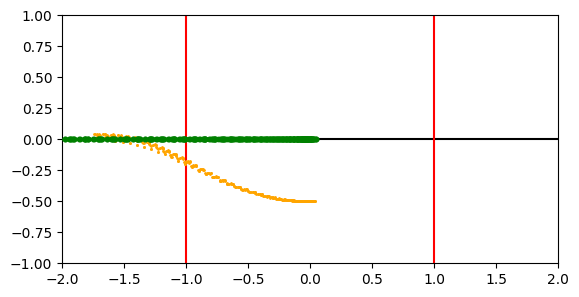}}}}{\tiny{04408}}
		\stackunder[5pt]{\subfloat{{\includegraphics[width=0.04\textwidth]{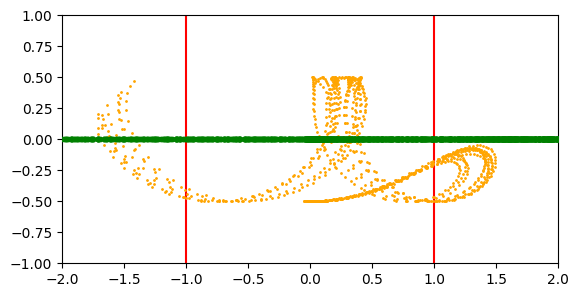}}}}{\tiny{04443}}
		\stackunder[5pt]{\subfloat{{\includegraphics[width=0.04\textwidth]{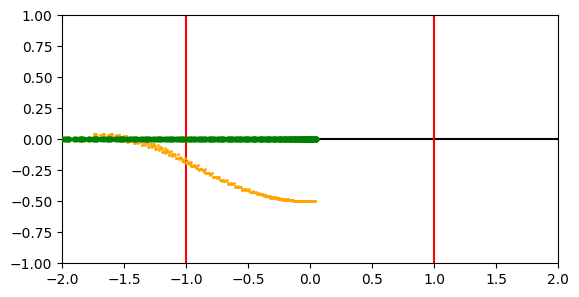}}}}{\tiny{04484}}
		\stackunder[5pt]{\subfloat{{\includegraphics[width=0.04\textwidth]{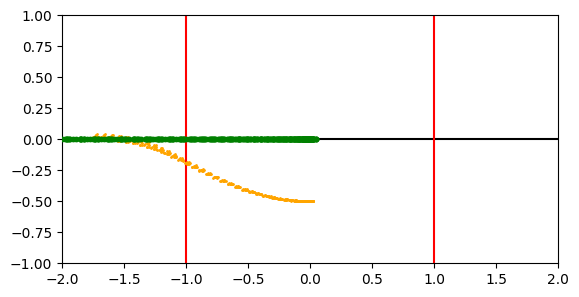}}}}{\tiny{10488}}
		\stackunder[5pt]{\subfloat{{\includegraphics[width=0.04\textwidth]{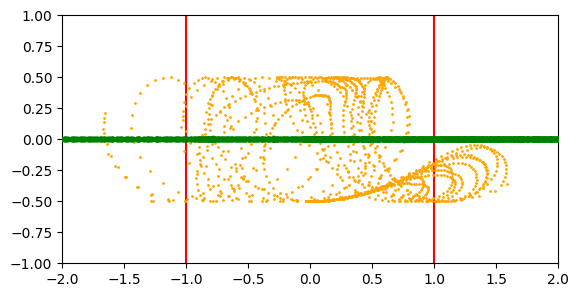}}}}{\tiny{13725}}
		\stackunder[5pt]{\subfloat{{\includegraphics[width=0.04\textwidth]{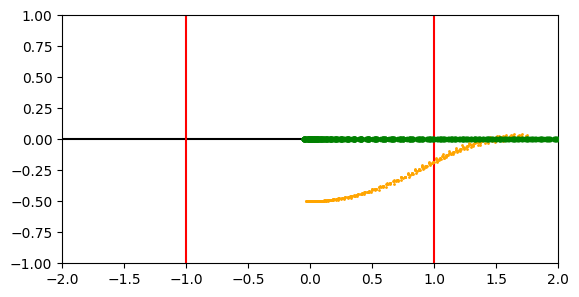}}}}{\tiny{44478}}
		\stackunder[5pt]{\subfloat{{\includegraphics[width=0.04\textwidth]{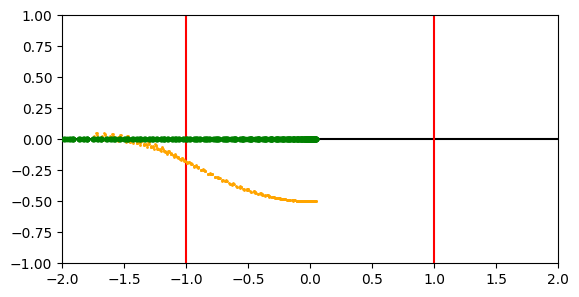}}}}{\tiny{60446}}
		\stackunder[5pt]{\subfloat{{\includegraphics[width=0.04\textwidth]{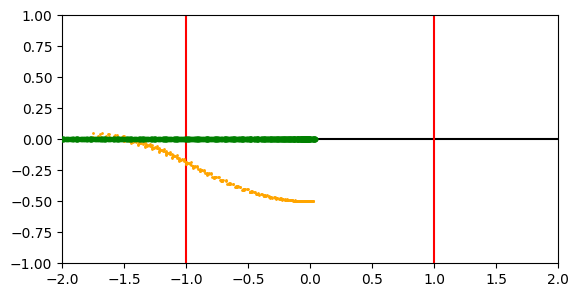}}}}{\tiny{74344}}
		\label{fig:cartsafe_var_ac}
	}{\tiny{Policy visualisations for VaR\_AC in the CartSafe environment.}}
\stackunder[5pt]{
		\stackunder[5pt]{\subfloat{{\includegraphics[width=0.04\textwidth]{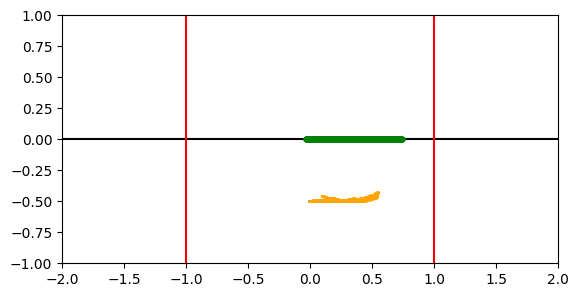}}}}{\tiny{02738}}
		\stackunder[5pt]{\subfloat{{\includegraphics[width=0.04\textwidth]{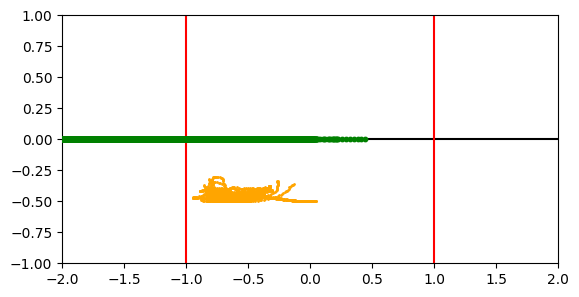}}}}{\tiny{04336}}
		\stackunder[5pt]{\subfloat{{\includegraphics[width=0.04\textwidth]{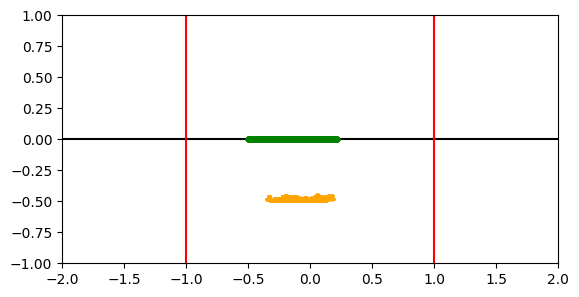}}}}{\tiny{04411}}
		\stackunder[5pt]{\subfloat{{\includegraphics[width=0.04\textwidth]{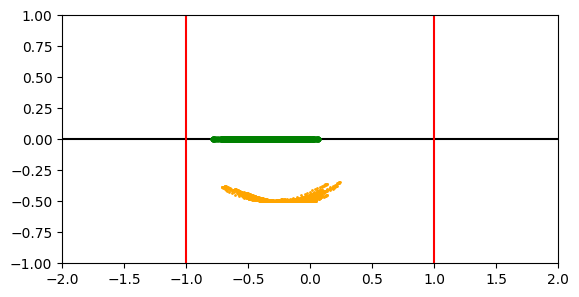}}}}{\tiny{04536}}
		\stackunder[5pt]{\subfloat{{\includegraphics[width=0.04\textwidth]{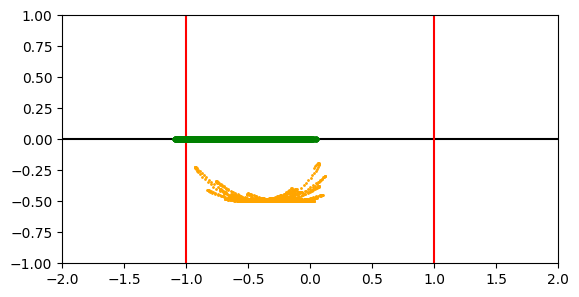}}}}{\tiny{04643}}
		\stackunder[5pt]{\subfloat{{\includegraphics[width=0.04\textwidth]{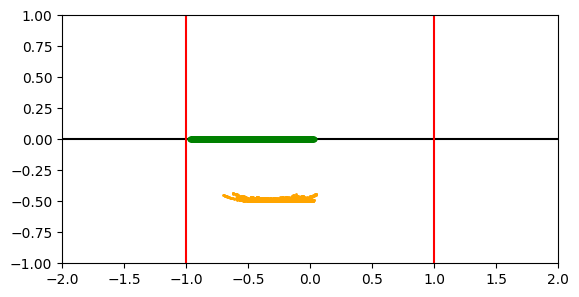}}}}{\tiny{04915}}
		\stackunder[5pt]{\subfloat{{\includegraphics[width=0.04\textwidth]{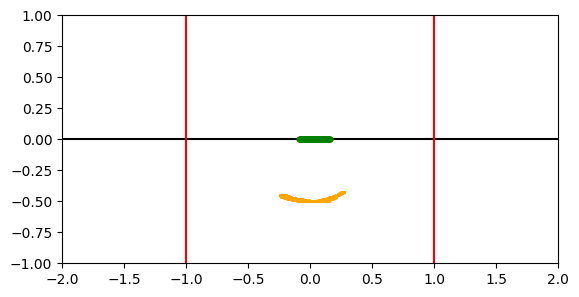}}}}{\tiny{44643}}
		\stackunder[5pt]{\subfloat{{\includegraphics[width=0.04\textwidth]{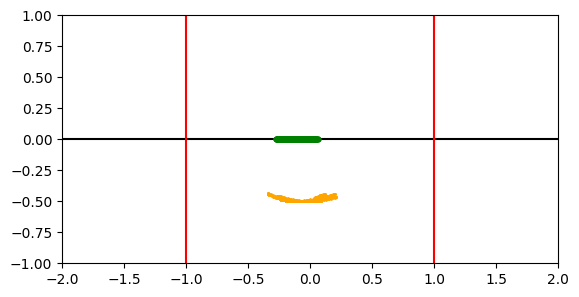}}}}{\tiny{45532}}
		\stackunder[5pt]{\subfloat{{\includegraphics[width=0.04\textwidth]{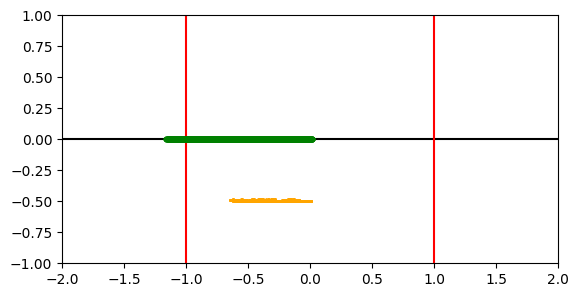}}}}{\tiny{48825}}
		\stackunder[5pt]{\subfloat{{\includegraphics[width=0.04\textwidth]{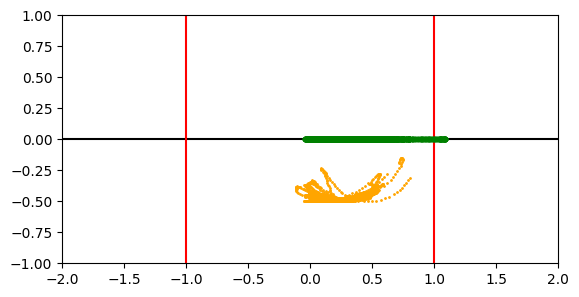}}}}{\tiny{60446}}
		\label{fig:cartsafe_rcpo}
	}{\tiny{Policy visualisations for RCPO in the CartSafe environment.}}
	\stackunder[5pt]{
		\stackunder[5pt]{\subfloat{{\includegraphics[width=0.04\textwidth]{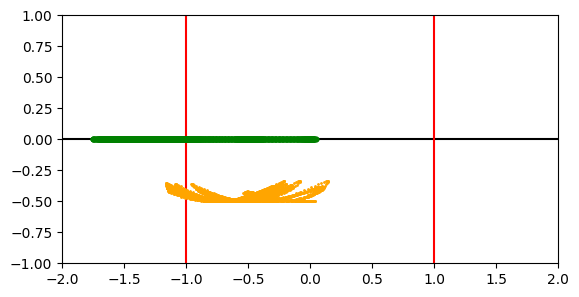}}}}{\tiny{03451}}
		\stackunder[5pt]{\subfloat{{\includegraphics[width=0.04\textwidth]{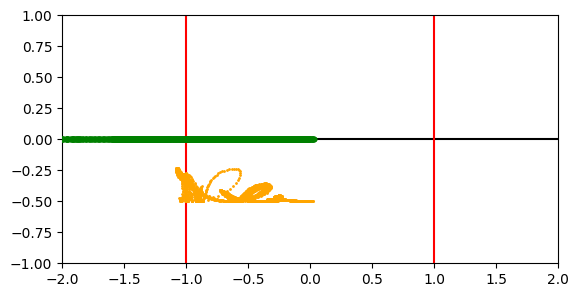}}}}{\tiny{04274}}
		\stackunder[5pt]{\subfloat{{\includegraphics[width=0.04\textwidth]{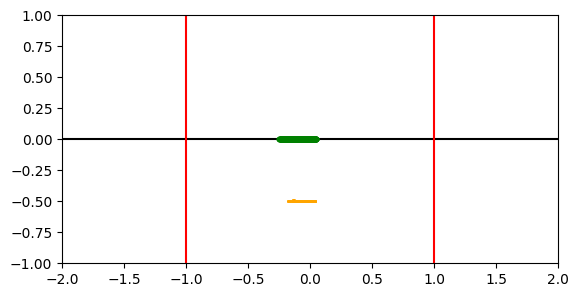}}}}{\tiny{04567}}
		\stackunder[5pt]{\subfloat{{\includegraphics[width=0.04\textwidth]{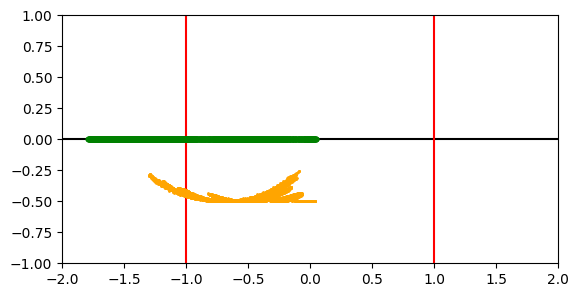}}}}{\tiny{04913}}
		\stackunder[5pt]{\subfloat{{\includegraphics[width=0.04\textwidth]{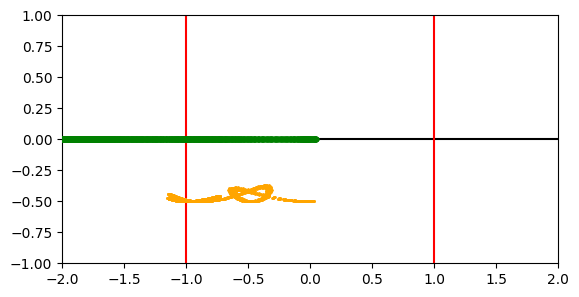}}}}{\tiny{05177}}
		\stackunder[5pt]{\subfloat{{\includegraphics[width=0.04\textwidth]{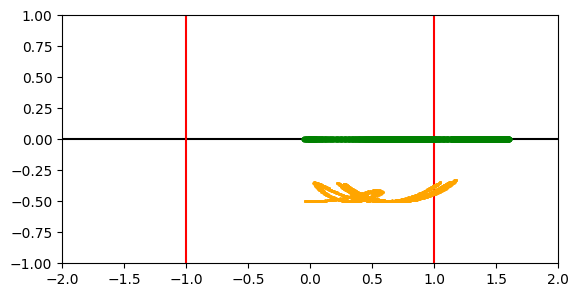}}}}{\tiny{28312}}
		\stackunder[5pt]{\subfloat{{\includegraphics[width=0.04\textwidth]{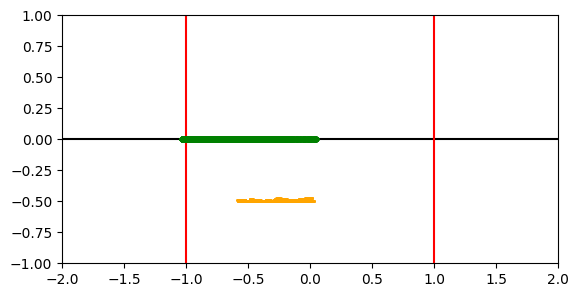}}}}{\tiny{44478}}
		\stackunder[5pt]{\subfloat{{\includegraphics[width=0.04\textwidth]{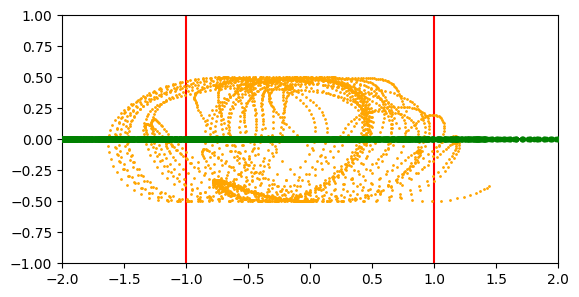}}}}{\tiny{44588}}
		\stackunder[5pt]{\subfloat{{\includegraphics[width=0.04\textwidth]{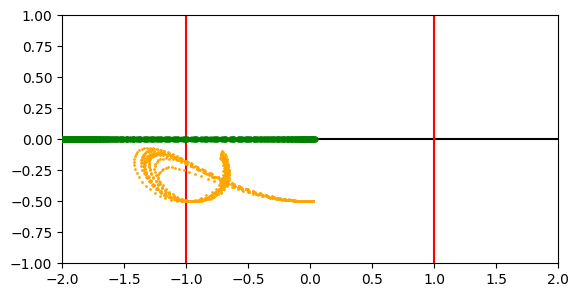}}}}{\tiny{60203}}
		\stackunder[5pt]{\subfloat{{\includegraphics[width=0.04\textwidth]{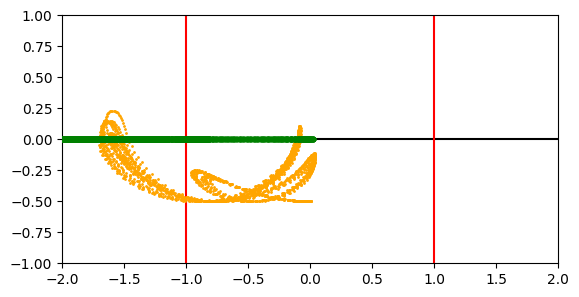}}}}{\tiny{73893}}
		\label{fig:cartsafe_apropo}
	}{\tiny{Policy visualisations for AproPO in the CartSafe environment.}}
    \caption{This figure displays visualisations of the policies learned by the algorithmsin the CartSafe environment. Each plot corresponds to a policy learned by an algorithm, and is labelled with the random seed used in that run of the experiment.They were generated by running each policy ten times, and drawing trajectory of the cart body (in green) and the tip of the pole (in orange).}
\label{fig:cartsafe_policy_visualisations}
\end{figure}

\newpage

\begin{figure}[h]
	\centering
\stackunder[5pt]{
		\stackunder[5pt]{\subfloat{{\includegraphics[width=0.04\textwidth]{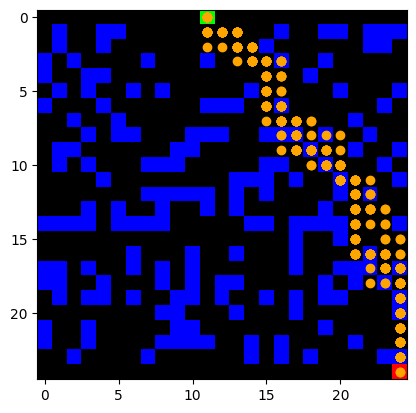}}}}{\tiny{36123}}
		\stackunder[5pt]{\subfloat{{\includegraphics[width=0.04\textwidth]{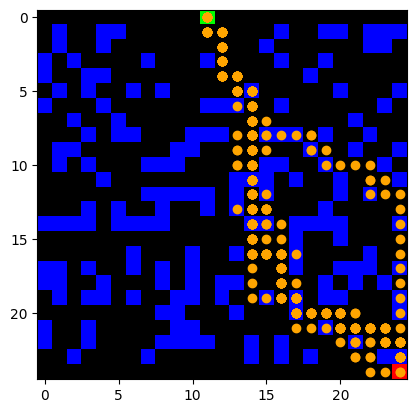}}}}{\tiny{84193}}
		\stackunder[5pt]{\subfloat{{\includegraphics[width=0.04\textwidth]{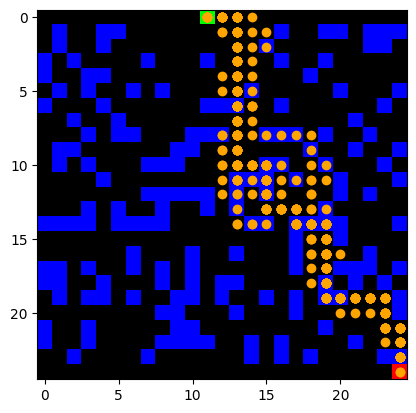}}}}{\tiny{73132}}
		\stackunder[5pt]{\subfloat{{\includegraphics[width=0.04\textwidth]{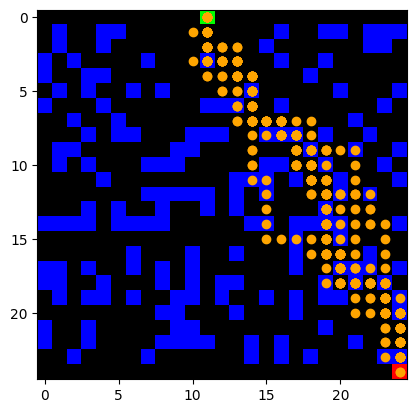}}}}{\tiny{93386}}
		\stackunder[5pt]{\subfloat{{\includegraphics[width=0.04\textwidth]{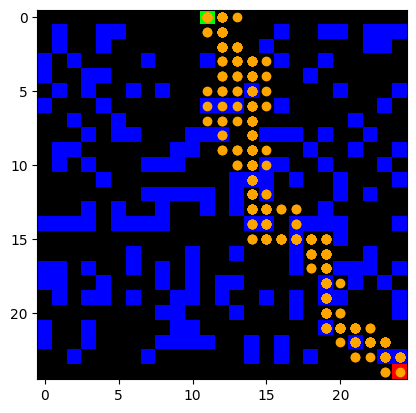}}}}{\tiny{89232}}
		\stackunder[5pt]{\subfloat{{\includegraphics[width=0.04\textwidth]{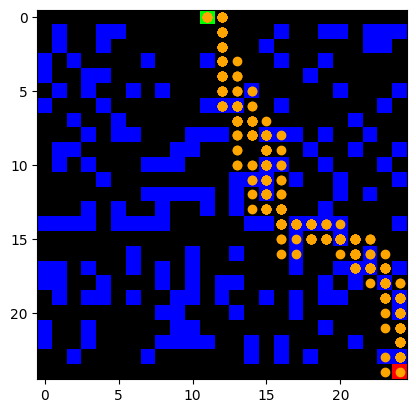}}}}{\tiny{60262}}
		\stackunder[5pt]{\subfloat{{\includegraphics[width=0.04\textwidth]{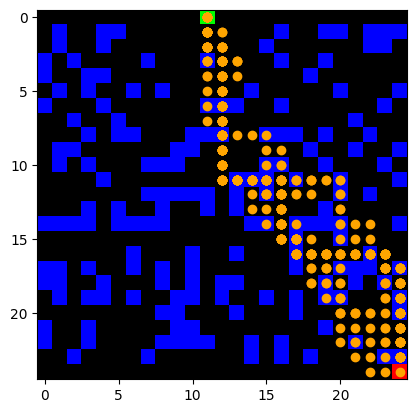}}}}{\tiny{72034}}
		\stackunder[5pt]{\subfloat{{\includegraphics[width=0.04\textwidth]{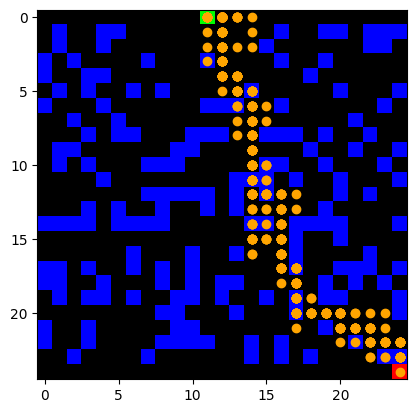}}}}{\tiny{68335}}
		\stackunder[5pt]{\subfloat{{\includegraphics[width=0.04\textwidth]{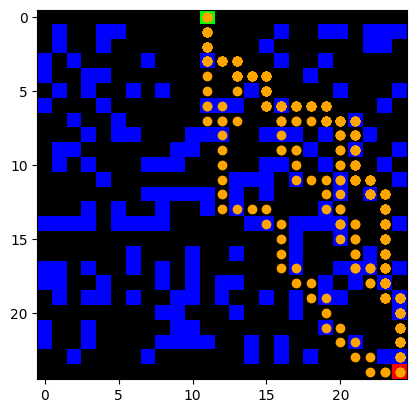}}}}{\tiny{45641}}
		\stackunder[5pt]{\subfloat{{\includegraphics[width=0.04\textwidth]{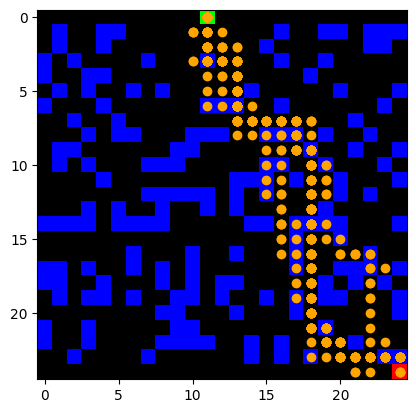}}}}{\tiny{65287}}
		\label{fig:gridnav_dqn}
	}{\tiny{Policy visualisations for DQN in the GridNav environment.}}
\stackunder[5pt]{
		\stackunder[5pt]{\subfloat{{\includegraphics[width=0.04\textwidth]{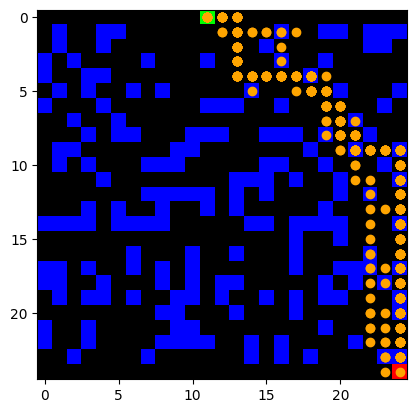}}}}{\tiny{88698}}
		\stackunder[5pt]{\subfloat{{\includegraphics[width=0.04\textwidth]{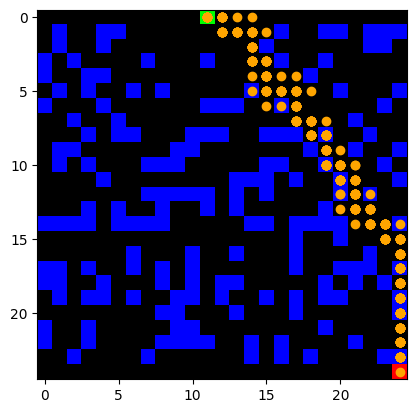}}}}{\tiny{89514}}
		\stackunder[5pt]{\subfloat{{\includegraphics[width=0.04\textwidth]{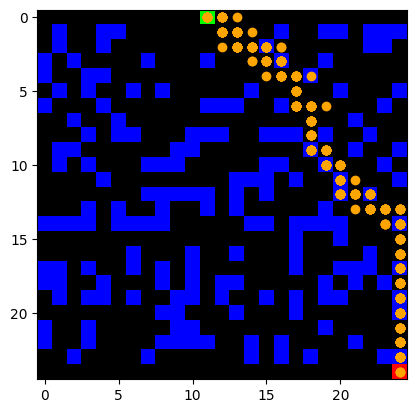}}}}{\tiny{57523}}
		\stackunder[5pt]{\subfloat{{\includegraphics[width=0.04\textwidth]{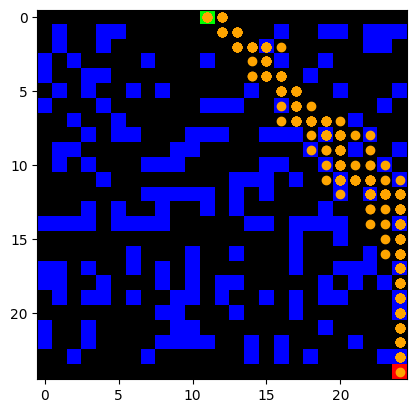}}}}{\tiny{76875}}
		\stackunder[5pt]{\subfloat{{\includegraphics[width=0.04\textwidth]{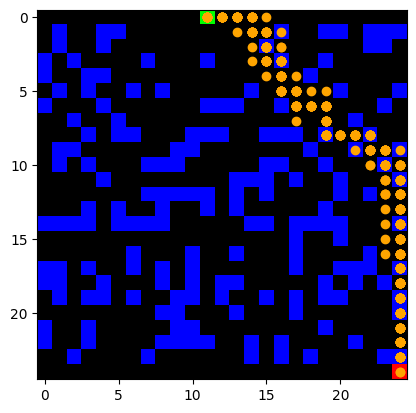}}}}{\tiny{43837}}
		\stackunder[5pt]{\subfloat{{\includegraphics[width=0.04\textwidth]{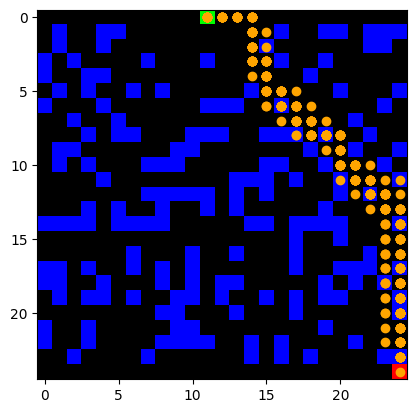}}}}{\tiny{52617}}
		\stackunder[5pt]{\subfloat{{\includegraphics[width=0.04\textwidth]{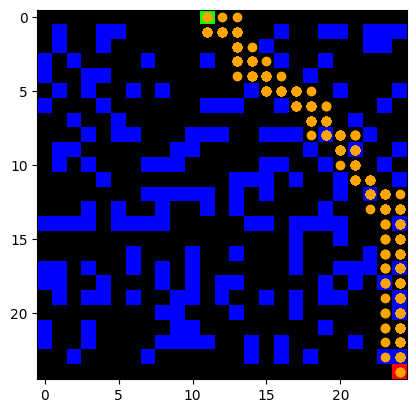}}}}{\tiny{87835}}
		\stackunder[5pt]{\subfloat{{\includegraphics[width=0.04\textwidth]{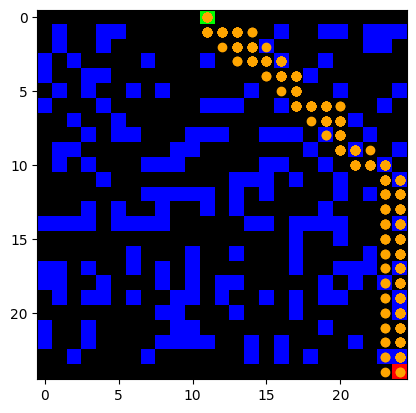}}}}{\tiny{68489}}
		\stackunder[5pt]{\subfloat{{\includegraphics[width=0.04\textwidth]{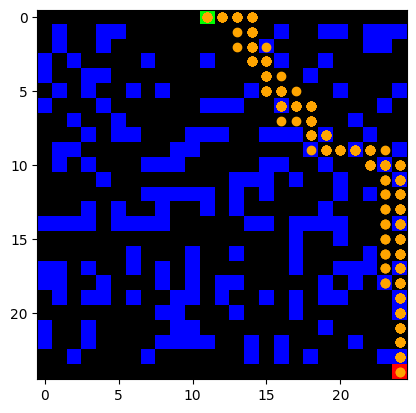}}}}{\tiny{96176}}
		\stackunder[5pt]{\subfloat{{\includegraphics[width=0.04\textwidth]{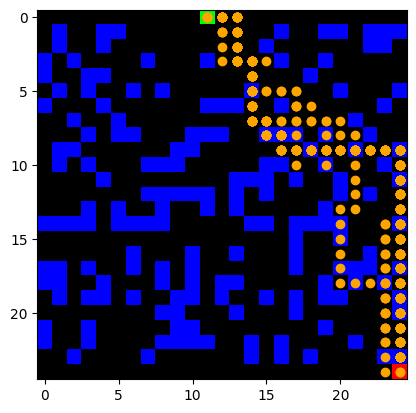}}}}{\tiny{50548}}
		\label{fig:gridnav_ac}
	}{\tiny{Policy visualisations for A2C in the GridNav environment.}}
\stackunder[5pt]{
		\stackunder[5pt]{\subfloat{{\includegraphics[width=0.04\textwidth]{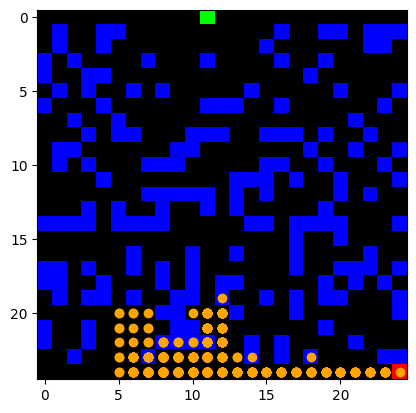}}}}{\tiny{63072}}
		\stackunder[5pt]{\subfloat{{\includegraphics[width=0.04\textwidth]{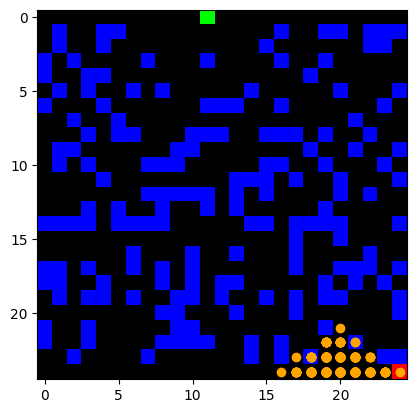}}}}{\tiny{99606}}
		\stackunder[5pt]{\subfloat{{\includegraphics[width=0.04\textwidth]{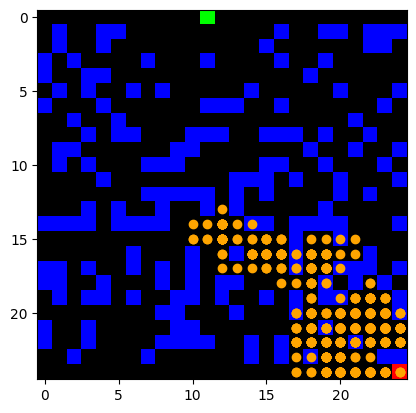}}}}{\tiny{67403}}
		\stackunder[5pt]{\subfloat{{\includegraphics[width=0.04\textwidth]{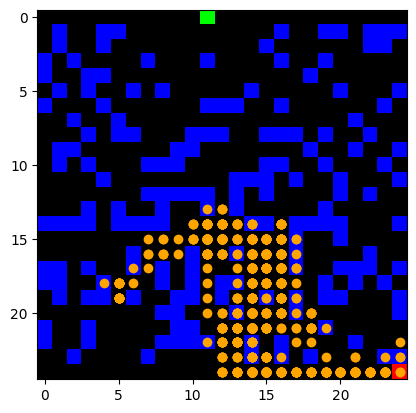}}}}{\tiny{86338}}
		\stackunder[5pt]{\subfloat{{\includegraphics[width=0.04\textwidth]{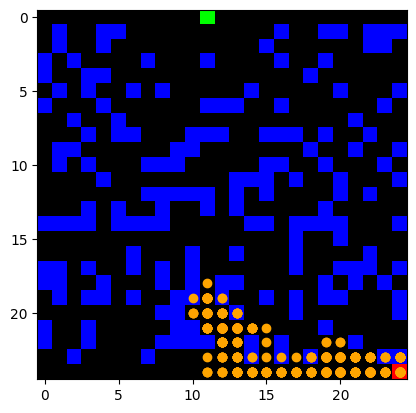}}}}{\tiny{60598}}
		\stackunder[5pt]{\subfloat{{\includegraphics[width=0.04\textwidth]{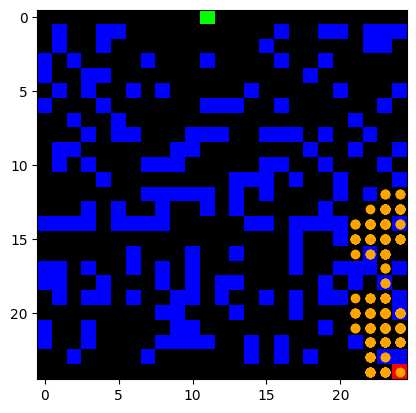}}}}{\tiny{97542}}
		\stackunder[5pt]{\subfloat{{\includegraphics[width=0.04\textwidth]{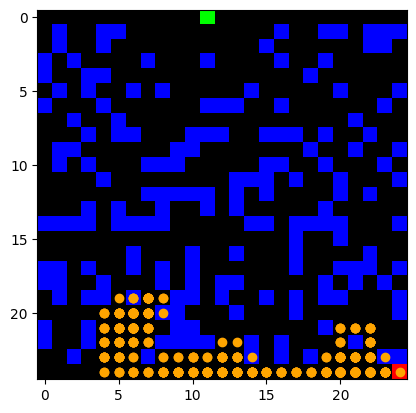}}}}{\tiny{15345}}
		\stackunder[5pt]{\subfloat{{\includegraphics[width=0.04\textwidth]{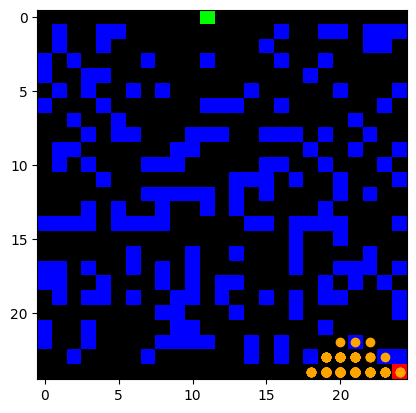}}}}{\tiny{18917}}
		\stackunder[5pt]{\subfloat{{\includegraphics[width=0.04\textwidth]{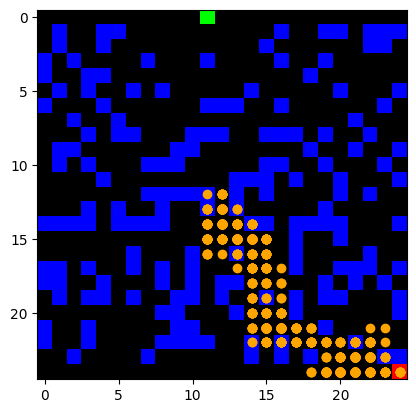}}}}{\tiny{85219}}
		\stackunder[5pt]{\subfloat{{\includegraphics[width=0.04\textwidth]{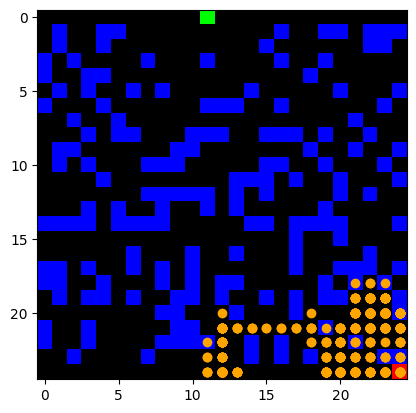}}}}{\tiny{71943}}
		\label{fig:gridnav_ldqn}
	}{\tiny{Policy visualisations for LDQN in the GridNav environment.}}
\stackunder[5pt]{
		\stackunder[5pt]{\subfloat{{\includegraphics[width=0.04\textwidth]{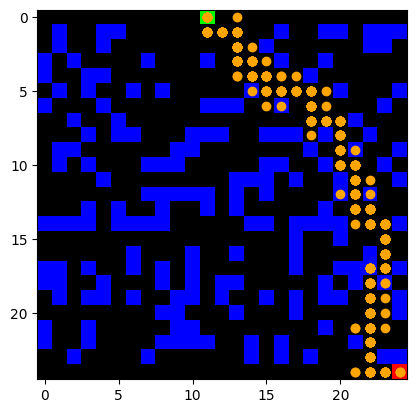}}}}{\tiny{26885}}
		\stackunder[5pt]{\subfloat{{\includegraphics[width=0.04\textwidth]{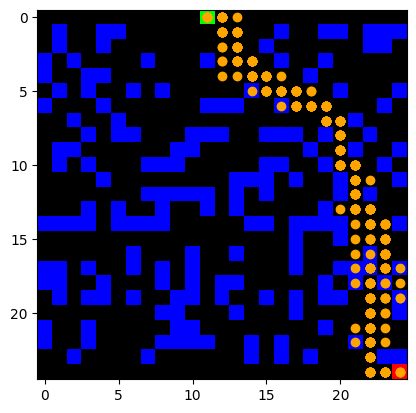}}}}{\tiny{41308}}
		\stackunder[5pt]{\subfloat{{\includegraphics[width=0.04\textwidth]{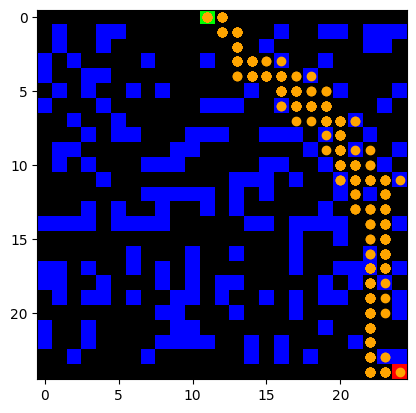}}}}{\tiny{94614}}
		\stackunder[5pt]{\subfloat{{\includegraphics[width=0.04\textwidth]{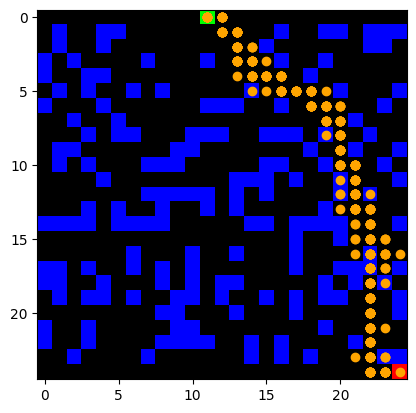}}}}{\tiny{78343}}
		\stackunder[5pt]{\subfloat{{\includegraphics[width=0.04\textwidth]{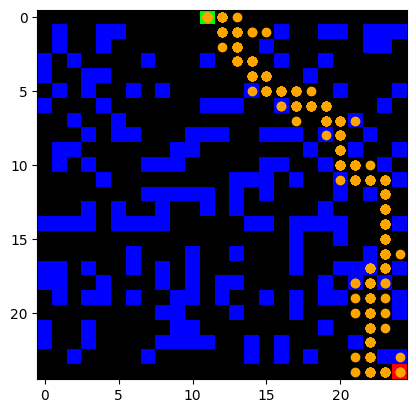}}}}{\tiny{05945}}
		\stackunder[5pt]{\subfloat{{\includegraphics[width=0.04\textwidth]{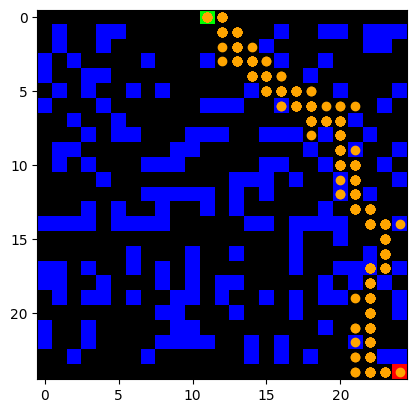}}}}{\tiny{93562}}
		\stackunder[5pt]{\subfloat{{\includegraphics[width=0.04\textwidth]{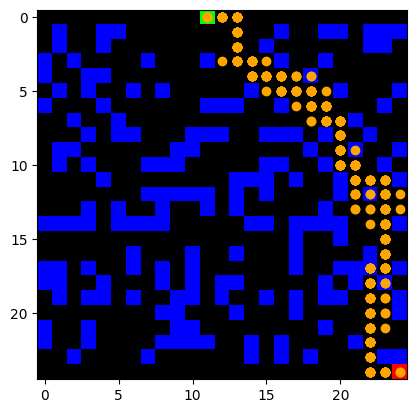}}}}{\tiny{92684}}
		\stackunder[5pt]{\subfloat{{\includegraphics[width=0.04\textwidth]{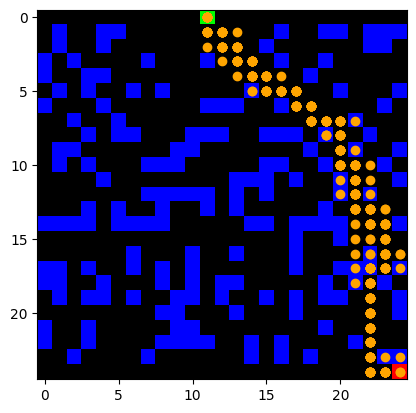}}}}{\tiny{15632}}
		\stackunder[5pt]{\subfloat{{\includegraphics[width=0.04\textwidth]{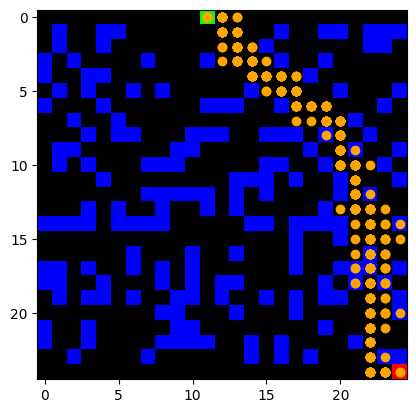}}}}{\tiny{71868}}
		\stackunder[5pt]{\subfloat{{\includegraphics[width=0.04\textwidth]{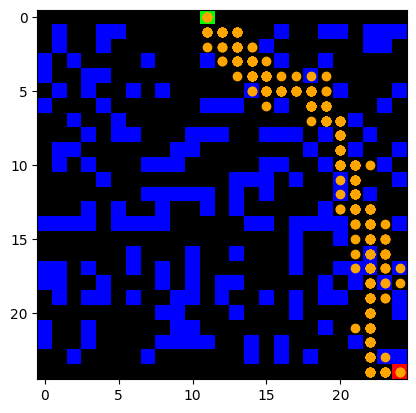}}}}{\tiny{73066}}
		\label{fig:gridnav_la2c}
	}{\tiny{Policy visualisations for LA2C in the GridNav environment.}}
\stackunder[5pt]{
		\stackunder[5pt]{\subfloat{{\includegraphics[width=0.04\textwidth]{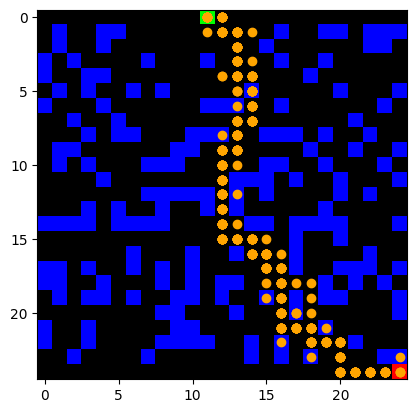}}}}{\tiny{97499}}
		\stackunder[5pt]{\subfloat{{\includegraphics[width=0.04\textwidth]{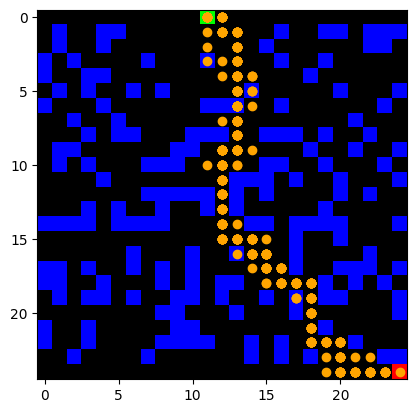}}}}{\tiny{21563}}
		\stackunder[5pt]{\subfloat{{\includegraphics[width=0.04\textwidth]{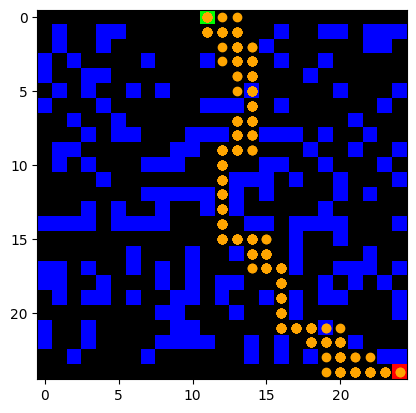}}}}{\tiny{04486}}
		\stackunder[5pt]{\subfloat{{\includegraphics[width=0.04\textwidth]{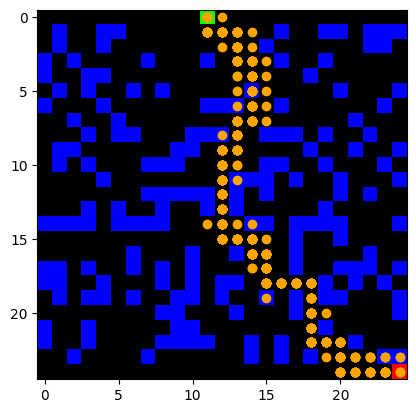}}}}{\tiny{70288}}
		\stackunder[5pt]{\subfloat{{\includegraphics[width=0.04\textwidth]{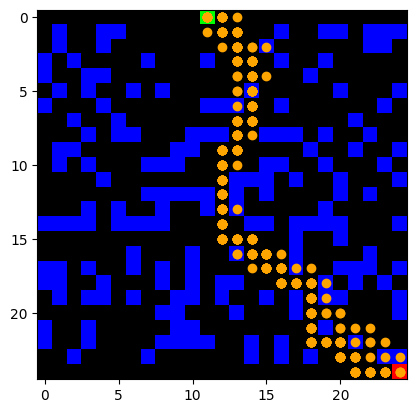}}}}{\tiny{99387}}
		\stackunder[5pt]{\subfloat{{\includegraphics[width=0.04\textwidth]{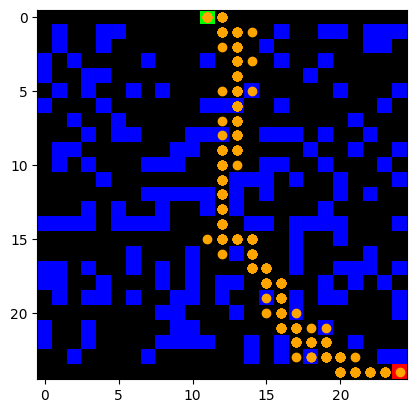}}}}{\tiny{90305}}
		\stackunder[5pt]{\subfloat{{\includegraphics[width=0.04\textwidth]{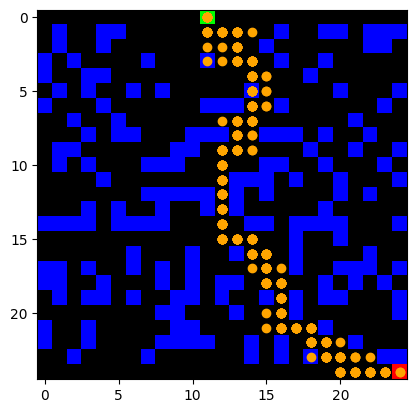}}}}{\tiny{07494}}
		\stackunder[5pt]{\subfloat{{\includegraphics[width=0.04\textwidth]{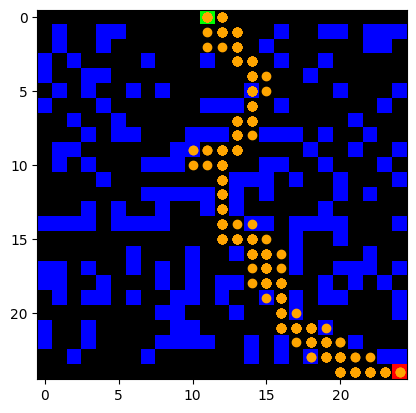}}}}{\tiny{47423}}
		\stackunder[5pt]{\subfloat{{\includegraphics[width=0.04\textwidth]{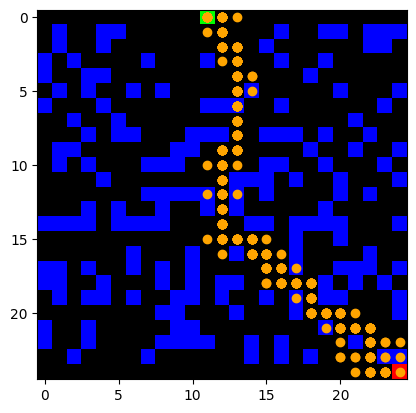}}}}{\tiny{43696}}
		\stackunder[5pt]{\subfloat{{\includegraphics[width=0.04\textwidth]{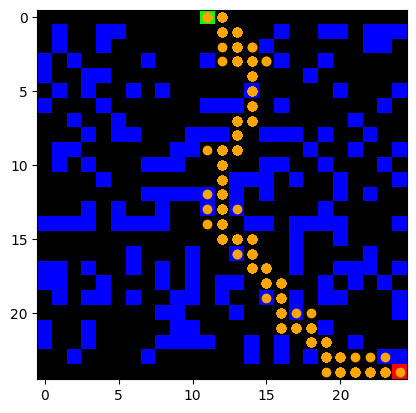}}}}{\tiny{74696}}
		\label{fig:gridnav_lppo}
	}{\tiny{Policy visualisations for LPPO in the GridNav environment.}}
\stackunder[5pt]{
		\stackunder[5pt]{\subfloat{{\includegraphics[width=0.04\textwidth]{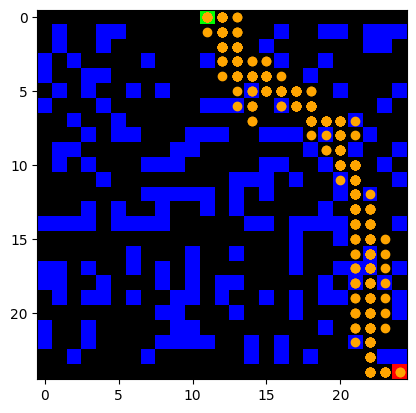}}}}{\tiny{82257}}
		\stackunder[5pt]{\subfloat{{\includegraphics[width=0.04\textwidth]{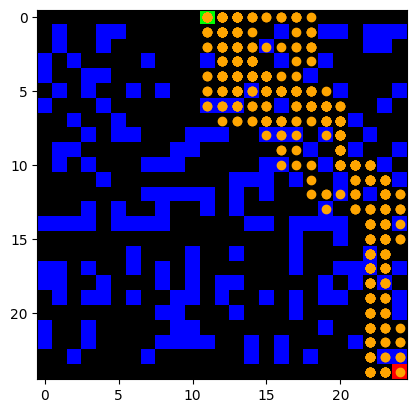}}}}{\tiny{23775}}
		\stackunder[5pt]{\subfloat{{\includegraphics[width=0.04\textwidth]{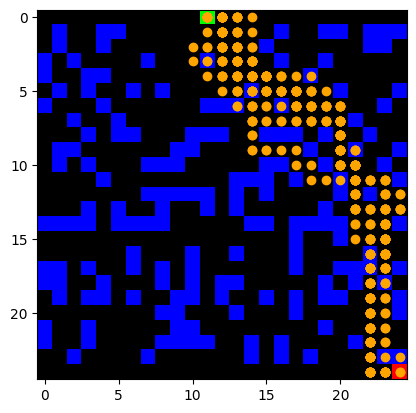}}}}{\tiny{98557}}
		\stackunder[5pt]{\subfloat{{\includegraphics[width=0.04\textwidth]{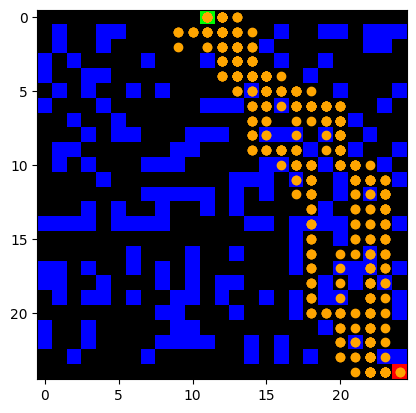}}}}{\tiny{99376}}
		\stackunder[5pt]{\subfloat{{\includegraphics[width=0.04\textwidth]{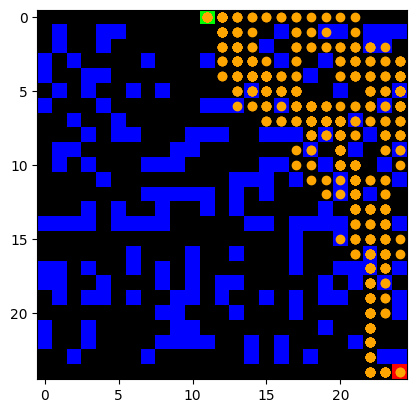}}}}{\tiny{74841}}
		\stackunder[5pt]{\subfloat{{\includegraphics[width=0.04\textwidth]{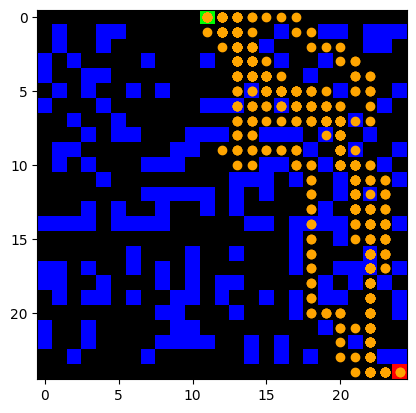}}}}{\tiny{43257}}
		\stackunder[5pt]{\subfloat{{\includegraphics[width=0.04\textwidth]{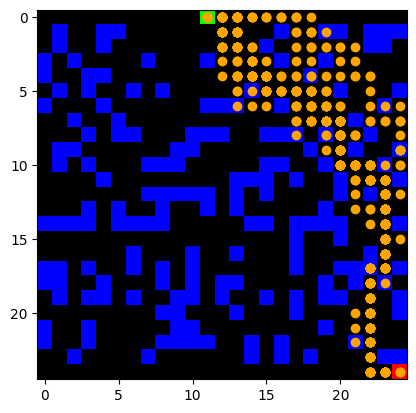}}}}{\tiny{17738}}
		\stackunder[5pt]{\subfloat{{\includegraphics[width=0.04\textwidth]{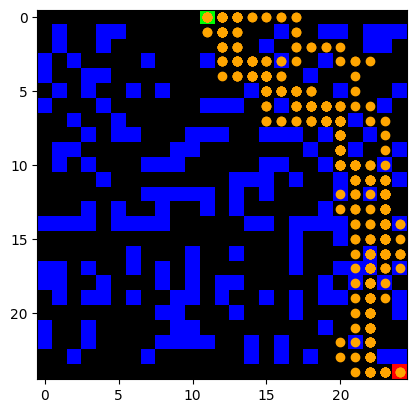}}}}{\tiny{69106}}
		\stackunder[5pt]{\subfloat{{\includegraphics[width=0.04\textwidth]{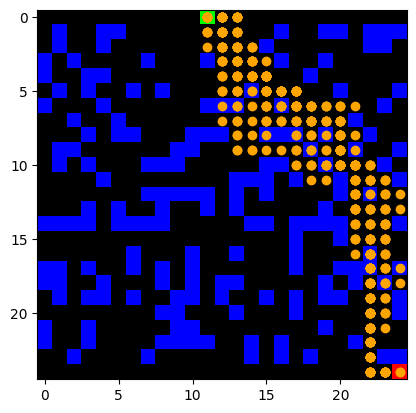}}}}{\tiny{90246}}
		\stackunder[5pt]{\subfloat{{\includegraphics[width=0.04\textwidth]{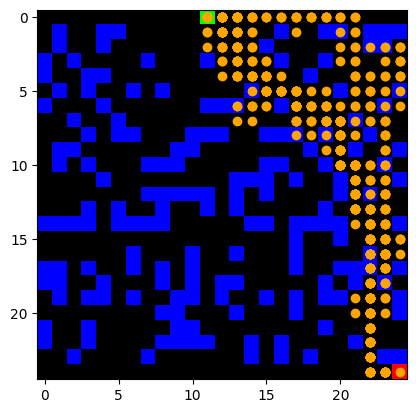}}}}{\tiny{68138}}
		\label{fig:gridnav_var_ac}
	}{\tiny{Policy visualisations for VaR\_AC in the GridNav environment.}}
\stackunder[5pt]{
		\stackunder[5pt]{\subfloat{{\includegraphics[width=0.04\textwidth]{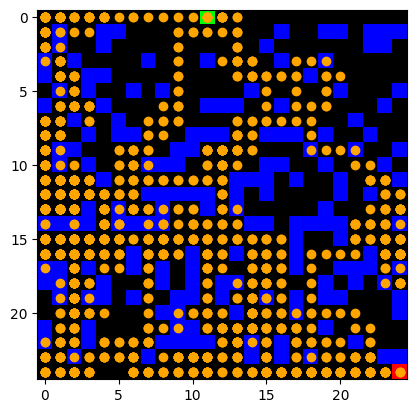}}}}{\tiny{00946}}
		\stackunder[5pt]{\subfloat{{\includegraphics[width=0.04\textwidth]{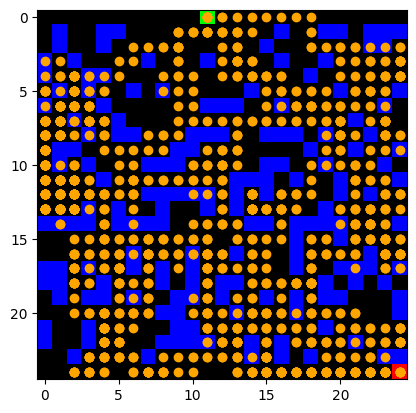}}}}{\tiny{49265}}
		\stackunder[5pt]{\subfloat{{\includegraphics[width=0.04\textwidth]{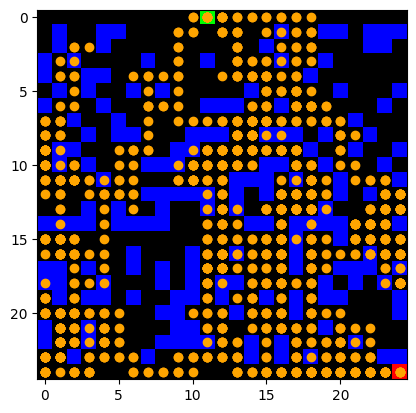}}}}{\tiny{41783}}
		\stackunder[5pt]{\subfloat{{\includegraphics[width=0.04\textwidth]{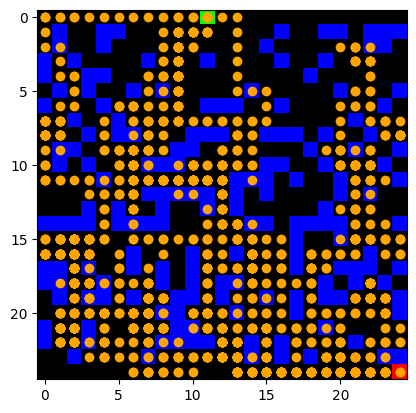}}}}{\tiny{30912}}
		\stackunder[5pt]{\subfloat{{\includegraphics[width=0.04\textwidth]{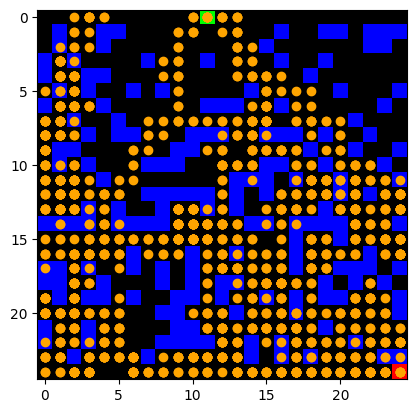}}}}{\tiny{27818}}
		\stackunder[5pt]{\subfloat{{\includegraphics[width=0.04\textwidth]{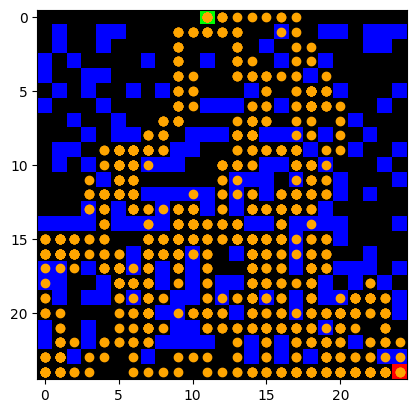}}}}{\tiny{93295}}
		\stackunder[5pt]{\subfloat{{\includegraphics[width=0.04\textwidth]{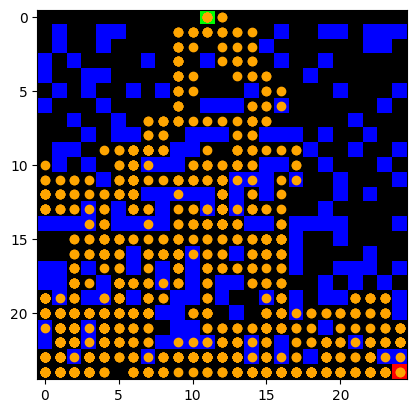}}}}{\tiny{47182}}
		\stackunder[5pt]{\subfloat{{\includegraphics[width=0.04\textwidth]{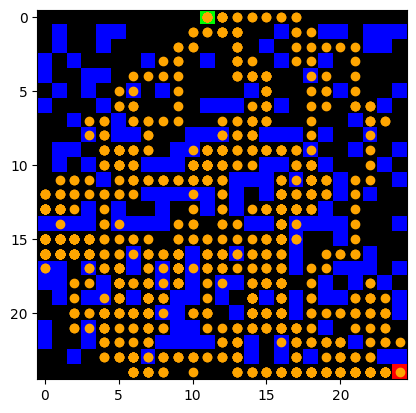}}}}{\tiny{57035}}
		\stackunder[5pt]{\subfloat{{\includegraphics[width=0.04\textwidth]{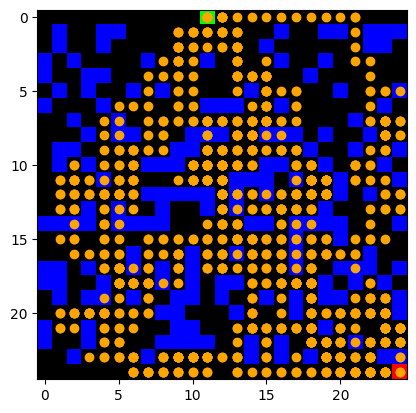}}}}{\tiny{07983}}
		\stackunder[5pt]{\subfloat{{\includegraphics[width=0.04\textwidth]{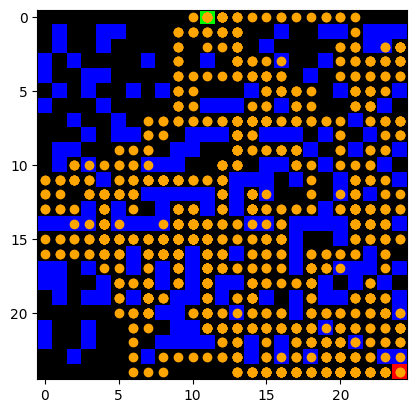}}}}{\tiny{82687}}
		\label{fig:gridnav_rcpo}
	}{\tiny{Policy visualisations for RCPO in the GridNav environment.}}
\stackunder[5pt]{
		\stackunder[5pt]{\subfloat{{\includegraphics[width=0.04\textwidth]{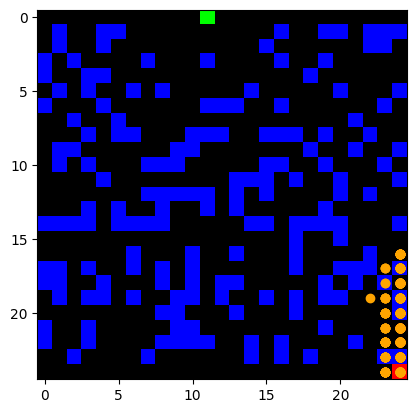}}}}{\tiny{91066}}
		\stackunder[5pt]{\subfloat{{\includegraphics[width=0.04\textwidth]{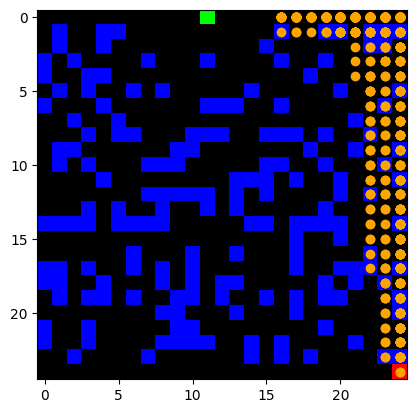}}}}{\tiny{07667}}
		\stackunder[5pt]{\subfloat{{\includegraphics[width=0.04\textwidth]{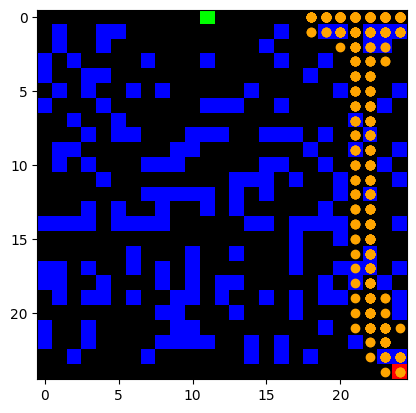}}}}{\tiny{64165}}
		\stackunder[5pt]{\subfloat{{\includegraphics[width=0.04\textwidth]{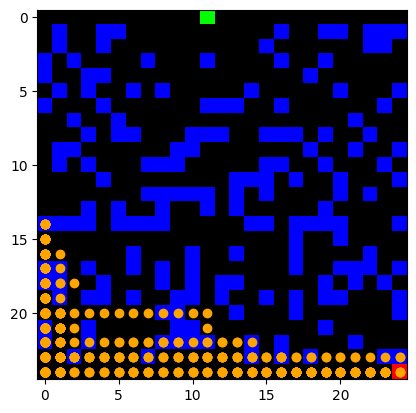}}}}{\tiny{45452}}
		\stackunder[5pt]{\subfloat{{\includegraphics[width=0.04\textwidth]{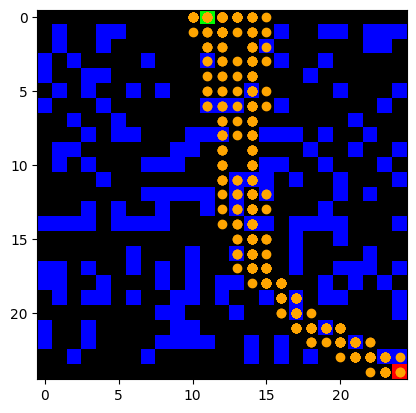}}}}{\tiny{80456}}
		\stackunder[5pt]{\subfloat{{\includegraphics[width=0.04\textwidth]{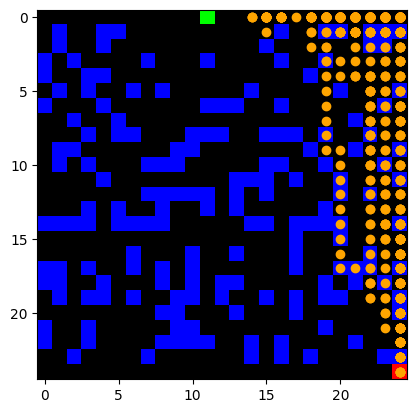}}}}{\tiny{17898}}
		\stackunder[5pt]{\subfloat{{\includegraphics[width=0.04\textwidth]{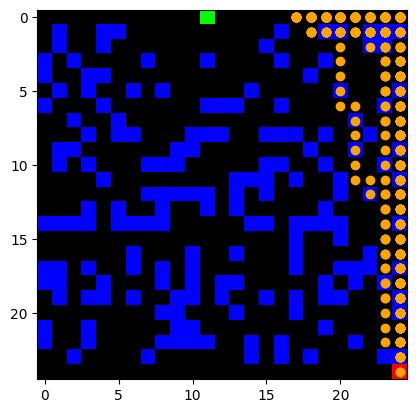}}}}{\tiny{83029}}
		\stackunder[5pt]{\subfloat{{\includegraphics[width=0.04\textwidth]{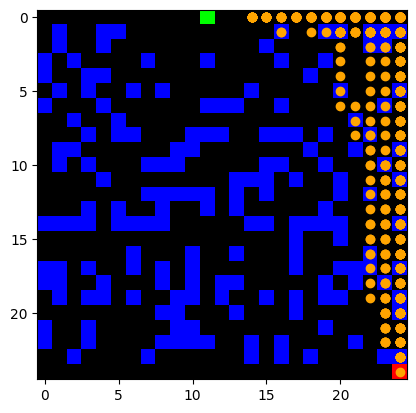}}}}{\tiny{15208}}
		\stackunder[5pt]{\subfloat{{\includegraphics[width=0.04\textwidth]{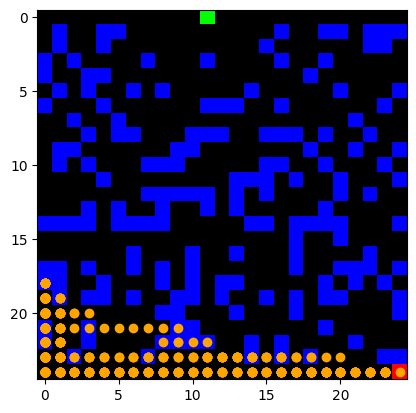}}}}{\tiny{35357}}
		\stackunder[5pt]{\subfloat{{\includegraphics[width=0.04\textwidth]{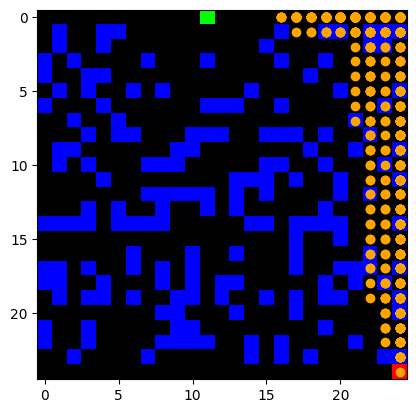}}}}{\tiny{21522}}
		\label{fig:gridnav_apropo}
	}{\tiny{Policy visualisations for AproPO in the GridNav environment.}}
\caption{This figure displays visualisations of the policies learned by the algorithms in the GridNav environment. Each plot corresponds to a policy learned by an algorithm, and is labelled with the random seed used in that run of the experiment. They were generated by running each policy ten times, and drawing the path of the algorithm through the environment (in orange).}
\label{fig:gridnav_policy_visualisations}
\end{figure}


\end{document}